\documentclass{article}

\usepackage[final]{microtype}
\usepackage{graphicx}
\usepackage{subcaption}
\usepackage{booktabs}
\usepackage{hyperref}

\usepackage[preprint]{icml2026}

\usepackage{amsmath,amssymb,amsfonts,amsthm,mathtools}
\usepackage{bm}
\usepackage{graphicx}
\usepackage{url}
\usepackage{mathrsfs}
\usepackage{enumitem}
\usepackage{bbm}

\newcommand{\mat}[1]{\ensuremath{\begin{bmatrix} #1 \end{bmatrix}}}
\newcommand{\ol}[1]{\overline{#1}}

\renewcommand{\rho}{\varrho}
\newcommand{\mean}[1]{\mathbb{E}[#1]}
\newcommand{\var}[1]{\mathbb{V}\left[#1\right]}
\newcommand{\prob}[1]{\mathbb{P}\left[#1\right]}
\newcommand{\norm}[1]{\left\lVert #1 \right\rVert}
\renewcommand{\exp}[1]{\mathrm{exp}\left(#1\right)}
\newcommand{\inprod}[1]{\left\langle #1 \right\rangle}

\usepackage[capitalize,noabbrev]{cleveref}

\theoremstyle{plain}
\newtheorem{theorem}{Theorem}[section]
\newtheorem{proposition}[theorem]{Proposition}
\newtheorem{lemma}[theorem]{Lemma}
\newtheorem{corollary}[theorem]{Corollary}
\theoremstyle{definition}

\newtheorem{assumption}[theorem]{Assumption}
\theoremstyle{remark}

\newenvironment{sketchproof}{\proof}{\endproof}

\let\emptyset\varnothing

\icmltitlerunning{On Uniform Error Bounds for Kernel Regression under Non-Gaussian Noise}

\begin{document}

\twocolumn[
  \icmltitle{On Uniform Error Bounds for Kernel Regression under Non-Gaussian Noise}

  \icmlsetsymbol{equal}{*}

  \begin{icmlauthorlist}
    \icmlauthor{Johannes Teutsch}{lsr}
    \icmlauthor{Oleksii Molodchyk}{ics}
    \icmlauthor{Marion Leibold}{lsr}
    \icmlauthor{Timm Faulwasser}{ics}
    \icmlauthor{Armin Lederer}{nus}
  \end{icmlauthorlist}

    \icmlaffiliation{lsr}{Chair of Automatic Control Engineering,
      Department of Computer Engineering,
      Technical University of Munich,
      80333 Munich, Germany}
    \icmlaffiliation{ics}{Institute of Control Systems,
      Hamburg University of Technology,
      21073 Hamburg, Germany}
    \icmlaffiliation{nus}{
      Department of Electrical and Computer Engineering,
      National University of Singapore,
      117583 Singapore}

  \icmlcorrespondingauthor{Johannes Teutsch}{johannes.teutsch@tum.de}

  \icmlkeywords{Kernel Regression, Error Bounds}

  \vskip 0.3in
]

\printAffiliationsAndNotice{} 

\begin{abstract}
    Providing non-conservative uncertainty quantification for function estimates derived from noisy observations remains a fundamental challenge in statistical machine learning, particularly for applications in safety-critical domains. In this work, we propose novel non-asymptotic probabilistic uniform error bounds for kernel-based regression. Compared to related bounds in the literature that are restricted to (conditionally) independent sub-Gaussian noise, our bounds allow to consider a broad class of non-Gaussian distributions, such as sub-Gaussian, bounded, sub-exponential, and variance/moment-bounded noise. Moreover, our results apply to correlated and uncorrelated noise. We compare our proposed error bounds with existing results in terms of the induced uncertainty region and their performance in safe control, demonstrating the tightness of the proposed bounds.
\end{abstract}

\section{Introduction}

Many machine learning tasks can be formulated as the estimation of an unknown function from input/output data obtained from a finite number of noisy function evaluations. Widely used non-parametric learning approaches for this purpose are kernel-based methods, such as kernel ridge regression (KRR) \cite{Schoelkopf2001} and Gaussian Process regression (GPR) \cite{Rasmussen2006}; see \cite{kanagawa2025gaussian} for a discussion on connections between KRR and GPR. 
These kernel-based methods provide not only estimates of the unknown function, but also bounds on the error between the unknown function and its estimate. When employing function estimates in safety-critical domains such as safe Bayesian optimization \cite{sui2015safe}, reinforcement learning \cite{chua2018deep}, or robot control \cite{berkenkamp2023bayesian}, the availability of error bounds is of utmost importance for safety certificates. 

For practical utility, such error bounds need to be non-asymptotic (i.e., valid for a finite amount of data) and non-conservative. For the case of bounded noise on the function evaluations, various deterministic error bounds have been proposed, e.g., \cite{maddalena2021deterministic,scharnhorst2022robust,hashimoto2022learning,reed2025error,lahr2025optimal}. However, such deterministic error bounds tend to be overly conservative. Conservatism can be reduced by leveraging probabilistic knowledge of the noise, resulting in probabilistic error bounds \cite{Srinivas2009,abbasi2013online,Chowdhury2017,Fiedler2021}. While probabilistic error bounds allow for (high-probability) safety certificates even in the case of unbounded noise (e.g., Gaussian), they suffer mainly from two issues: i) the uncertainty stemming from insufficient exploration of the hypothesis space and noise corruption of observed data is not treated separately, resulting in unnecessary conservatism, and ii) they are restricted to the class of sub-Gaussian distributions. Recent works \cite{reed2025error,molodchyk2025towards} attempt to overcome the former issue, but fail to provide error bounds that hold uniformly over the input domain, crucial for safety-critical applications~\cite{Lederer2019}. 
While Chowdhury \& Gopalan~\yrcite{Chowdhury2019bayesian} have examined error bounds for heavy‑tailed distributions, their results rely on output statistics instead of purely noise statistics, thereby requiring additional knowledge on the unknown function.\vspace{-2mm}
\paragraph{Contributions} 
In this work, we propose novel probabilistic uniform error bounds for kernel-based regression. The proposed bounds leverage a separation of the uncertainty into two components: one stemming from lack of exploration of the function space, and one induced by the noise corruption of data. This separation allows us to derive error bounds tailored to the statistical properties of the noise distribution. 
In particular, we propose uniform error bounds for (a)~sub-Gaussian, (b)~bounded, (c)~sub-exponential, and (d)~variance-bounded noise. Furthermore, we extend our analysis and results to the cases of correlated, i.e. non-i.i.d., sub-Gaussian noise, as well as moment-bounded outputs.
We draw upon numerical examples to demonstrate the tightness of the proposed error bounds relative to related results from the literature \cite{abbasi2013online,Fiedler2021,Chowdhury2019bayesian}. These results underpin that the proposed probabilistic uniform error bounds can significantly improve upon common baselines.

\paragraph{Outline}
The problem setup is specified in Section~\ref{sec:setup}. In Section~\ref{sec:main}, we present the proposed error bounds. Section~\ref{sec:discuss} conceptually discusses the proposed error bounds in relation to the literature. In Section~\ref{sec:eval}, we numerically evaluate the proposed bounds, before we conclude the work in Section~\ref{sec:conclusion}.

\paragraph{Notation} 
We denote $\mathbb{N}_{a} \doteq \{1, 2, \dots, a\}$ with $\mathbb{N}_\infty = \mathbb{N}$. 
For $n\in\mathbb{N}$, we denote the $n \times n$ identity matrix as $\mathbf{I}_n$.
The weighted 2-norm of a vector $\mathbf{s}$ is $\lVert \mathbf{s}\rVert_{\mathbf{S}} \coloneqq \sqrt{\mathbf{s}^{\top} \mathbf{S} \mathbf{s}}$ with $\mathbf{S} \succeq \bm{0}$.
Given a compact set $\mathcal{X}$, we write $Z(\zeta,\mathcal{X})$ for its covering number with grid constant $\zeta > 0$, i.e., 
$Z(\zeta,\mathcal{X}) = {\mathrm{min}}_{\mathcal{X}_\zeta \subset \mathcal{X}}\, \lvert \mathcal{X}_\zeta \rvert\,\mathrm{s.t.}\,\forall x \in \mathcal{X}\, \exists x' \in \mathcal{X}_\zeta: \lVert x - x' \rVert \le \zeta$,
where $\lvert \mathcal{X}_\zeta \rvert$ denotes the cardinality of $\mathcal{X}_\zeta$.
We define all random variables (RVs) $X : \Omega \to \mathbb{R}$ as measurable functions on a common probability space $(\Omega, \mathcal{F}, \mathbb{P})$, with sample space $\Omega$, $\sigma$-algebra $\mathcal{F}$, and probability measure $\mathbb{P}: \mathcal{F} \to [0,1]$. 

\section{Problem Setup} \label{sec:setup}
We consider an unknown function $f: \mathcal{X} \to \mathbb{R}$ with $n_x$-dimensional inputs $x \in \mathcal{X} \subset \mathbb{R}^{n_x}$ and (for ease of exposition) scalar outputs, satisfying the following assumption. 
\begin{assumption} \label{asm:system}
     The function $f$ lies in the reproducing kernel Hilbert space (RKHS) $\mathcal{H}_k$ corresponding to a known positive definite kernel function $k: \mathcal{X} \times \mathcal{X} \to \mathbb{R}$. The domain $\mathcal{X} \neq \emptyset$ is compact and a bound $B \in \mathbb{R}_{\ge0}$ on the RKHS norm $\norm{f}_{\mathcal{H}_k} \le B$ of $f$ is known.
\end{assumption}
The upper bound $B$ on the RKHS norm of the unknown function $f$ limits the complexity of the function in the (possibly infinite-dimensional) hypothesis space $\mathcal{H}_k$ and is commonly assumed in the literature, e.g., \cite{Srinivas2009,abbasi2013online,Chowdhury2017,Fiedler2021,reed2025error,lahr2025optimal}.

Although the function $f$ is unknown, data consisting of $t \in \mathbb{N}$ input-output pairs is available, i.e., 
\begin{equation} \label{eq:dataset}
    \mathcal{D}_t = \left\{(x_i,\, y_i) \in \mathcal{X} \times \mathbb{R} \mid i \in \mathbb{N}_t \right\}.
\end{equation}
The output data $(y_1, \dots,  y_t)$ are obtained from finitely many function evaluations affected by noise, i.e., 
\begin{equation} \label{eq:system}
    y_i = f(x_i) + M_i(\omega), \quad i \in \mathbb{N}_t,
\end{equation}
with the noise $M_i$ being an RV whose distribution class will be specified later (see  Section~\ref{sec:randvar}). 

Using the data from \eqref{eq:dataset}, we define the Gramian as $\mathbf{K}_t \doteq \left[k(x_i,x_j)\right]_{i,j \in \mathbb{N}_t}$ and the vector of kernels centered on $x_i$ as $\mathbf{k}_t(x) \doteq [k(x_1,x)~\ldots~k(x_t,x)]^\top$. Let us further denote the stacked data vectors as $\mathbf{x}_t~=~\mat{x_1 ~ \dots ~ x_t}^{\top}$, $\mathbf{y}_t = \mat{y_1 ~ \dots ~ y_t}^{\top}$, and the vector of noise realizations as $\mathbf{M}_t(\omega)$ with the corresponding vector of RVs $\mathbf{M}_t = \mat{M_1 ~ \dots ~ M_t}^{\top}$. We assume that the noise mean is known, and thus $\mean{\mathbf{M}_t}=\bm{0}$ without loss of generality.\footnote{The mean $\mean{\mathbf{M}_t}\neq\bm{0}$ can be integrated into the estimate~\eqref{eq:kermean}, shifting the error analysis towards the centered noise $\mathbf{M}_t - \mean{\mathbf{M}_t}$.}
We phrase the estimation of the unknown function $f$ as a KRR problem \cite{Schoelkopf2001,kanagawa2025gaussian}, i.e.,
\begin{equation} \label{eq:krr_problem}
    \mu_t \doteq \underset{g \in \mathcal{H}_k}{\mathrm{argmin}} \sum_{i=1}^t (g(x_i) - y_i)^2 + \rho^2 \norm{g}^2_{\mathcal{H}_k},
\end{equation}
with regularization parameter $\rho > 0$. Thus, for given $x \in \mathcal{X}$, we obtain the common kernel-based point-estimate
\begin{equation}
    \mu_t(x) \doteq \mathbf{y}_t^{\top}\left(\mathbf{K}_t+\rho^2\mathbf{I}_t\right)^{-1}\mathbf{k}_t(x), \label{eq:kermean}
\end{equation}
corresponding to the mean predictor from GPR \cite{kanagawa2025gaussian} with posterior variance
\begin{equation}
    \sigma^2_{t}(x) \doteq k(x,x) - \mathbf{k}_t(x)^{\top}\left(\mathbf{K}_t+\rho^2\mathbf{I}_t\right)^{-1} \mathbf{k}_t(x). \label{eq:kervar}
\end{equation}

To employ the estimate~\eqref{eq:kermean} in safety-critical contexts (e.g., safe Bayesian optimization~\citep{sui2015safe}), we aim to find probabilistic uniform error bounds for the deviation between the estimate $\mu_t$ and the unknown function $f$, i.e., a function $\eta_t: \mathcal{X} \to \mathbb{R}_{\ge 0}$ such that
\begin{equation} \label{eq:errorbound_uniform}
    \prob{\forall x \in \mathcal{X}, t \in \mathbb{N}:\lvert f(x) - \mu_t(x)\rvert \le {\eta}_t(x) } \ge 1-\delta.
\end{equation}
Achieving this uniformity over all $x\in \mathcal{X}$ is challenging whenever $\mathcal{X}$ is uncountable \cite{Srinivas2009}.
To overcome this and to derive uniform bounds, we rely on the following mild assumption on the kernel \citep{lederer2023gaussian}. 
\begin{assumption} \label{asm:hoelder}
    The kernel $k$ is Hölder continuous on $\mathcal{X}$ with order $p\in(0,\,1]$ and constant $L\in\mathbb{R}_{\ge0}$, i.e., ${\lvert k(x,x') - k(x,x'')\rvert} \le L \lVert x' - x''\rVert^p ~~\forall x, x', x'' \in \mathcal{X}$.
\end{assumption}
Note that Hölder continuity with $p=1$ corresponds to Lipschitz continuity of the kernel $k$, which most commonly used kernels satisfy. For a discussion on Hölder continuity of kernels and its implications, we refer to Fiedler~\yrcite{fiedler2023lipschitz}.

\subsection{Classes of Random Variables} \label{sec:randvar}
Handling general noise distributions is often non-trivial. Therefore, we focus on specific subsets of the probability space $(\Omega, \mathbb{P}; \mathbb{R})$, in which RVs admit favorable properties:    
\paragraph{(a) Sub-Gaussian ($\mathcal{SG}$)} 
A real-valued RV $M \in (\Omega, \mathbb{P}; \mathbb{R})$ is zero-mean sub-Gaussian with variance proxy $\sigma^2\in \mathbb{R}_{\ge0}$, i.e., $M \in \mathcal{SG}(\sigma^2)$, if \cite{vershynin2018high}
\begin{equation*}
    \forall \lambda \in \mathbb{R}:~~ \mean{\exp{\lambda M}} \le \exp{0.5{\lambda^2 \sigma^2}}.
\end{equation*}
Furthermore, a vector-valued RV $\mathbf{M} \in (\Omega, \mathbb{P}; \mathbb{R}^n)$ is zero-mean sub-Gaussian with matrix variance proxy $\bm{\Sigma} \succeq \bm{0}$, i.e., $\mathbf{M} \in \mathcal{SG}(\bm{\Sigma})$, if \cite{ao2025stochastic}
\begin{equation*}
    \forall \bm{\lambda} \in \mathbb{R}^n:\, \mean{\exp{\bm{\lambda}^{\top} \mathbf{M}}} \le \exp{{0.5\lVert \bm{\lambda}\rVert_{\bm{\Sigma}}}}.\hfill
\end{equation*}
\paragraph{(b) Bounded ($\mathcal{L}^{\infty}$)} 
A real-valued RV $M \in (\Omega, \mathbb{P};\mathbb{R})$ is bounded by $\overline{m}\in \mathbb{R}_{\ge0}$, i.e., $M \in \mathcal{L}^\infty( \overline{m})$, if $\lvert M\rvert \leq \overline{m}$ $\mathbb{P}$-almost everywhere \citep{Sullivan2015}.
\paragraph{(c) Sub-Exponential ($\mathcal{SE}$)} 
A real-valued RV $M \in (\Omega, \mathbb{P};\mathbb{R})$ is zero-mean sub-exponential with parameters $\nu,\alpha \in \mathbb{R}_{\ge0}$, i.e., $M \in \mathcal{SE}(\nu^2,\,\alpha)$, if \cite{wainwright2019high}
\begin{equation*}
    \forall \lvert\lambda\rvert <1/{\alpha}:~~ \mean{\exp{\lambda M}} \le \exp{0.5\lambda^2 \nu^2}.
\end{equation*}
\paragraph{(d) Variance-bounded ($\mathcal{L}^2$)} 
 A real-valued RV $M \in (\Omega, \mathbb{P};\mathbb{R})$ is variance-bounded with variance $\sigma^2\in \mathbb{R}_{\ge0}$, i.e., $M \in \mathcal{L}^2(\sigma^2)$, if $\mathbb{V}[M] \leq \sigma^2$ \citep{Sullivan2015}.

Note that on the probability space $(\Omega, \mathbb{P}; \mathbb{R})$, the considered classes admit a hierarchy $\mathcal{L}^\infty \subset \mathcal{SG} \subset \mathcal{SE} \subset \mathcal{L}^2$. Each of these sets entails practically relevant distributions. For instance, Gaussians ($\mathcal{SG}$) are well suited to describe thermal noise on resistors \cite{johnsonThermalAgitationElectricity1928}. Some hyperbolic distributions ($\mathcal{SE}$) are effective in quantifying errors in wind power forecasting~\cite{hodge2012wind}. On the other hand, in economic studies, log-normal distributions ($\mathcal{L}^2$) can be used to describe evolutions of company sizes \cite{sutton1996gibrat}.

\section{Main Results} \label{sec:main}
We present probabilistic uniform error bounds of the form \eqref{eq:errorbound_uniform} for the deviation between the unknown function $f$ from \eqref{eq:system} and the kernel-based estimate $\mu_t$ from \eqref{eq:kermean}, considering the distribution classes presented in Section~\ref{sec:randvar}. To this end, we leverage the following general result on probabilistic uniform error bounds, using $\tilde{\sigma}_{t}(x) \doteq \sqrt{\sigma^2_{t}(x) - \rho^2 \norm{\mathbf{h}_t(x)}_2^2}$ and $\mathbf{h}_t(x) \doteq \left(\mathbf{K}_t+\rho^2\mathbf{I}_t\right)^{-1}\mathbf{k}_t(x)$ as short-hand notations.

\begin{lemma} \label{lem:errorbound_general}
    Consider the data in \eqref{eq:dataset} generated via \eqref{eq:system} under Assumption~\ref{asm:system} with noise realizations $\mathbf{M}_t(\omega) = \mat{M_1(\omega) ~ \dots ~ M_t(\omega)}^{\top}$. Furthermore, let $\eta_t^M: \mathcal{X} \to \mathbb{R}_{\ge 0}$ be such that, for every $\rho > 0$ and $\delta \in (0,1)$,
    \begin{align}
       \mathbb{P}\left[ \forall x \in \mathcal{X},\,t\in\mathbb{N}:\,\lvert\mathbf{M}_t^{\top}\mathbf{h}_t(x)\rvert \le {\eta}^{M}_t(x)\right] \geq 1 - \delta.\label{eq:noisebound_uniform}
    \end{align}
    Then, $\lvert f(x) - \mu_t(x) \rvert$ is bounded by the probabilistic uniform error bound~\eqref{eq:errorbound_uniform} with ${\eta}_t(x) = B \tilde{\sigma}_t(x) + {\eta}^{M}_t(x)$.
\end{lemma}
\begin{sketchproof}
    The assertion follows from the triangle inequality and the bound \eqref{eq:noisebound_uniform}; see Appendix~\ref{app:proof_generalbound}.
\end{sketchproof}

The error bound ${\eta}_t(x)$ from Lemma~\ref{lem:errorbound_general} decomposes the overall uncertainty into two terms: $B\tilde{\sigma}_t(x)$, arising from lack of sufficient exploration of the RKHS, and ${\eta}^{M}_t(x)$, induced by the noise in the data. 
Crucially, a function ${\eta}_t^M$ satisfying \eqref{eq:noisebound_uniform} is not readily available. In our following main result, we propose such functions ${\eta}_t^M$ for a wide range of (joint) non-Gaussian distributions of the noise $M_i$, $i \in \mathbb{N}_{t}$ (see Section~\ref{sec:randvar}), leading to tight uniform error bounds~\eqref{eq:errorbound_uniform}.

\begin{table*}[!t]
    \renewcommand{\arraystretch}{1.5}
    \setlength{\tabcolsep}{8pt}
    \centering
    \caption{Parameters of the proposed probabilistic uniform error bounds from Theorem~\ref{thm:errorbounds_uniform}, with $\beta_{1,t}\doteq2\ln(4\pi_t Z(\zeta_{t},\mathcal{X})/\delta)$, $\beta_{2,t}\doteq2\ln(4\pi_t t/\delta)$, $a_{1,t} \doteq \sqrt{L/({2\rho^2})}\zeta_{t}^{p/2}$, $a_{2,t} \doteq t L \zeta_{t}^{p}$, and $\Delta_t(\sigma) = \sqrt{\beta_{1,t}(\zeta_{t})}\sigma a_{1,t} + \sqrt{\beta_{2,t}} \sigma a_{2,t}$ for $\sigma\in\mathbb{R}_{\ge0}$}
    \scalebox{1}{
    \begin{tabular}{l|ll} \hline
        Noise class & Scaling factor & Discretization term \\ \hline
        \textbf{(a)} $\mathcal{SG}$: ~\eqref{eq:errorbound_subG_uniform} & $\beta^{\mathcal{SG}}_{t} \doteq \sqrt{\beta_{1,t}}$ & $\Delta^{\mathcal{SG}}_{t} \doteq \Delta_{t} (\sigma_M)$ \\ 
        \textbf{(b)} $\mathcal{L}^{\infty}$: \eqref{eq:errorbound_bnd_hoeffding_uniform} &$\beta^{\mathrm{bnd}}_{1,t} \doteq \sqrt{\beta_{1,t}}$ & $\Delta^{\mathrm{bnd}}_{1,t} \doteq \Delta_{t} (\overline{m})$\\
        \phantom{\textbf{(b)} $\mathcal{L}^{\infty}$:} \eqref{eq:errorbound_bnd_bernstein_uniform} & $\beta^{\mathrm{bnd}}_{2,t} \doteq \frac{1}{3}\beta_{1,t}$ & $\Delta^{\mathrm{bnd}}_{2,t} \doteq (\sqrt{\beta_{1,t}}\,\overline{\sigma} +\frac{1}{3}\beta_{1,t} \,\overline{m})a_{1,t} + (\sqrt{\beta_{2,t}}\,\overline{\sigma} +\frac{1}{3}\beta_{2,t} \,\overline{m}) a_{2,t}$\\
        \textbf{(c)} $\mathcal{SE}$: ~\eqref{eq:errorbound_subE_uniform} & $\beta^{\mathcal{SE}}_{t} \doteq \beta_{1,t}$ & $\Delta^{\mathcal{SE}}_{t}\doteq \max\{ \beta_{1,t}\, \alpha_M,\,\sqrt{\beta_{1,t}}\, \nu_M\}a_{1,t} + \max\{ \beta_{2,t}\, \alpha_M,\,\sqrt{\beta_{2,t}}\, \nu_M\} a_{2,t}$ \\ 
        \textbf{(d)} $\mathcal{L}^2$: ~\eqref{eq:errorbound_L2_uniform}& $\beta^{\mathcal{L}2}_{t} \doteq \sqrt{\frac{1}{2}\exp{\frac{\beta_{1,t}}{2}}}$ & $\Delta^{\mathcal{L}2}_{t} \doteq \beta^{\mathcal{L}2}_{t}\sigma_M a_{1,t} + \sqrt{\frac{1}{2}\exp{\frac{\beta_{2,t}}{2}}} \sigma_M a_{2,t}$ \\[3mm] \hline
    \end{tabular}
    }
    \label{tab:bounds}
    \renewcommand{\arraystretch}{1}
\end{table*}

\begin{theorem}[Uniform error bounds] \label{thm:errorbounds_uniform}
    Consider data \eqref{eq:dataset} generated via \eqref{eq:system} under Assumption~\ref{asm:system} with realizations $M_1(\omega), \dots,  M_t(\omega)$ of i.i.d. noise $M_i$, $i \in \mathbb{N}_t$, some $\{\pi_t\}_{t=1}^{\infty}$ that satisfies $\sum_{t=1}^{\infty}\pi_t^{-1} = 1$ (e.g., $\pi_t = \pi^2 t^2 / 6$), and the covering number $Z(\cdot,\mathcal{X})$ of $\mathcal{X}$.\\     
    If the kernel $k(\cdot,\cdot)$ satisfies Assumption~\ref{asm:hoelder}, then, for every $\rho > 0$, $\delta \in (0,1)$, and grid constants $\zeta_t>0$, the regression error $\lvert f(x) - \mu_t(x) \rvert$ is bounded by~\eqref{eq:errorbound_uniform} with ${\eta}_t(x) = B \tilde{\sigma}_t(x) + {\eta}^{M}_t(x)$ and ${\eta}^{M}_t(x)$ defined as follows:
   
    \paragraph{(a)} If $M_i \in \mathcal{SG}(\sigma_M^2)$, then ${\eta}^{M}_t(x) = {\eta}^{\mathcal{SG}}_t(x)$ with
    \begin{equation} \label{eq:errorbound_subG_uniform}
        {\eta}^{\mathcal{SG}}_t(x) \doteq  \beta^{\mathcal{SG}}_{t}\,\sigma_M\norm{\mathbf{h}_t(x)}_{2} + \Delta^{\mathcal{SG}}_{t},
    \end{equation}
    and parameters $\beta^{\mathcal{SG}}_{t}$, $\Delta^{\mathcal{SG}}_{t}$ from Table~\ref{tab:bounds}(a).
    
    \paragraph{(b)} If $M_i \in \mathcal{L}^{\infty}(\overline{m})$ and $\var{M_i} \le \overline{\sigma}^2$, then ${\eta}^{M}_t(x) = \min_{j\in\{1,2\}}{\eta}^\mathrm{bnd}_{j,t}(x)$ with  
    \begin{align} 
        {\eta}^\mathrm{bnd}_{1,t}(x) &\doteq  \beta^\mathrm{bnd}_{1,t}\, \overline{m} \lVert \mathbf{h}_t(x) \rVert_2 + \Delta^\mathrm{bnd}_{1,t},\label{eq:errorbound_bnd_hoeffding_uniform}
        \\
        {\eta}^\mathrm{bnd}_{2,t}(x) &\doteq \beta^\mathrm{bnd}_{1,t}\,\overline{\sigma} \lVert \mathbf{h}_t(x) \rVert_2 + \notag 
        \\ &~~~~~\beta_{2,t}^{\mathrm{bnd}} \,\overline{m} \lVert\mathbf{h}_t(x)\rVert_\infty + \Delta^{\mathrm{bnd}}_{2,t},\label{eq:errorbound_bnd_bernstein_uniform}
    \end{align}
   and parameters $\beta^{\mathrm{bnd}}_{1,t}$, $\beta^{\mathrm{bnd}}_{2,t}$, $\Delta^{\mathrm{bnd}}_{1,t}$, $\Delta^{\mathrm{bnd}}_{2,t}$ from Table~\ref{tab:bounds}(b).
    
    \paragraph{(c)} If $M_i \in \mathcal{SE}(\nu_M^2,\alpha_M)$, then ${\eta}^{M}_t(x) =  {\eta}^{\mathcal{SE}}_t(x)$ with
    \begin{align} \label{eq:errorbound_subE_uniform}
        {\eta}^{\mathcal{SE}}_t(x) \doteq \max&\{\beta^{\mathcal{SE}}_{t}\, \alpha_M \lVert \mathbf{h}_t(x)\rVert_\infty ,\notag \\ &  \sqrt{\beta^{\mathcal{SE}}_{t}}\, \nu_M \norm{\mathbf{h}_t(x)}_{2}\} + \Delta^{\mathcal{SE}}_{t},
    \end{align}
    and parameters $\beta^{\mathcal{SE}}_{t}$, $\Delta^{\mathcal{SE}}_{t}$ from Table~\ref{tab:bounds}(c).
    
    \paragraph{(d)} If $M_i \in \mathcal{L}^2(\sigma_M^2)$, then ${\eta}^{M}_t(x) = {\eta}^{\mathcal{L}2}_{t}(x)$ with         
    \begin{equation} \label{eq:errorbound_L2_uniform}
        {\eta}^{\mathcal{L}2}_{t}(x) \doteq  \beta^{\mathcal{L}2}_{t}~ \sigma_M \lVert \mathbf{h}_t(x)\rVert_2 + \Delta^{\mathcal{L}2}_{t},
    \end{equation}
    and parameters $\beta^{\mathcal{L}2}_{t}$, $\Delta^{\mathcal{L}2}_{t}$ from Table~\ref{tab:bounds}(d).
 
\end{theorem}
\begin{sketchproof}
    For all cases \textbf{(a)}--\textbf{(d)}, the assertions follow from Lemma~\ref{lem:errorbound_general} by deriving suitable functions ${\eta}_t^M$ satisfying \eqref{eq:noisebound_uniform} based on the statistics of the noise $M_i$, $i\in\mathbb{N}_t$. First, we derive nonuniform probabilistic error bounds of the form 
    \begin{equation} \label{eq:errorbound_nonuniform}
        \forall x \in \mathcal{X}, t \in \mathbb{N}:\prob{\lvert f(x) - \mu_t(x)\rvert \le \tilde{\eta}_t(x) } \ge 1-\delta
    \end{equation}
    using concentration inequalities for the respective distribution class \cite{ao2025stochastic,vershynin2018high,wainwright2019high}. Then, we obtain time-uniform versions of these nonuniform bounds by setting $\delta \leftarrow \delta/\pi_t$ and applying the union bound over all times \cite{Srinivas2009}.
    To obtain uniform bounds over all $x \in \mathcal{X}$, we leverage ideas from Lederer et al.~\yrcite{Lederer2019}: Consider a discretized version $\mathcal{X}_\zeta$ of the input domain $\mathcal{X}$ with $\lvert \mathcal{X}_\zeta \rvert = Z(\zeta_t,\mathcal{X})$ grid points and grid width $\zeta_t$, i.e., ${\max_{x \in \mathcal{X}} \min_{x'\in\mathcal{X}_\zeta} \lVert x - x' \rVert} \le \zeta_t$. Using the triangle inequality, we obtain $\lvert \mathbf{M}_t^{\top}\mathbf{h}_t(x) \rvert \le \lvert \mathbf{M}_t^{\top}\mathbf{h}_t([x]_\zeta) \rvert + \lvert \mathbf{M}_t^{\top}\left(\mathbf{h}_t(x) - \mathbf{h}_t([x]_\zeta)\right)\rvert$, with $[x]_\zeta \in \mathcal{X}_\zeta$ denoting the closest point in $\mathcal{X}_\zeta$ to $x$. A uniform bound for $\lvert \mathbf{M}_t^{\top}\mathbf{h}_t([x]_\zeta) \rvert$ results from applying the bounds~\eqref{eq:errorbound_nonuniform} and the union bound over all $Z(\zeta_t,\mathcal{X})$ grid points of $\mathcal{X}_\zeta$, whereas a uniform bound for the term $\lvert \mathbf{M}_t^{\top}\left(\mathbf{h}_t(x) - \mathbf{h}_t([x]_\zeta)\right)\rvert$ results from the Hölder inequality and Assumption~\ref{asm:hoelder}. A detailed proof is presented in Appendix~\ref{app:proof_uniform}.
\end{sketchproof}
Note that for Theorem~\ref{thm:errorbounds_uniform}(b), one can employ independent grid constants $\zeta_t$ for \eqref{eq:errorbound_bnd_hoeffding_uniform} and \eqref{eq:errorbound_bnd_bernstein_uniform} respectively to further sharpen the bound, and also leverage $\overline{\sigma}^2 \le \overline{m}^2$ if $\var{M_i}$ is not available~\cite{wainwright2019high}. 
Furthermore, the sub-exponential bound~\eqref{eq:errorbound_subE_uniform} in Theorem~\ref{thm:errorbounds_uniform}(c) recovers the sub-Gaussian bound~\eqref{eq:errorbound_subG_uniform} for $\alpha_M \to 0$ and $\nu_M = \sigma_M$.

Crucially, the uniform bounds in Theorem~\ref{thm:errorbounds_uniform} depend on grid constants $\zeta_t$ (respectively, $\zeta_{t}$), which are an additional tuning parameter for the bound. With the choice of $\zeta_t$, we can balance the discretization error and the noise-induced uncertainty, e.g., measured by the term $\norm{\mathbf{h}_t(x)}_{2}$ in~\eqref{eq:errorbound_subG_uniform}: A larger choice of $\zeta_t$ increases the discretization error $\Delta^{\mathcal{SG}}_t$, while a smaller choice of $\zeta_t$ leads to an increase in the covering number $Z(\zeta_t,\mathcal{X})$, thus increasing the scaling $\beta^{\mathcal{SG}}_t(\zeta_t,\delta)$ of $\norm{\mathbf{h}_t(x)}_{2}$, cf. \citep{Lederer2019}. Exact computation of the $\zeta$-covering number $Z(\zeta, \mathcal{X})$ is often intractable; however, by considering an over-approximation of the set $\mathcal{X}$ in form of an $n_x$-dimensional hypercube with edge length $r = \max_{x,x' \in \mathcal{X}}\lVert x - x' \rVert_{\infty}$, one can employ the upper bound \citep{shalev2014understanding,omainska2023rigid}
\begin{equation} \label{eq:covering_ub}
    Z(\zeta,\mathcal{X}) \le \left( 1 + \sqrt{n_x} r/(2\zeta)\right)^{n_x}.
\end{equation}

Uniformity of the error bounds over all times $t\in \mathbb{N}$ and inputs $x\in\mathcal{X}$ is especially interesting for safe Bayesian optimization and for the analysis of regret bounds~\cite{Srinivas2009,Chowdhury2017}. In applications where only uniformity in the input $x$ is sufficient (e.g., as in \cite{Lederer2019}), one can simply set $\pi_t = 1$ in the proposed bounds from Theorem~\ref{thm:errorbounds_uniform}. Likewise, when uniformity over a finite time horizon $t \in \mathbb{N}_T$ is sufficient, one can set $\pi_t = T$, with $T \in \mathbb{N}$ and $T < \infty$. In addition, when only subsets $\tilde{\mathcal{X}} \subset \mathcal{X}$ of the input domain $\mathcal{X}$ are of relevance in practice, the employed discretization approach can be adapted to the subset $\tilde{\mathcal{X}}$ of interest, thus reducing conservatism. This is especially relevant, e.g., if a specific rule is used at each $t \in \mathbb{N}$ to select the new input $x_t$, potentially rendering the subsets of $\mathcal{X}\backslash\tilde{\mathcal{X}}$ inadmissible.

\subsection{Non-i.i.d. Noise}
Theorem \ref{thm:errorbounds_uniform} provides probabilistic uniform error bounds under i.i.d. noise. 
The following result generalizes the sub-Gaussian bound~\eqref{eq:errorbound_subG_uniform} to potentially correlated noise, focusing solely on the sub-Gaussian case for ease of exposition.

\begin{proposition} \label{prop:cor_subG}
    Consider the setting of Theorem~\ref{thm:errorbounds_uniform}(a) with $\beta^{\mathcal{SG}}_{t}$ and $\Delta_t(\cdot)$ from Table~\ref{tab:bounds} and $\mathbf{M}_t \in \mathcal{SG}(\mathbf{C}_t)$ for some $\mathbf{C}_t\succeq\bm{0}$ with maximum singular value $\varsigma_{\mathbf{C}_t} = \lVert \mathbf{C}_t\rVert_2$.\\ Then, $\lvert f(x) - \mu_t(x) \rvert$ is bounded by the probabilistic uniform error bound \eqref{eq:errorbound_uniform} with ${\eta}_t(x) = B \tilde{\sigma}_t(x)  + \overline{\eta}^{\mathcal{SG}}_t(x)$ where
    \begin{equation} \label{eq:errorbound_subG_uniform_corr}
        \overline{\eta}^{\mathcal{SG}}_t(x) \doteq  \beta^{\mathcal{SG}}_{t}\, \norm{\mathbf{h}_t(x)}_{\mathbf{C}_t} + \Delta_{t}(\sqrt{\varsigma_{\mathbf{C}_t}}).
    \end{equation}
\end{proposition}

\begin{sketchproof}
    With $\mathbf{M}_t \in \mathcal{SG}(\mathbf{C}_t)$, \citep[Theorem~1a]{ao2025stochastic} yields $\mathbf{M}_t^{\top}\mathbf{h}_t(x) \in \mathcal{SG}(\lVert\mathbf{h}_t(x)\rVert_{\mathbf{C}_t}^2)$. The claims then follow from the proof of Theorem~\ref{thm:errorbounds_uniform}(a). A detailed proof is presented in Appendix~\ref{app:proof_cor-subG}.
\end{sketchproof}
In Proposition~\ref{prop:cor_subG}, potential correlations are expressed through the matrix variance proxy $\mathbf{C}_t$, recovering the i.i.d. result for $\mathbf{C}_t = \sigma_M^2 \mathbf{I}_t$. Note that, due to the properties of sub-Gaussian RVs, the variance proxy $\mathbf{C}_t$ can be replaced by an upper-bound $\overline{\mathbf{C}}_t \succeq \mathbf{C}_t$ if exact correlations are not known. Lastly, we remark that Theorem~\ref{thm:errorbounds_uniform}(b)--(d) can be adapted to potentially correlated noise accordingly.

While general noise correlations have not been considered in the literature to the best of our knowledge, multiple results for the special case of conditionally independent $\sigma_M$-sub-Gaussian noise exist \cite{abbasi2013online,Chowdhury2017}, i.e., $\forall t \in \mathbb{N}$, $\forall \lambda \in \mathbb{R}$: $\mean{\exp{\lambda M_t} | \mathcal{F}_{t-1}} \le \exp{\lambda^2 \sigma_M^2/2}$, where $\mathcal{F}_{t-1}$ is the $\sigma$-algebra generated by $M_1, \dots, M_{t-1}$.\footnote{Different to \cite{abbasi2013online,Chowdhury2017}, we treat the input $x_i \in \mathcal{X}$, $i \in \mathbb{N}_t$, as a \textit{deterministic} variable, i.e., $x_i$ only generates the trivial $\sigma$-algebra $\{\emptyset, \Omega\}$.} Our results can be straightforwardly adapted to this special case.
\begin{corollary} \label{cor:cond_subG}
    Consider the setting of Theorem~\ref{thm:errorbounds_uniform}(a) with $M_1, \dots, M_t$ being conditionally $\sigma_M$-sub-Gaussian. Then, $\lvert f(x) - \mu_t(x) \rvert$ is bounded by the probabilistic uniform error bound \eqref{eq:errorbound_uniform} with ${\eta}_t(x) = B\tilde{\sigma}_t + {\eta}^{\mathcal{SG}}_t(x)$ via \eqref{eq:errorbound_subG_uniform}.
\end{corollary}
\begin{sketchproof}
    The claim follows from  \citep[Theorem~1]{ao2025stochastic} and the proof of Theorem~\ref{thm:errorbounds_uniform}(a). A detailed proof is presented in Appendix~\ref{app:proof_cond-subG}.
\end{sketchproof}

\subsection{Heavy-tailed Noise}
While Theorem~\ref{thm:errorbounds_uniform}(a)--(c) provides probabilistic uniform error bounds under light-tailed noise, Theorem~\ref{thm:errorbounds_uniform}(d) extends to variance-bounded distributions, leveraging the Chebyshev-based bound~\eqref{eq:errorbound_L2_uniform}.
In contrast to the bounds of Theorem~\ref{thm:errorbounds_uniform}(a)--(c), the bound~\eqref{eq:errorbound_L2_uniform} of Theorem~\ref{thm:errorbounds_uniform}(d) does not show a logarithmic dependence on the covering number $Z(\zeta_t,\mathcal{X})$, but instead grows with $\sqrt{Z(\zeta_t,\mathcal{X})}$ due to the employed Chebyshev inequality. Specifically for large $t$ and small $\zeta_t$, this easily leads to a blow-up in the scaling factor in~\eqref{eq:errorbound_L2_uniform} (see Table~\ref{tab:bounds}(d)), thus limiting the applicability. Hence, for practical feasibility, one alternative is to employ a nonuniform version of~\eqref{eq:errorbound_L2_uniform} as in~\eqref{eq:errorbound_nonuniform} (see Theorem~\ref{thm:errorbounds_nonuniform}(d)), at the cost of weakening the safety guarantees. 

Another possibility to obtain uniform error bounds under heavy-tailed noise is by following a truncation approach \cite{Chowdhury2019bayesian}. Instead of  \eqref{eq:kermean}, consider the truncated estimate $\hat{\mu}_t(x) \doteq \hat{\mathbf{y}}_t^\top\left(\mathbf{K}_t+\rho^2\mathbf{I}_t\right)^{-1}\mathbf{k}_t(x)$ using the truncated outputs $\hat{\mathbf{y}}_t \doteq [\hat{y}_1,\dots,\hat{y}_t]^{\top}$, with $\hat{y}_i \doteq y_i \mathbbm{1}_{\lvert y_i\rvert \le b_i}$ for $i\in\mathbb{N}_t$, the indicator function $\mathbbm{1}_{(\cdot)}$, and some truncation level $b_i \in\mathbb{R}_{\ge 0}$.

\begin{proposition} \label{prop:errorbound_HT}
    Consider data \eqref{eq:dataset} generated via \eqref{eq:system} under Assumption~\ref{asm:system} with realizations $M_1(\omega), \dots,  M_t(\omega)$ of i.i.d. noise $M_i$, $i \in \mathbb{N}_t$, some $\{\pi_t\}_{t=1}^{\infty}$ that satisfies $\sum_{t=1}^{\infty}\pi_t^{-1} = 1$ (e.g., $\pi_t = \pi^2 t^2 / 6$), and the covering number $Z(\cdot,\mathcal{X})$ of~$\mathcal{X}$. Moreover, for all $i\in\mathbb{N}_t$, let $Y_i = f(x_i) + M_i$ be such that $\mean{\lvert Y_i\rvert^{1+a}} \le \ol{v}$ with some $a\in\mathbb{R}_{>0}$ and $\ol{v}\in\mathbb{R}_{\ge 0}$.\\     
    If the kernel $k(\cdot,\cdot)$ satisfies Assumption~\ref{asm:hoelder}, then, for every $\rho > 0$, $\delta \in (0,1)$, and grid constants $\zeta_t>0$, it holds that
    \begin{equation*} 
        \prob{\forall x \in \mathcal{X},~ t \in \mathbb{N}:~\lvert f(x) - \hat{\mu}_t(x)\rvert \le \eta^{\mathrm{HT}}_t(x)} \ge 1-\delta
    \end{equation*}
     with $\eta^{\mathrm{HT}}_t(x) \doteq B\tilde{\sigma}_t(x) + \beta^{\mathrm{HT}}_{t}b_t\,\norm{\mathbf{h}_t(x)}_{2} + \Delta^{\mathrm{HT}}_{t}$, scaling $b_t = \ol{v}^{\frac{1}{1+a}} t^{\frac{1}{2(1+a)}}$ and parameters $\beta^{\mathrm{HT}}_{t} \doteq 2 \beta^{\mathcal{SG}}_{t} + \sqrt{1+a}$ and $\Delta^{\mathrm{HT}}_{t} \doteq 2 b_t \Delta^{\mathcal{SG}}_{t}$ (see Table~\ref{tab:bounds}(a)).
\end{proposition}
\begin{sketchproof}
    The claim follows from the proofs of \citep[Lemma~8]{Chowdhury2019bayesian} and Theorem~\ref{thm:errorbounds_uniform}(a); see Appendix~\ref{app:proof_heavytail} for a detailed proof.
\end{sketchproof}

Due to the truncation approach, the error bound from Proposition~\ref{prop:errorbound_HT} recovers the logarithmic dependence on $Z(\zeta_t,\mathcal{X})$ from the proposed sub-Gaussian bound (Theorem~\ref{thm:errorbounds_uniform}(a)). However, compared to Theorem~\ref{thm:errorbounds_uniform}(d), more knowledge is required: In particular, for $a = 1$, Proposition~\ref{prop:errorbound_HT} requires a bound on the second moment of $Y_i$, i.e., $\mean{\lvert Y_i\rvert^{2}} \le \ol{v}$. Due to $Y_i = f(x_i) + M_i$, such a bound can be obtained via
\begin{align} \label{eq:momentbound}
    \mean{\lvert Y_i\rvert^{2}} \le 2 \ol{f}^2 + 2\mean{\lvert M_i\rvert^2},
\end{align}
thus requiring an upper-bound $\ol{f}\ge f(x_i)~\forall x_i \in \mathcal{X}$ on the unknown function, as well as the second noise moment $\mean{\lvert M_i\rvert^2}$. Alternatively, if (an upper-bound on) the first noise moment $\mean{\lvert M_i\rvert}$ is known, one can rewrite inequality~\eqref{eq:momentbound} to obtain $\mean{\lvert Y_i\rvert^{2}} \le \ol{f}^2+ 2\ol{f}~\mean{\lvert M_i\rvert} + \mean{\lvert M_i\rvert^2}$.

\subsection{Remarks}

Thus far, we have presented error bounds considering specific noise distribution classes, i.e., $\mathcal{L}^{\infty}$, $\mathcal{SG}$, $\mathcal{SE}$, and $\mathcal{L}^2$.
While we have focused on these distribution classes, the proposed framework straightforwardly extends to other classes (e.g., sub-$\psi$ distributions \cite{howard2020time}), merely requiring the derivation of a concentration inequality for the noise term $\lvert\mathbf{M}_t^{\top}\mathbf{h}_t(x)\rvert$ in \eqref{eq:bound_deterministic}. Such concentration inequalities allow for the derivation of the noise bound \eqref{eq:noisebound_uniform} in Lemma~\ref{lem:errorbound_general}, which we leverage in Theorem~\ref{thm:errorbounds_uniform}. 

The proposed error bounds rely on distributional parameters from the considered noise classes in Section 2.1, see Table~\ref{tab:bounds}. In case these parameters are not known exactly, the proposed bounds are still valid when \textit{upper-bounds} on the parameters are employed (e.g., some $\bar{\sigma}_M^2 \ge \sigma_M^2$ for sub-Gaussian noise), due to the general definition of the distribution classes. Although this renders the proposed bounds more conservative, the theoretical guarantees are preserved. Moreover, note that due to the hierarchy $\mathcal{L}^\infty \subset \mathcal{SG} \subset \mathcal{SE} \subset \mathcal{L}^2$ of the distribution classes discussed in Section 2.1, the provided bounds can be chosen depending on the available knowledge on (upper-bounds of) distributional parameters, yielding better bounds the more knowledge is available.

Lastly, we highlight that the proposed error bounds are valid for general functions $f$ lying in a possibly infinite dimensional hypothesis space $\mathcal{H}_k$, see Assumption~\ref{asm:system}. For the case that a finite-dimensional representation $f(x) = \bm{\theta}^\top \bm{\phi}(x)$ exists, with unknown parameter vector $\bm{\theta} \in \mathbb{R}^{n_\phi}$ and known vector of basis functions $\bm{\phi}: \mathcal{X} \to \mathbb{R}^{n_\phi}$, one might be interested in parameter estimation error bounds corresponding to \eqref{eq:errorbound_uniform}. In Appendix~\ref{app:parambounds}, we analyze this setting and provide such bounds as corresponding to Theorem~\ref{thm:errorbounds_uniform}. 
For an asymptotic analysis of the proposed error bounds, we remark that the approach of Lederer et al.~\yrcite{Lederer2019} can be adapted due to $\rho\lVert \mathbf{h}_t(x)\rVert_\infty \le \rho\lVert \mathbf{h}_t(x)\rVert_2\le\sigma_t(x)$.

\section{Discussion and Related Work} \label{sec:discuss}
In this section, we discuss and conceptually compare the proposed distribution class-specific uniform error bounds from Section~\ref{sec:main} with related bounds from the literature.

\paragraph{General Remarks} 
The proposed error bounds suffer from the common limitations of kernel regression, namely potential misspecifications of the kernel $k(\cdot,\cdot)$ and the RKHS-norm bound $B$, as discussed in \citep[Section 4.4]{lahr2025optimal}. Note that, as for related bounds from the literature, the parameter $\rho$ can be freely chosen, e.g., to minimize the proposed error bounds. However, the choice of $\rho$ not only affects the error bounds but also the estimate $\mu_t(x)$ via \eqref{eq:kermean}.

\paragraph{Bounded/Sub-Gaussian Noise}
In the literature, various uniform error bounds have been presented for the setting of bounded/sub-Gaussian noise; we refer to Fiedler et al.~\yrcite{fiedler2024safety} for an overview. Notably, the bound from \citep[Theorem~3.11]{abbasi2013online} is often used in this setting since it provides a valid error bound under (conditionally independent) sub-Gaussian noise, tighter than many related bounds, e.g., \citep[Theorem~6]{Srinivas2009}, \citep[Theorem~2]{Chowdhury2017}, \citep[Theorem~1]{Fiedler2021}. Like many related bounds, Abbasi-Yadkori's bound is solely based on a scaled version of the posterior variance $\sigma_t(x)$ \eqref{eq:kervar}, similar to Bayesian uniform error bounds for GPR \cite{Lederer2019}. 
In contrast, the proposed bounds from Theorem~\ref{thm:errorbounds_uniform} leverage the kernel- and data-dependent terms $\rho^2 \lVert \mathbf{h}_t(x)\rVert^2_2$ and $\rho^2 \lVert \mathbf{h}_t(x)\rVert^2_{\infty}$ to bound the uncertainty induced by the noise in the data. 
Figure~\ref{fig:variance_comparison} visualizes the different functional behavior of these terms.
\begin{figure*}[!t]
    \includegraphics[page=4, clip, trim=2cm 18.75cm 2cm 6.2cm, width=\textwidth]{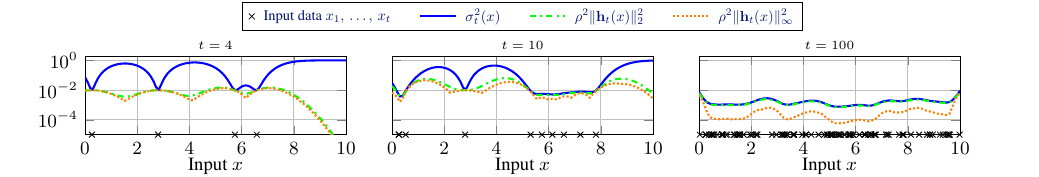}
    \caption{Comparison of the posterior variance $\sigma_t^2(x)$ from~\eqref{eq:kervar} and the kernel- and data-dependent terms $\rho^2\lVert \mathbf{h}_t(x)\rVert^2_2$ and $\rho^2\lVert \mathbf{h}_t(x)\rVert^2_\infty$, leveraged in the proposed bounds from Theorem~\ref{thm:errorbounds_uniform} using $k(x,x') = \exp{-(x-x')^2}$ and $\rho=0.1$.}
    \label{fig:variance_comparison}
\end{figure*}
\begin{figure*}[!t]
    \includegraphics[page=4, clip, trim=1.9cm 6.55cm 2.1cm 17.1cm, width=\textwidth]{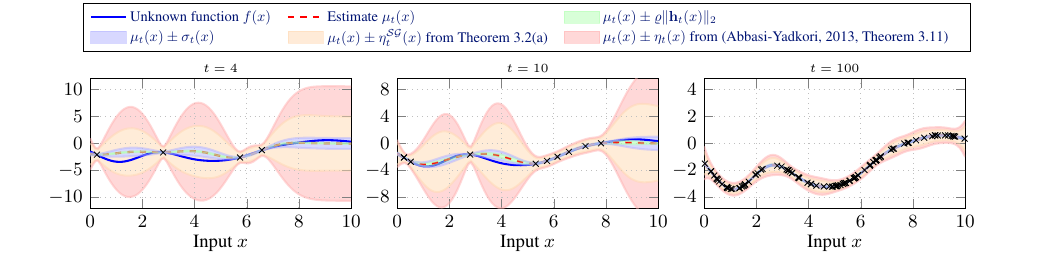}
    \caption{Comparison of confidence regions and error bands in a regression example under sub-Gaussian noise using $k(x,x') = \exp{-(x-x')^2}$, $\sigma_{M} = \rho=0.1$, $\delta = 0.001$, and $B = 5$. The black crosses represent the collected data points $(x_1,y_1),\dots,(x_t,y_t)$. 
    }
    \label{fig:regression}
\end{figure*}
As can be seen, the posterior variance $\sigma_t(x)$ is large in yet unexplored regions, whereas the terms $\rho^2 \lVert \mathbf{h}_t(x)\rVert^2_2$ and $\rho^2 \lVert \mathbf{h}_t(x)\rVert^2_{\infty}$ are comparably small. When the input domain is densely sampled, the variance $\sigma_t(x)$ approaches $\rho^2 \lVert \mathbf{h}_t(x)\rVert^2_2$. Moreover, Figure~\ref{fig:regression} depicts the corresponding confidence regions and error bands in a regression example under sub-Gaussian noise. There, it can be seen that the $1-\delta$ naive confidence region $\mu_t(x) \pm \sigma_t(x)$ is fully contained in the $1-\delta$ uniform error band $\mu_t(x) \pm \eta^{\mathcal{SG}}_t(x)$, and that $\mu_t(x) \pm \sigma_t(x)$ does not necessarily cover the ground-truth function $f(x)$.

The term $\rho^2 \lVert \mathbf{h}_t(x)\rVert^2_2$ is also found in the nonuniform error bounds by Reed et al.~\yrcite{reed2025error} and Molodchyk et al.~\yrcite{molodchyk2025towards}, as well as in the uniform error bound from \citep[Propositon 2]{Fiedler2021}, valid for i.i.d. sub-Gaussian noise. However, the scaling factor of Fiedler et al.'s bound increases with the square root of the number $t$ of data points, in contrast to the favorable logarithmic growth in the scaling factors of the proposed bounds in Theorem~\ref{thm:errorbounds_uniform}(a)--(b). Furthermore, in contrast to Fiedler et al.'s bound, the proposed sub-Gaussian bound from Proposition~\ref{prop:cor_subG} extends beyond the independence assumption and allows for correlated noise. 

A crucial difference between the proposed uniform bounds from Theorem~\ref{thm:errorbounds_uniform} and the related bounds mentioned above is that we rely on a discretization approach and Hölder continuous kernels, similar to Bayesian error bounds for GPR \cite{lederer2023gaussian}. The reason for this is that, to obtain a bound on the noise term $\lvert \mathbf{M}_t^{\top}\mathbf{h}_t(x)\rvert$ in \eqref{eq:noisebound_uniform} (see Lemma~\ref{lem:errorbound_general}), related works such as \cite{abbasi2013online,Fiedler2021} employ the Cauchy-Schwarz inequality to separate the decision-dependent term $\mathbf{h}_t(x)$ from the noise $\mathbf{M}_t$. Despite introducing conservatism, this readily ensures that the error bounds hold uniformly over $x \in \mathcal{X}$ after applying concentration inequalities for the noise $\mathbf{M}_t$. In contrast, in Theorem~\ref{thm:errorbounds_uniform}, we directly apply concentration inequalities to the decision-dependent noise term $\lvert \mathbf{M}_t^{\top}\mathbf{h}_t(x)\rvert$. This avoids the use of the conservative Cauchy-Schwarz inequality, but renders the obtained error bound nonuniform as in~\eqref{eq:errorbound_nonuniform}.
Thus, we employ a discretization approach to derive uniform bounds as in \eqref{eq:errorbound_uniform}. We refer to Appendix~\ref{app:compare_errorbounds} for a more technical discussion of the issue. Numerical experiments suggest that the proposed approach introduces less conservatism than approaches that leverage the Cauchy-Schwarz inequality to separate the noise term; see Section~\ref{sec:eval}.

\paragraph{Beyond Sub-Gaussian Noise}
Theorem~\ref{thm:errorbounds_uniform}(c)--(d) and Proposition~\ref{prop:errorbound_HT} provide error bounds under more general noise distributions, widening the field of application. 
In the context of nonparametric regression with heavy-tailed noise distributions, nonasymptotic bounds on the expected excess risk have been proposed by, e.g., Lederer~\yrcite{lederer2020risk}, Fan et al.~\yrcite{fan2024noise}, and Mollenhauer et al.~\yrcite{mollenhauer2025regularized}; however, regression error bounds of the form \eqref{eq:errorbound_uniform} have not been investigated. Chowdhury \& Gopalan~\yrcite{Chowdhury2019bayesian} and Proposition~\ref{prop:errorbound_HT} provide alternative error bounds under heavy‑tailed distributions, including variance-bounded RVs. However, these results rely on output statistics rather than purely noise statistics, leveraging a truncation approach to exploit sub-Gaussian bounds as from \cite{Chowdhury2017} and Theorem~\ref{thm:errorbounds_uniform}(a). In order to obtain the output statistics from the noise statistics, additional bounds on the unknown function $f$ are required, see \eqref{eq:momentbound}. In contrast, the proposed bounds from Theorem~\ref{thm:errorbounds_uniform}(c)--(d) rely solely on the noise statistics.

Notably, for the considered error bounds of the form~\eqref{eq:errorbound_uniform}, we focus on the estimate \eqref{eq:kermean} that results from the KRR problem~\eqref{eq:krr_problem}. The focus on this widely employed estimate \eqref{eq:kermean} allows us to directly compare the proposed bounds with bounds from the literature; see the discussion above. However, under heavy-tailed noise, the square loss in \eqref{eq:krr_problem} can render the estimate~\eqref{eq:kermean} sensitive to outliers~\cite{huber1981robuststats}, although proper regularization can provide a certain degree of robustness~\cite{mollenhauer2025regularized}. An investigation of more general KRR problems and alternative estimates, e.g., resulting from a medians-of-means approach~\cite{kanakeri2025outlier}, is out of the scope of this work.

\section{Numerical Experiments} \label{sec:eval}
Here, we numerically evaluate the proposed uniform error bounds. In particular, we compare our bounds from Theorem~\ref{thm:errorbounds_uniform}(a)--(b) with the bounds from \citep[Theorem~3.11]{abbasi2013online} and \citep[Proposition~2]{Fiedler2021} in Sections~\ref{sec:eval_confidence}--\ref{sec:eval_control}, and we compare our bounds from Theorem~\ref{thm:errorbounds_uniform}(c)--(d) and Proposition~\ref{prop:errorbound_HT} with the bound from \citep[Lemma 8]{Chowdhury2019bayesian} in Section~\ref{sec:eval_subE}.

\subsection{Size of Uncertainty Region} \label{sec:eval_confidence}
We first compare the size of the uncertainty region corresponding to a probabilistic uniform error bound ${\eta}_t$ as in \eqref{eq:errorbound_uniform} for an increasing number of data points $t \in \{1,\dots,1\,000\}$. We consider i.i.d. bounded noise $M_i\in\mathcal{L}^{\infty}(\overline{m})$, $i \in \mathbb{N}_t$, with bound $\overline{m} \in \{0.01,\,0.1,\,1\}$, standard deviation $\overline{\sigma} = \overline{m}/5$, and variance proxy $\sigma_M = \overline{m}$ \cite{wainwright2019high}. The grid constants $\zeta_t$ in Theorem~\ref{thm:errorbounds_uniform}(a)--\ref{thm:errorbounds_uniform}(b) are chosen such that the discretization terms are constant, i.e., $\Delta_t^{\mathcal{SG}} = \Delta_{1,t}^{\mathrm{bnd}} = \Delta_{2,t}^{\mathrm{bnd}} = 0.001$ for all $t \in \mathbb{N}_{1\,000}$. We use the confidence parameter $\delta = 0.001$ and the common choice $\rho = \sigma_M$ of the regularization parameter \cite{kanagawa2025gaussian} for all considered methods. For the unknown function $f$, we consider the RKHS-norm bound $B=5$ and the squared exponential kernel $k(x,x') = \exp{-\lVert x-x'\rVert_2^2/l_{\mathrm{SE}}^2}$ with lengthscale $l_{\mathrm{SE}}=1$ and input domain $\mathcal{X} = [0, \,r]^{n_x}$ with $r=10$. 
Since our proposed bounds from Theorem~\ref{thm:errorbounds_uniform}(a)--(b) depend on the dimension $n_x$ of the input domain via \eqref{eq:covering_ub}, we evaluate the bounds for $n_x\in \{1,2,3\}$. In $100$ Monte Carlo runs of the learning process, we draw $1\,000$ input data points $x_1, \dots, x_{1\,000}$ uniformly over the domain $\mathcal{X}$ and evaluate the size of the uncertainty region of every method. 

\begin{figure*}[!t]
    \includegraphics[page=2, clip, trim=2cm 6.7cm 2cm 17.9cm, width=\textwidth]{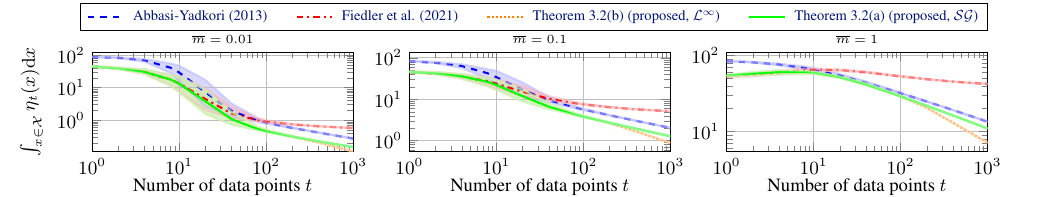}
    \caption{Comparison of the size of the uncertainty region via the integral of the probabilistic uniform error bound ${\eta}_t$ \eqref{eq:errorbound_uniform} over the input domain $\mathcal{X} =[0,\,10]$ for an increasing number of data points $t$, with input dimension $n_x = 1$, noise level $\overline{m}\in\{0.01,\,0.1,\,1\}$, and kernel $k(x,x') = \exp{-\lVert x-x'\rVert_2^2}$. The shaded areas show the $5\%$--$95\%$ percentile range over $100$ Monte Carlo data collection runs.}
    \label{fig:size_errbnd_sigma}
\end{figure*}

\begin{figure*}[!t]
    \includegraphics[page=2, clip, trim=2cm 18.3cm 2cm 6.35cm, width=\textwidth]{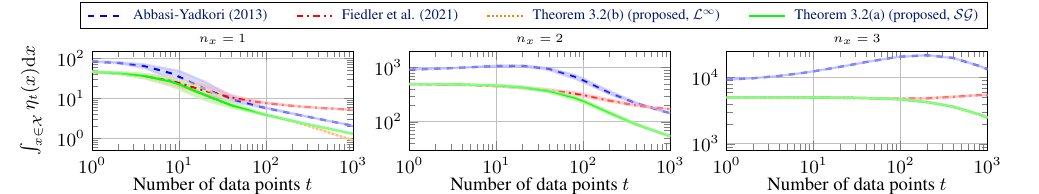}
    \caption{Comparison of the size of the uncertainty region via the integral of the probabilistic uniform error bound ${\eta}_t$ \eqref{eq:errorbound_uniform} over the input domain $\mathcal{X} = [0,\,10]^{n_x}$ for an increasing number of data points $t$, with input dimension $n_x\in\{1,\,2,\,3\}$, noise level $\overline{m}=0.1$ and kernel $k(x,x') = \exp{-(x-x')^2}$. The shaded areas show the $5\%$--$95\%$ percentile range over $100$ Monte Carlo data collection runs.}
    \label{fig:size_errbnd_nx}
\end{figure*}

Figure~\ref{fig:size_errbnd_sigma} depicts the results for varying noise magnitude $\overline{m}$, whereas the results for varying input dimension $n_x$ are shown in Figure~\ref{fig:size_errbnd_nx}. In the considered setting, the proposed uniform bounds outperform the bound from Abbasi-Yadkori~\yrcite{abbasi2013online} in all scenarios. The relative improvement increases with the input dimension, despite the explicit dependence of the parameters of the proposed bounds on the input dimension (see Table~\ref{tab:bounds}). The reason for this is that Abbasi-Yadkori's bound expresses the noise-induced uncertainty solely via the GP posterior variance $\sigma^2_t(x)$, which is
overly conservative in quantifying this uncertainty in yet unexplored regions of the input domain.
In contrast, the proposed bounds and the bound from Fiedler et al.~\yrcite{Fiedler2021} leverage the favorable term $\rho^2 \lVert \mathbf{h}_t\rVert^2_2$; see Appendix~\ref{app:addeval} for additional discussions and experiments on varying input size $r$, discretization terms $\Delta_t$, lengthscale $l_{\mathrm{SE}}$, and kernels.

It is also evident that there are instances, specifically when the number of data points is low, where the bound from Fiedler et al.~\yrcite{Fiedler2021} slightly outperforms the proposed bounds by up to a $7.5\%$ reduction in the mean uncertainty size. However, since the scaling factor of Fiedler et al.'s bound increases according to $\sqrt{t}$, performance rapidly deteriorates as the number of data points grows.

For most learning iterations in the considered scenarios, the proposed uniform bound~\eqref{eq:errorbound_bnd_bernstein_uniform}  tailored to bounded noise (Theorem~\ref{thm:errorbounds_uniform}(b)) coincides with the sub-Gaussian bound~\eqref{eq:errorbound_subG_uniform}  from Theorem~\ref{thm:errorbounds_uniform}(a). However, specifically when $n_x = 1$ and the number of data points is large, the bound~\eqref{eq:errorbound_bnd_bernstein_uniform} dominates and thus yields a tighter bound than~\eqref{eq:errorbound_subG_uniform}. 

\subsection{Safe Control of Uncertain Systems} \label{sec:eval_control}
We now evaluate the proposed error bound from Theorem~\ref{thm:errorbounds_uniform}(a) in the context of safe control, adopting the benchmark from Lahr et al.~\yrcite{lahr2025optimal}. Specifically, consider the dynamical system $x_{t+1} = f^{\mathrm{nom}}(x_t,u_t) + f(x_t)$ with state $x_t \in \mathbb{R}$, control input $u_t \in \mathbb{R}$, known nominal dynamics $f^{\mathrm{nom}}(x_t,u_t) \doteq 0.5 x_t + u_t - 1$, and unknown residual dynamics $f(x_t) \doteq  \exp{-x_t^2}\sin(10x_t)$. Given the current state $x_t$ at time $t \in \mathbb{N}_{1\,000}$, the goal is to find an optimal control input $u_t^*$ that minimizes the cost $J(x_t,u_t) \doteq (f^{\mathrm{nom}}(x_t,u_t) + f(x_t))^2 + u_t^2$ subject to the safety-critical constraint $f^{\mathrm{nom}}(x_t,u_t) + f(x_t) \ge 0.05x_t$ with the input constraint $\lvert u_t\rvert \le 2.05$. Since the function $f(x_t)$ is unknown, we approximate it using the mean predictor $\mu_t(x_t)$ from \eqref{eq:kermean} and the kernel $k(x,x') = \exp{-(x - x')^2/l_{\mathrm{SE}}^2}$ with $l_{\mathrm{SE}} = \sqrt{2}/20$. Data of $f$ is obtained according to \eqref{eq:system} with i.i.d. sub-Gaussian noise $M_i\in\mathcal{SG}(0.1)$.
To ensure satisfaction of the safety constraint with high probability despite the uncertainty in $f(x)$, we employ the reformulation $f^{\mathrm{nom}}(x_t,u_t) + \mu_t(x_t) - {\eta}_t(x_t) \ge 0.05x_t$ using a probabilistic uniform error bound ${\eta}_t$ of the form \eqref{eq:errorbound_uniform} with $\delta = 0.001$. Specifically, we compare the success rate and the inferred costs of this safe control algorithm for an increasing amount of data when implementing the proposed bound from Theorem~\ref{thm:errorbounds_uniform}(a) and the bounds from Abbasi-Yadkori~\yrcite{abbasi2013online} and Fiedler et al.~\yrcite{Fiedler2021} using $B=1$. The success rate is measured in terms of the ratio of feasible problems on a grid of $1\,000$ test points in the domain $\lvert x_t \rvert \le 2$.

Figure~\ref{fig:successrate} depicts the evolution of the success rate over the learning steps for the considered error bounds. Using Abbasi-Yadkori's bound for the control algorithm yields the worst success rate. In contrast, both the proposed bound and Fiedler et al.'s bound yield a substantially higher success rate even for a low amount of data. Although the success rate via Fiedler et al.'s bound is slightly higher ($<1.1\%$) than that of the proposed bound in the low-data regime, our result quickly outperforms all comparison methods when many data points are available. This is again due to the $\sqrt{t}$-like growth of the scaling factor in Fiedler et al.'s bound.

In Figure~\ref{fig:cost}, we compare the inferred costs of the corresponding optimal safe control inputs $u_t^*(x_t)$ over the state domain at $t = 1\,000$. As depicted, using the proposed bound does not only lead to a larger feasible region (i.e., higher success rate), but also to lower inferred costs, close to the optimal costs that result from perfect model knowledge.

\begin{figure}[!t]
    \includegraphics[page=3, clip, trim=6.6cm 18.25cm 6.9cm 6.45cm, width=0.5\textwidth]{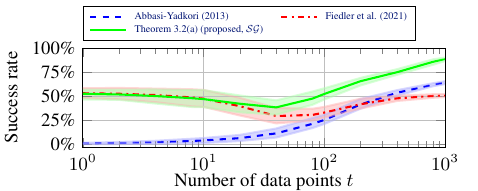}
    \caption{Comparison of the success rate of the safe optimal control algorithm over the number of data points under bounded noise. 
    Specifically when lots of data is available, the proposed bound from Theorem~\ref{thm:errorbounds_uniform} outperforms the comparison methods.
    }
    \label{fig:successrate}
\end{figure}

\begin{figure}[!t]
    \includegraphics[page=3, clip, trim=6.6cm 6.85cm 7cm 18.3cm, width=0.5\textwidth]{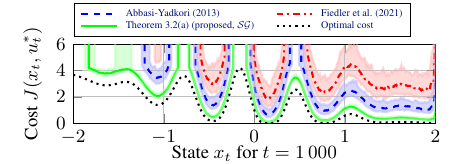}
    \caption{Comparison of the inferred cost over the state domain at the final learning instant ($t=1\,000$), resulting from the optimal safe control $u_t^*$ based on various error bounds. 
    The proposed method yields the largest feasible region and lowest inferred costs.
    }
    \label{fig:cost}
\end{figure}

\subsection{Beyond Sub-Gaussian Noise} \label{sec:eval_subE}
We consider the setting from Section~\ref{sec:eval_control} again, but under i.i.d. chi-squared noise $M_i~\in \mathcal{SE}(4\sigma_M^2,\, 4\sigma_M^2)$ with $\var{M_i} = 2\sigma_M^4$ and $\sigma_M=0.1$. Since the previously discussed error bounds cannot be applied, we employ the proposed bounds from Theorem~\ref{thm:errorbounds_uniform}(c)--(d), Proposition~\ref{prop:errorbound_HT}, and \citep[Lemma~8]{Chowdhury2019bayesian}, where the latter leverage the additional information $\lvert f(x_t)\rvert <2$ (cf.~\eqref{eq:momentbound}); see Appendix~\ref{app:addeval} for implementation details.

Figure~\ref{fig:size_errbnd_subE} depicts the size of the uncertainty region corresponding to the respective error bound over the learning steps (see Section~\ref{sec:eval_confidence}). The uncertainty regions corresponding to both Theorem~\ref{thm:errorbounds_uniform}(d) and Chowdhury \& Gopalan's bound blow up for an increasing amount of data. Although the bound from Theorem~\ref{thm:errorbounds_uniform}(d) blows up more rapidly, it outperforms Chowdhury \& Gopalan's bound in the low data regime. In contrast, the uncertainty region corresponding to the bound from Proposition~\ref{prop:errorbound_HT} stagnates for an increasing amount of data and also outperforms both Theorem~\ref{thm:errorbounds_uniform}(d) and Chowdhury \& Gopalan's bound over all learning iterations. Nevertheless, leveraging the proposed bound from Theorem~\ref{thm:errorbounds_uniform}(c) yields a reduction in uncertainty size of over two orders of magnitude compared to  Chowdhury \& Gopalan's bound, and over an order of magnitude compared to Proposition~\ref{prop:errorbound_HT}. Crucially, note that Theorem~\ref{thm:errorbounds_uniform}(d) assumes much less knowledge than the comparison bounds (merely an upper bound on the noise variance $\var{M_i}$), and is thus applicable in more general cases than related bounds.

Figure~\ref{fig:successrate_subE} shows the evolution of the success rate of the safe control algorithm (see Section~\ref{sec:eval_control}) over the learning steps for the proposed bounds from Theorem~\ref{thm:errorbounds_uniform}(c) and Proposition~\ref{prop:errorbound_HT}. Due to their conservatism, no successful run is achieved when using the bounds from Theorem~\ref{thm:errorbounds_uniform}(d) and Chowdhury \& Gopalan~\yrcite{Chowdhury2019bayesian}. For the bound from Proposition~\ref{prop:errorbound_HT}, the success rate drops to $0\%$ for $t > 50$ due to increasing conservatism. In contrast, employing the proposed bound from Theorem~\ref{thm:errorbounds_uniform}(c) yields a non-zero success rate for all learning steps, terminating at $71.6\%$. 

\begin{figure}[!t]
    \includegraphics[page=8, clip, trim=6.6cm 6.75cm 6.9cm 17.85cm, width=0.5\textwidth]{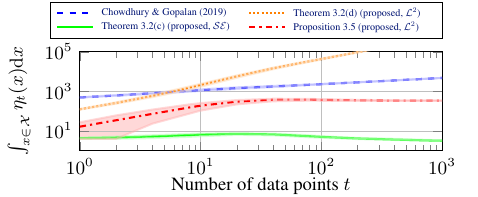}
    \caption{Comparison of the size of the uncertainty region via the integral of the error bound ${\eta}_t$ \eqref{eq:errorbound_uniform} under sub-exponential noise. 
    Leveraging the proposed sub-exponential bound from Theorem~\ref{thm:errorbounds_uniform}(c) yields a significant reduction of uncertainty.
    }
    \label{fig:size_errbnd_subE}
\end{figure}

\begin{figure}[!t]
    \includegraphics[page=8, clip, trim=6.6cm 18.45cm 6.9cm 6.4cm, width=0.5\textwidth]{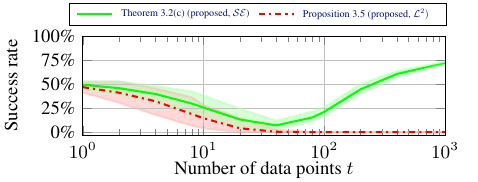}
    \caption{Success rate of the safe optimal control algorithm over the number of data points under sub-exponential noise. 
    No successful run is achieved when using the bounds from Theorem~\ref{thm:errorbounds_uniform}(d) and Chowdhury \& Gopalan~\yrcite{Chowdhury2019bayesian}.
    }
    \label{fig:successrate_subE}
\end{figure}

\section{Conclusions} \label{sec:conclusion}
We have presented novel probabilistic uniform error bounds for kernel-based regression tailored to specific classes of noise distributions. The proposed framework enables the derivation of uniform error bounds for various distribution classes, including potentially correlated and non-Gaussian noise, thereby widening the field of applications. For the commonly considered case of bounded/sub-Gaussian noise as well as for sub-exponential noise, the numerical results demonstrate the competitiveness of the proposed uniform error bounds in safety-critical contexts, specifically when a large amount of data is available.

\clearpage

\section*{Acknowledgements}
This work has been partially funded by Deutsche Forschungsgemeinschaft (DFG, German Research Foundation) via the TRR 391 Spatio-temporal Statistics for the Transition of Energy and Transport (project 520388526) and the research unit Active Learning for Systems and Control (ALeSCo) (project 535860958), and by a start-up grant of the National University of Singapore (A-0010215-00-00). Johannes Teutsch gratefully acknowledges financial support from the German Academic Scholarship
Foundation.

\section*{Impact Statement}
This paper presents work whose goal is to advance the field
of Machine Learning. There are many potential societal
consequences of our work, none which we feel must be
specifically highlighted here.

\bibliography{references.bib}
\bibliographystyle{icml2026}

\newpage
\appendix
\onecolumn

\section{Proofs}
This appendix provides detailed proofs of our proposed results.

\subsection{Proof of Lemma~\ref{lem:errorbound_general}} \label{app:proof_generalbound}
\begin{proof}
    Let us first decompose the available output data $\mathbf{y}_t$ (generated via \eqref{eq:system}) into noise-free component $\mathbf{f}_t$ and noise component $\mathbf{M}_t(\omega)$, i.e., $\mathbf{y}_t = \mathbf{f}_t + \mathbf{M}_t(\omega)$. By the triangle inequality, the regression error is bounded by
    \begin{align}
         \lvert f(x) - \mu_t(x) \rvert &=  \lvert f(x)  - \left(\mathbf{f}_t + \mathbf{M}_t(\omega)\right)^{\top}\mathbf{h}_t(x) \rvert \le \lvert f(x) - \mathbf{f}_t^{\top}\mathbf{h}_t(x) \rvert + \lvert \mathbf{M}_t(\omega)^{\top}\mathbf{h}_t(x) \rvert \label{eq:bound_triangle}
    \end{align}
    with $\mathbf{h}_t(x) = \left(\mathbf{K}_t+\rho^2\mathbf{I}_t\right)^{-1}\mathbf{k}_t(x)$.
    The noise term $\lvert \mathbf{M}_t(\omega)^{\top}\mathbf{h}_t(x) \rvert$ is probabilistically bounded via \eqref{eq:noisebound_uniform}. Thus, it remains to show that $\lvert f(x) - \mathbf{f}^{\top}\mathbf{h}_t(x) \rvert \le B \tilde{\sigma}_t(x)$, with $\tilde{\sigma}_t(x) =\sqrt{\sigma^2_{t}(x) - \rho^2 \norm{\mathbf{h}_t(x)}_2^2}$ and $\sigma^2_{t}(x)$ from~\eqref{eq:kervar}.
    
    First, let us denote $\mathbf{h}_t(x) = [h_1(x), \dots, h_t(x)]^{\top}$. Due to Assumption~\ref{asm:system}, the reproducing property holds, i.e., for any $x \in \mathcal{X}$ and any $f \in \mathcal{H}_k$, the value $f(x)$ can be reproduced via $f(x) = \langle f, k(\cdot, x)\rangle_{\mathcal{H}_k}$. Thus, we can write $\mathbf{f}^{\top}\mathbf{h}_t(x) = \sum_{i=1}^t f(x_i) h_i(x) = \sum_{i=1}^t \inprod{f, k(\cdot,x_i)}_{\mathcal{H}_k} h_i(x)$. Thus, using the Cauchy-Schwarz inequality, we obtain
\begin{equation}
        \lvert f(x) - \mathbf{f}^{\top}\mathbf{h}_t(x)\rvert^2 = \lvert\inprod{f, k(\cdot, x)}_{\mathcal{H}_k} - \sum\limits_{i=1}^t \inprod{f, k(\cdot,x_i)}_{\mathcal{H}_k} h_i(x)\rvert^2 \le \norm{f}_{\mathcal{H}_k}^2 \norm{k(\cdot, x) - \textstyle\sum\limits_{i=1}^t k(\cdot,x_i) h_i(x)}^2_{\mathcal{H}_k}. \label{eq:helper_rkhsbound}
\end{equation}
Following \citep[Lemma~3]{molodchyk2025towards}, the second term on the right-hand side of \eqref{eq:helper_rkhsbound} can be expanded as
\begin{align}
        &\norm{k(\cdot, x) - \textstyle\sum\limits_{i=1}^t k(\cdot,x_i) h_i(x)}^2_{\mathcal{H}_k} \notag\\
        &\hspace{1cm}=\norm{k(\cdot, x)}^2_{\mathcal{H}_k} - 2 \textstyle\sum\limits_{i=1}^t \inprod{k(\cdot, x), k(\cdot,x_i)}_{\mathcal{H}_k} h_i(x) + \textstyle\sum\limits_{i,j=1}^t \inprod{k(\cdot,x_i), k(\cdot,x_j)}_{\mathcal{H}_k} h_i(x) h_j(x) \notag\\
        &\hspace{1cm}= k(x, x) - 2 \mathbf{k}_t(x)^\top \mathbf{h}_t(x) \, + \mathbf{h}_t(x)^\top \mathbf{K}_t \mathbf{h}_t(x) \notag\\
        &\hspace{1cm}= k(x, x) - 2 \mathbf{k}_t(x)^\top \left(\mathbf{K}_t + \rho^2 \mathbf{I}_t\right)^{-1} \mathbf{k}_t(x) + \mathbf{k}_t(x)^\top \left(\mathbf{K}_t + \rho^2 \mathbf{I}_t\right)^{-1}  \mathbf{K}_t \left(\mathbf{K}_t + \rho^2 \mathbf{I}_t\right)^{-1} \mathbf{k}_t(x) \notag\\
        &\hspace{1cm}\stackrel{\text{\eqref{eq:kervar}}}{=} \sigma_t^2(x) - \mathbf{k}_t(x)^\top \left(\mathbf{K}_t + \rho^2 \mathbf{I}_t\right)^{-1} \mathbf{k}_t(x) + \mathbf{k}_t(x)^\top \left(\mathbf{K}_t + \rho^2 \mathbf{I}_t\right)^{-1}  \mathbf{K}_t \left(\mathbf{K}_t + \rho^2 \mathbf{I}_t\right)^{-1} \mathbf{k}_t(x) \notag\\
        &\hspace{1cm}=\sigma_t^2(x) - \mathbf{k}_t(x)^\top \left(\mathbf{K}_t + \rho^2 \mathbf{I}_t\right)^{-1} \left(\mathbf{K}_t + \rho^2 \mathbf{I}_t\right) \left(\mathbf{K}_t + \rho^2 \mathbf{I}_t\right)^{-1} \mathbf{k}_t(x) \notag\\
        &\hspace{3cm}+ \mathbf{k}_t(x)^\top \left(\mathbf{K}_t + \rho^2 \mathbf{I}_t\right)^{-1}  \mathbf{K}_t \left(\mathbf{K}_t + \rho^2 \mathbf{I}_t\right)^{-1} \mathbf{k}_t(x) \notag\\
        &\hspace{1cm}= \sigma_t^2(x) - \mathbf{k}_t(x)^\top \left(\mathbf{K}_t + \rho^2 \mathbf{I}_t\right)^{-1} (\left(\mathbf{K}_t + \rho^2 \mathbf{I}_t\right) - \mathbf{K}_t) \left(\mathbf{K}_t + \rho^2 \mathbf{I}_t\right)^{-1} \mathbf{k}_t(x) \notag\\
        &\hspace{1cm}= \sigma_t^2(x) - \rho^2\mathbf{k}_t(x)^\top \left(\mathbf{K}_t + \rho^2 \mathbf{I}_t\right)^{-1} \mathbf{I}_t\left(\mathbf{K}_t + \rho^2 \mathbf{I}_t\right)^{-1} \mathbf{k}_t(x) \notag\\
        &\hspace{1cm} =\sigma_t^2(x) -  \rho^2 \lVert \mathbf{h}_t(x)\rVert_2^2. \label{eq:helper_rkhsbound_reform}
    \end{align}
    Hence, by leveraging \eqref{eq:helper_rkhsbound_reform} and $\norm{f}_{\mathcal{H}_k}\le B$ from Assumption~\ref{asm:system} and by taking the square root on both sides of \eqref{eq:helper_rkhsbound}, we obtain $\lvert f(x) - \mathbf{f}^{\top}\mathbf{h}_t(x) \rvert \le B \tilde{\sigma}_t(x)$,
    and combination with \eqref{eq:bound_triangle} yields
    \begin{equation}
        \lvert f(x) - \mu_t(x) \rvert \le B \tilde{\sigma}_{t}(x) + \lvert \mathbf{M}_t(\omega)^{\top}\mathbf{h}_t(x) \rvert. \label{eq:bound_deterministic}
    \end{equation}
    Lastly, the assertion follows combining \eqref{eq:bound_deterministic} and \eqref{eq:noisebound_uniform}.
\end{proof}

\subsection{Proof of Theorem \ref{thm:errorbounds_uniform}} \label{app:proof_uniform}
Before we prove Theorem \ref{thm:errorbounds_uniform}, we first present the corresponding result involving nonuniform error bounds as in~\eqref{eq:errorbound_nonuniform}.

\begin{theorem}[Nonuniform Error Bounds]\label{thm:errorbounds_nonuniform}
    Consider data \eqref{eq:dataset} generated via \eqref{eq:system} under Assumption~\ref{asm:system} with realizations $M_1(\omega), \dots,  M_t(\omega)$ of the noise $M_i$, $i \in \mathbb{N}_t$. 
    Then, for every $\rho > 0$ and $\delta \in (0,1)$, the regression error $\lvert f(x) - \mu_t(x) \rvert$ is bounded by the probabilistic error bound~\eqref{eq:errorbound_nonuniform} with $\tilde{\eta}_t(x) = B \tilde{\sigma}_t(x) + \tilde{\eta}^{M}_t(x)$ and $\tilde{\eta}^{M}_t(x)$ defined as follows:
        
    \paragraph{(a)} If $\mathbf{M}_t \in \mathcal{SG}(\mathbf{C}_t)$ for some $\mathbf{C}_t\succeq\bm{0}$, then $\tilde{\eta}^{M}_t(x) = \tilde{\eta}^{\mathcal{SG}}_t(x)$ with
    \begin{equation} \label{eq:errorbound_subG_nonuniform}
        \tilde{\eta}^{\mathcal{SG}}_t(x) \doteq \sqrt{2\ln(2/\delta)} \norm{\mathbf{h}_t(x)}_{\mathbf{C}_t}.
    \end{equation}
    
    \paragraph{(b)} If $M_i$, $i \in \mathbb{N}_t$, are i.i.d., $M_i \in \mathcal{L}^{\infty}(\overline{m})$, and $\var{M_i} = \overline{\sigma}^2$, then $\tilde{\eta}^{M}_t(x) = \min_{j\in\{1,2\}}\tilde{\eta}^\mathrm{bnd}_{j,t}(x)$ with 
    \begin{subequations}
        \begin{align} 
            \tilde{\eta}^\mathrm{bnd}_{1,t}(x) &\doteq  \sqrt{2\ln(2/\delta)}~ \overline{m} \lVert \mathbf{h}_t(x) \rVert_2,\label{eq:errorbound_bnd_hoeffding}
            \\
            \tilde{\eta}^\mathrm{bnd}_{2,t}(x) &\doteq \frac{2}{3}\ln(2/\delta)~ \overline{m} \lVert\mathbf{h}_t(x)\rVert_\infty + \sqrt{2\ln(2/\delta)} \overline{\sigma} \lVert \mathbf{h}_t(x) \rVert_2.\label{eq:errorbound_bnd_bernstein}
        \end{align}
    \end{subequations}
    
    \paragraph{(c)} If $M_i$, $i \in \mathbb{N}_t$, are i.i.d. and $M_i \in \mathcal{SE}(\nu_M^2,\alpha_M)$, then $\tilde{\eta}^{M}_t(x) = \tilde{\eta}^\mathcal{SE}_{t}(x)$ with
    \begin{align} \label{eq:errorbound_subE_nonuniform}
        \tilde{\eta}^\mathcal{SE}_{t}(x) \doteq \max&\left\{2\ln(2/\delta)\, \alpha_M \lVert \mathbf{h}_t(x)\rVert_\infty,~ \sqrt{2\ln(2/\delta)}\, \nu_M \lVert \mathbf{h}_t(x)\rVert_2\right\}.
    \end{align}
    
    \paragraph{(d)} If $M_i$, $i \in \mathbb{N}_t$, are i.i.d. and $M_i \in \mathcal{L}^2(\sigma_M^2)$, then $\tilde{\eta}^{M}_t(x) = \tilde{\eta}^{\mathcal{L}2}_{t}(x)$ with
    \begin{equation} \label{eq:errorbound_L2_nonuniform}
        \tilde{\eta}^{\mathcal{L}2}_{t}(x) \doteq  \sqrt{1/\delta}~ \sigma_M \lVert \mathbf{h}_t(x)\rVert_2.
    \end{equation}
\end{theorem}

\begin{proof}
    Choose any regularization parameter $\rho > 0$, input $x \in \mathcal{X}$, time/number of data points $t \in \mathbb{N}$, and confidence parameter $\delta \in (0,1)$. Further, we denote $\mathbf{h}_t(x) = [h_1(x), \dots, h_t(x)]^{\top}$, which allows us to write $\mathbf{M}_t^{\top}\mathbf{h}_t(x) = \sum_{i=1}^t h_i(x) M_i$.\\
    For all cases \textbf{(a)}--\textbf{(d)}, the assertions follow from \eqref{eq:bound_deterministic} by deriving suitable functions $\tilde{\eta}_t^M$ satisfying the nonuniform bound
    \begin{equation}
        \forall x \in \mathcal{X},~t\in\mathbb{N}:~ \mathbb{P}\left[\lvert\mathbf{M}_t^{\top}\mathbf{h}_t(x)\rvert \le \tilde{\eta}_t^{M}(x)\right] \geq 1 - \delta. \label{eq:noisebound_nonuniform}
    \end{equation}
    
    \paragraph{(a)} Since $\mathbf{M}_t \in \mathcal{SG}(\mathbf{C}_t)$, we have $\mathbf{M}_t^{\top}\mathbf{h}_t(x) \in \mathcal{SG}(\norm{\mathbf{h}_t(x)}_{\mathbf{C}_t}^2$) due to the properties of vector-values sub-Gaussian random variables~\citep[Theorem~1a]{ao2025stochastic}. From the concentration inequality of sub-Gaussian random variables \citep[Sec.~2.1]{wainwright2019high}, it follows that 
    \begin{equation*}
        \mathbb{P}\left[\lvert\mathbf{M}_t^{\top}\mathbf{h}_t(x)\rvert \le c\right] \geq 1-2\exp{-\frac{c^2}{2\norm{\mathbf{h}_t(x)}_{\mathbf{C}_t}^2}}
    \end{equation*}
    for all $c > 0$. Introducing the confidence parameter $\delta \in (0,\,1)$ and setting $c = \tilde{\eta}_t^M(x) = \tilde{\eta}_t^{\mathcal{SG}}(x) = \sqrt{2\ln(2/\delta)} \norm{\mathbf{h}_t(x)}_{\mathbf{C}_t}$, we obtain the nonuniform probabilistic bound \eqref{eq:noisebound_nonuniform}. 
    
    \paragraph{(b)} 
    The sub-Gaussian bound~\eqref{eq:errorbound_bnd_hoeffding} follows from \textbf{(a)} by setting the variance proxy to $ \mathbf{C}_t = \overline{m}^2\mathbf{I}_t$, since $\mathcal{L}^\infty(\overline{m}) \subset \mathcal{SG}(\overline{m})$ \cite{wainwright2019high}, cf. Theorem~1 of Reed et al.~\yrcite{reed2025error}.
    
    Furthermore, from Bernstein's inequality for the sum of bounded random variables \citep[Theorem~2.8.4]{vershynin2018high}, we have that 
    \begin{equation*}
        \prob{\lvert \mathbf{M}_t^{\top}\mathbf{h}_t(x) \rvert \le c} \ge 1 - 2\exp{-\frac{c^2}{2\overline{\sigma}^2 \lVert\mathbf{h}_t(x)\rVert_\infty^2 +\frac{2}{3} \overline{m} \lVert\mathbf{h}_t(x)\rVert_\infty c}}
    \end{equation*}
    for all $c > 0$. Introducing the confidence parameter $\delta \in (0,1)$ and setting
    \begin{equation*}
        \delta = 2\exp{-\frac{c^2}{2\overline{\sigma}^2 \lVert\mathbf{h}_t(x)\rVert_\infty^2 +\frac{2}{3} \overline{m} \lVert\mathbf{h}_t(x)\rVert_\infty c}},
    \end{equation*}
    we obtain the condition
    \begin{equation}
        c^2 - \ln(2/\delta) \frac{2}{3} \overline{m} \lVert\mathbf{h}_t(x)\rVert_\infty c - \ln(2/\delta) 2 \overline{\sigma}^2 \lVert\mathbf{h}_t(x)\rVert_\infty^2 = 0. \label{eq:quadcond_bernstein}
    \end{equation}
    The nonnegative solution to \eqref{eq:quadcond_bernstein} is $c = \tilde{c}(x)$ with
    \begin{equation}
        \tilde{c}(x) \doteq \frac{1}{3}\ln(2/\delta)~ \overline{m} \lVert\mathbf{h}_t(x)\rVert_\infty + \sqrt{ \frac{1}{9}\ln(2/\delta)^2~ \overline{m}^2 \lVert\mathbf{h}_t(x)\rVert_\infty^2 + 2\ln(2/\delta)~\overline{\sigma}^2 \lVert \mathbf{h}_t(x) \rVert_2^2}, \label{eq:bernsteinbound_tight}
    \end{equation}
    thus yielding $\prob{\lvert \mathbf{M}_t^{\top}\mathbf{h}_t(x) \rvert \le \tilde{c}(x)} \ge 1 -\delta$. To simplify the expression of $\tilde{c}(x)$ in \eqref{eq:bernsteinbound_tight}, we apply the triangle inequality 
    $$\sqrt{ (1/9)\ln(2/\delta)^2~ \overline{m}^2 \lVert\mathbf{h}_t(x)\rVert_\infty^2 + 2\ln(2/\delta)~\overline{\sigma}^2 \lVert \mathbf{h}_t(x) \rVert_2^2} \le (1/3)\ln(2/\delta)~ \overline{m} \lVert\mathbf{h}_t(x)\rVert_\infty + \sqrt{2\ln(2/\delta)}~\overline{\sigma} \lVert \mathbf{h}_t(x) \rVert_2,$$
    yielding the bound $\tilde{\eta}_t^M(x) = \tilde{\eta}^{\mathrm{bnd}}_{2,t}(x) \ge \tilde{c}(x)$ from \eqref{eq:errorbound_bnd_bernstein} satisfying the nonuniform probabilistic bound \eqref{eq:noisebound_nonuniform}.

    Since both the Hoeffding-type bound~\eqref{eq:errorbound_bnd_hoeffding} and the Bernstein-type bound~\eqref{eq:errorbound_bnd_bernstein} are valid nonuniform probabilistic bounds \eqref{eq:noisebound_nonuniform}, we can employ $\tilde{\eta}^{M}_t(x) = \min_{j\in\{1,2\}}\tilde{\eta}^\mathrm{bnd}_{j,t}(x)$ to obtain the tightest bound.
        
    \paragraph{(c)} 
    From the Bernstein inequality for the sum of weighted sub-exponential random variables \citep[Theorem~2.8.2]{vershynin2018high}, \citep[Proposition~2.9]{wainwright2019high}, it holds for all $c > 0$ that
    \begin{align*}
        \prob{\lvert \mathbf{M}_t^{\top}\mathbf{h}_t(x) \rvert \le c} &\ge 
        \begin{cases}
            1 - 2 \exp{-\dfrac{c^2}{2 \nu^2 \lVert \mathbf{h}_t(x)\rVert_2^2}}  & \textnormal{for } 0 \le c < \dfrac{\nu^2\lVert\mathbf{h}_t(x)\rVert_2^2}{\alpha \lVert\mathbf{h}_t(x)\rVert_\infty}\\ 
            1 - 2 \exp{-\dfrac{c}{2 \alpha \lVert \mathbf{h}_t(x)\rVert_\infty}} & \textnormal{else}
        \end{cases},\\
        \Longleftrightarrow~~~\prob{\lvert \mathbf{M}_t^{\top}\mathbf{h}_t(x) \rvert \le c}& \ge 1-\max\left\{ 2 \exp{-\frac{c^2}{2 \nu^2 \lVert \mathbf{h}_t(x)\rVert_2^2}},~ 2 \exp{-\frac{c}{2 \alpha \lVert \mathbf{h}_t(x)\rVert_\infty}}\right\}.
    \end{align*}
    Introducing the confidence parameter $\delta \in (0,\,1)$ and setting 
    \begin{equation*}
        \delta = \begin{cases}
            2 \exp{-\dfrac{c^2}{2 \nu^2 \lVert \mathbf{h}_t(x)\rVert_2^2}}  & \textnormal{for } 0 \le c < \dfrac{\nu^2\lVert\mathbf{h}_t(x)\rVert_2^2}{\alpha \lVert\mathbf{h}_t(x)\rVert_\infty}\\ 
            2 \exp{-\dfrac{c}{2 \alpha \lVert \mathbf{h}_t(x)\rVert_\infty}} & \textnormal{else}
        \end{cases}
    \end{equation*}
    we obtain the condition
    \begin{equation*}
        c = \begin{cases}
            \sqrt{2\ln(2/\delta)}\, \nu_M \lVert \mathbf{h}_t(x)\rVert_2  & \textnormal{for } 0 \le \sqrt{2\ln(2/\delta)} < \dfrac{\nu\lVert\mathbf{h}_t(x)\rVert_2}{\alpha \lVert\mathbf{h}_t(x)\rVert_\infty}\\ 
            ~~~~2\ln(2/\delta)\, \alpha_M \lVert \mathbf{h}_t(x)\rVert_\infty & \textnormal{else}
        \end{cases},
    \end{equation*}
    which can be equivalently expressed as $c = \tilde{\eta}^{\mathcal{SE}}_t(x)$ from \eqref{eq:errorbound_subE_nonuniform}, thus yielding the nonuniform probabilistic bound \eqref{eq:noisebound_nonuniform}.
    
    \paragraph{(d)} 
    Since the noise sequence $M_1,\,\dots,\,M_t$ is i.i.d., we have $\mathbf{M}_t^{\top}\mathbf{h}_t(x) \in \mathcal{L}^2(\sigma^2_M\lVert \mathbf{h}_t(x)\rVert^2_2)$.
    From the Chebychev inequality~\citep[Corollary~1.2.5]{vershynin2018high}, it holds for all $c > 0$ that 
    \begin{equation*}
        \prob{\lvert \mathbf{M}_t^{\top}\mathbf{h}_t(x) \rvert \ge c} \ge 1-  \frac{\sigma^2_M\lVert \mathbf{h}_t(x)\rVert^2_2}{c^2}.
    \end{equation*}
    Introducing the confidence parameter $\delta \in (0,\,1)$ and setting $c = \tilde{\eta}_t^M(x) = \tilde{\eta}_t^{\mathcal{L}2}(x) = \sqrt{1/\delta}~ \sigma_M \lVert \mathbf{h}_t(x)\rVert_2$, we obtain the nonuniform probabilistic bound ~\eqref{eq:noisebound_nonuniform}.
\end{proof}

Notably, the error bounds from Theorem~\ref{thm:errorbounds_nonuniform} hold nonuniformly as in~\eqref{eq:errorbound_nonuniform}. In order to obtain uniform counterparts as in~\eqref{eq:errorbound_uniform}, we employ discretization methods from Gaussian Process regression literature~\cite{Srinivas2009,Lederer2019}. As an intermediate step, we require the following technical lemma that translates the Hölder continuity of the kernel (Assumption~\ref{asm:hoelder}) into the Hölder continuity of norms of the data- and decision-dependent vector $\mathbf{h}_t(x)$, which will later allow us to take into account the discretization error.

\begin{lemma} \label{lem:hoelder} 
    Consider data \eqref{eq:dataset}, a kernel $k$ that satisfies Assumption~\ref{asm:hoelder}, and a positive semi-definite matrix $\mathbf{C}\in\mathbb{R}^{t\times t}$, $\mathbf{C} \succeq \bm{0}$, with maximum singular value $\varsigma_{\mathbf{C}} = \lVert \mathbf{C} \rVert_2$. Then, the terms $\rho \lVert \mathbf{h}_t(x)\rVert_2$, $\rho \lVert \mathbf{h}_t(x)\rVert_\infty$, and $\rho \lVert \mathbf{h}_t(x)\rVert_{\mathbf{C}}$ are Hölder continuous with order $\frac{p}{2}$ and constants $\sqrt{\frac{L}{2}}$ and $\sqrt{\varsigma_{\mathbf{C}}\frac{L}{2}}$, i.e., $ \forall x, x' \in \mathcal{X}$, $\forall t \in \mathbb{N}$:
    \begin{subequations}
        \begin{align}
            \lvert \rho \lVert \mathbf{h}_t(x)\rVert_2 - \rho \lVert \mathbf{h}_t(x')\rVert_2\rvert &\le \sqrt{\frac{L}{2}}\lVert x - x' \rVert^{\frac{p}{2}}, \label{eq:hoelder_2norm}\\
            \lvert \rho \lVert \mathbf{h}_t(x)\rVert_\infty - \rho \lVert \mathbf{h}_t(x')\rVert_\infty\rvert &\le \sqrt{\frac{L}{2}}\lVert x - x' \rVert^{\frac{p}{2}},\label{eq:hoelder_infnorm}\\
            \lvert \rho \lVert \mathbf{h}_t(x)\rVert_{\mathbf{C}} - \rho \lVert \mathbf{h}_t(x')\rVert_{\mathbf{C}}\rvert &\le \sqrt{\varsigma_{\mathbf{C}} \frac{L}{2}}\lVert x - x' \rVert^{\frac{p}{2}}.\label{eq:hoelder_2norm_weight}
        \end{align}
    \end{subequations}
\end{lemma}

\begin{proof}
    Consider any $x, x' \in \mathcal{X}$ and $t\in\mathbb{N}$. Recall that for any positive definite kernel $k: \mathcal{X} \times \mathcal{X} \to \mathbb{R}$ from Assumption~\ref{asm:system} there exists some (possibly infinite-dimensional) feature Hilbert space $\mathcal{H}$ and a feature map $\bm{\phi} : \mathcal{X} \to \mathcal{H}$ such that 
    \begin{equation} \label{eq:k_phi}
		k(x, x^\prime) = \langle \bm{\phi}(x), \bm{\phi}(x^\prime) \rangle_\mathcal{H},
	\end{equation}
    where $\langle \cdot, \cdot \rangle_\mathcal{H}$ is an inner product on $\mathcal{H}$. With slight abuse of notation, we can identify the Gramian $\mathbf{K}_t \in \mathbb{R}^{t \times t}$ with the linear operator $\mathbf{K}_t: \mathbb{R}^t \to \mathbb{R}^t$. Similarly, we define the identity operators $\mathbf{I}_t$ and $\mathbf{I}_\mathcal{H}$ on $\mathbb{R}^t$ and $\mathcal{H}$, respectively.
    With the data \eqref{eq:dataset}, we introduce the linear operator $\bm{\Phi}_t: \mathbb{R}^t \to \mathcal{H}$ as
    \begin{equation}
	\bm{\Phi}_t: \mathbb{R}^t \ni a\doteq[a_1 \ldots a_t]^\top \mapsto \mathbf{\Phi}_t a \doteq \sum_{i=1}^{t} a_i \phi(x_i) \in \mathcal{H}. \label{eq:phi_operator}
    \end{equation}
    Since $\mathbf{\Phi}_t$ is finite-rank, it is bounded and there exists a unique adjoint $\mathbf{\Phi}_t^\ast : \mathcal{H} \to \mathbb{R}^t$ satisfying
	\[
		\langle \mathbf{\Phi}_t a, b \rangle_\mathcal{H} = \langle a, \mathbf{\Phi}_t^\ast b \rangle_{\mathbb{R}^t} = a^\top \mathbf{\Phi}_t^\ast b, \qquad \forall \, a \in \mathbb{R}^t, b \in \mathcal{H},
	\]
	where we use $\langle \cdot, \cdot \rangle_{\mathbb{R}^t}$ to denote the dot product on $\mathbb{R}^t$.
	Note that $\mathbf{\Phi}_t^\ast$ is also bounded and satisfies $(\mathbf{\Phi}_t^\ast)^\ast = \mathbf{\Phi}_t$.

    Due to \eqref{eq:k_phi}, it holds that $\mathbf{k}_t(x) = \mathbf{\Phi}_t^\ast \phi(x)$ for all $x \in \mathcal{X}$ and $\mathbf{K}_t = \mathbf{\Phi}_t^\ast \mathbf{\Phi}_t$.
    Then, the term $\sigma^2_{t}(x)$ from \eqref{eq:kervar} can be equivalently written as
	\begin{align}
		\sigma^2_{t}(x) &= k(x,x) - \mathbf{k}_t(x)^\top (\rho^2 \mathbf{I}_t + \mathbf{K}_t)^{-1} \mathbf{k}_t(x)
		= \langle \phi(x), \phi(x) \rangle_\mathcal{H} - \langle \mathbf{k}_t(x), (\rho^2 \mathbf{I}_t + \mathbf{K}_t)^{-1} \mathbf{k}_t(x) \rangle_{\mathbb{R}^t} \notag \\
		& = \langle \phi(x), \phi(x) \rangle_\mathcal{H} - \langle \mathbf{\Phi}_t^\ast \phi(x), (\rho^2 \mathbf{I}_t + \mathbf{K}_t)^{-1} \mathbf{\Phi}_t^\ast \phi(x) \rangle_{\mathbb{R}^t} = \langle \phi(x), \phi(x) \rangle_\mathcal{H} - \langle \phi(x), \mathbf{\Phi}_t(\rho^2 \mathbf{I}_t + \mathbf{K}_t)^{-1} \mathbf{\Phi}_t^\ast \phi(x) \rangle_{\mathcal{H}} \notag\\
		& = \langle \phi(x), \left[\mathbf{I}_\mathcal{H} - \mathbf{\Phi}_t(\rho^2 \mathbf{I}_t + \mathbf{\Phi}_t^\ast \mathbf{\Phi}_t)^{-1} \mathbf{\Phi}_t^\ast \right] \phi(x) \rangle_{\mathcal{H}} \doteq \langle \phi(x), \mathbf{A}_t\phi(x) \rangle_{\mathcal{H}}. \label{eq:kervar_reform}
	\end{align}
    In the last equality, we have defined the (bounded self-adjoint) operator $\mathbf{A}_t: \mathcal{H} \to \mathcal{H}$ as
	$\mathbf{A}_t \doteq \mathbf{I}_\mathcal{H} - \mathbf{\Phi}_t(\rho^2 \mathbf{I}_t + \mathbf{\Phi}_t^\ast \mathbf{\Phi}_t)^{-1} \mathbf{\Phi}_t^\ast$.

    Similarly, the term $\rho^2 \lVert \mathbf{h}_t(x)\rVert_2^2$ (cf. \eqref{eq:helper_rkhsbound_reform}) is given as 
    \begin{align} \label{eq:hvar_reform}
		\rho^2 \lVert \mathbf{h}_t(x)\rVert_2^2 &= \rho^2 \mathbf{k}_t(x)^\top (\rho^2 \mathbf{I}_t + \mathbf{K}_t)^{-2} \mathbf{k}_t(x) = \rho^2 \langle \mathbf{k}_t(x), (\rho^2 \mathbf{I}_t + \mathbf{K}_t)^{-2} \mathbf{k}_t(x) \rangle_{\mathbb{R}^t} 
		\notag \\
        &= \rho^2 \langle \mathbf{\Phi}_t^\ast \phi(x), (\rho^2 \mathbf{I}_t + \mathbf{\Phi}_t^\ast\mathbf{\Phi}_t)^{-2} \mathbf{\Phi}_t^\ast \phi(x) \rangle_{\mathbb{R}^t} \notag \\
        &= \langle  \phi(x), \rho^2 \mathbf{\Phi}_t (\rho^2 \mathbf{I}_t + \mathbf{\Phi}_t^\ast\mathbf{\Phi}_t)^{-2} \mathbf{\Phi}_t^\ast \phi(x) \rangle_\mathcal{H} \doteq \langle \phi(x), \tilde{\mathbf{A}}_t \phi(x)\rangle_\mathcal{H}, 
	\end{align}
    where $\tilde{\mathbf{A}}_t: \mathcal{H} \to \mathcal{H}$ is defined as a compact self-adjoint operator $\tilde{\mathbf{A}}_t = \rho^2 \mathbf{\Phi}_t (\rho^2 \mathbf{I}_t + \mathbf{\Phi}_t^\ast \mathbf{\Phi}_t)^{-2}\mathbf{\Phi}_t^{\ast}$.
    
	Next, we recall the following identities, cf.~\citep[Appendix~D]{kirschner2018information}:
	\begin{itemize}
		\item[i)] $\rho^2 (\rho^2 \mathbf{I}_\mathcal{H} + \mathbf{\Phi}_t \mathbf{\Phi}_t^\ast)^{-1} = \mathbf{I}_\mathcal{H} - \mathbf{\Phi}_t \mathbf{\Phi}_t^\ast (\rho^2 \mathbf{I}_\mathcal{H} + \mathbf{\Phi}_t \mathbf{\Phi}_t^\ast)^{-1}$,
		\item[ii)] $\rho^2 (\rho^2 \mathbf{I}_\mathcal{H} + \mathbf{\Phi}_t \mathbf{\Phi}_t^\ast)^{-1} = \mathbf{I}_\mathcal{H} - \mathbf{\Phi}_t (\rho^2 \mathbf{I}_t + \mathbf{\Phi}_t^\ast\mathbf{\Phi}_t)^{-1} \mathbf{\Phi}_t^\ast$, and
		\item[iii)] $(\rho^2 \mathbf{I}_\mathcal{H} + \mathbf{\Phi}_t \mathbf{\Phi}_t^\ast)^{-1} \mathbf{\Phi}_t = \mathbf{\Phi}_t (\rho^2 \mathbf{I}_t + \mathbf{\Phi}_t^\ast\mathbf{\Phi}_t)^{-1}$.
	\end{itemize}
    By virtue of identities ii) and iii) from above, $\mathbf{A}_t$ and $ \tilde{\mathbf{A}}_t$ can be re-written respectively as
    \begin{subequations}\label{eq:kervar_operators} 
	\begin{align}
		\mathbf{A}_t &= \rho^2 (\rho^2 \mathbf{I}_\mathcal{H} + \mathbf{\Phi}_t \mathbf{\Phi}_t^\ast)^{-1} = (\mathbf{I}_\mathcal{H} + \rho^{-2}\mathbf{\Phi}_t \mathbf{\Phi}_t^\ast)^{-1}, \label{eq:kervar_weightmatrix}
        \\
        \tilde{\mathbf{A}}_t &= \rho^2 \mathbf{\Phi}_t (\rho^2 \mathbf{I}_t + \mathbf{\Phi}_t^\ast \mathbf{\Phi}_t)^{-2}\mathbf{\Phi}_t^{\ast} = \rho^2 (\rho^2 \mathbf{I}_\mathcal{H} + \mathbf{\Phi}_t \mathbf{\Phi}_t^\ast)^{-1} \mathbf{\Phi}_t \mathbf{\Phi}_t^\ast (\rho^2 \mathbf{I}_\mathcal{H} + \mathbf{\Phi}_t \mathbf{\Phi}_t^\ast)^{-1}.
	\end{align}
    \end{subequations}
    Utilizing the identity i) from above, we observe that
	\begin{multline*}
		\mathbf{A}_t - \tilde{\mathbf{A}}_t = \rho^2 (\rho^2 \mathbf{I}_\mathcal{H} + \mathbf{\Phi}_t \mathbf{\Phi}_t^\ast)^{-1} - \rho^2 (\rho^2 \mathbf{I}_\mathcal{H} + \mathbf{\Phi}_t \mathbf{\Phi}_t^\ast)^{-1} \mathbf{\Phi}_t \mathbf{\Phi}_t^\ast (\rho^2 \mathbf{I}_\mathcal{H} + \mathbf{\Phi}_t \mathbf{\Phi}_t^\ast)^{-1} \\
		= \rho^2 (\rho^2 \mathbf{I}_\mathcal{H} + \mathbf{\Phi}_t \mathbf{\Phi}_t^\ast)^{-1} \left[ \mathbf{I}_\mathcal{H} - \mathbf{\Phi}_t \mathbf{\Phi}_t^\ast (\rho^2 \mathbf{I}_\mathcal{H} + \mathbf{\Phi}_t \mathbf{\Phi}_t^\ast)^{-1} \right]  = \rho^2 (\rho^2 \mathbf{I}_\mathcal{H} + \mathbf{\Phi}_t \mathbf{\Phi}_t^\ast)^{-1} \cdot \rho^2 (\rho^2 \mathbf{I}_\mathcal{H} + \mathbf{\Phi}_t \mathbf{\Phi}_t^\ast)^{-1} = \mathbf{A}_t^2.
	\end{multline*}
    For the spectrum of $\mathbf{A}_t$, we have $\mathrm{spec}(\mathbf{A}_t) \subset (0,1]$ because $\mathrm{spec}(\mathbf{I}_\mathcal{H} + \rho^{-2}\mathbf{\Phi}_t \mathbf{\Phi}_t^\ast) \subset [1, \infty)$. Since $\tilde{\mathbf{A}}_t = \mathbf{A}_t -\mathbf{A}_t^2$, the spectrum of $\tilde{\mathbf{A}}_t$ can be bounded via
	\[
	\mathrm{spec}(\tilde{\mathbf{A}}_t) = \{\lambda - \lambda^2 \mid \lambda \in \mathrm{spec}(\mathbf{A}_t) \} \subset \{\lambda - \lambda^2 \mid \lambda \in (0,1] \}=[0,0.25].
	\]
    Similarly to \eqref{eq:hvar_reform} and by definition of $\mathbf{h}_t(x)$, it holds that
	\begin{multline*}
		\rho^2 \lVert \mathbf{h}_t(x) - \mathbf{h}_t(x')\rVert_2^2 = \rho^2 \left(\mathbf{k}_t(x)-\mathbf{k}_t(x')\right)^{\top} (\rho^2 \mathbf{I}_t + \mathbf{K}_t)^{-2}\left(\mathbf{k}_t(x)-\mathbf{k}_t(x')\right) = \\
		\rho^2 \langle \left(\mathbf{k}_t(x)-\mathbf{k}_t(x')\right), (\rho^2 \mathbf{I}_t + \mathbf{K}_t)^{-2}\left(\mathbf{k}_t(x)-\mathbf{k}_t(x')\right) \rangle_{\mathbb{R}_t} = \rho^2 \langle \mathbf{\Phi}_t^\ast \left(\phi(x)-\phi(x')\right), (\rho^2 \mathbf{I}_t + \mathbf{K}_t)^{-2}\mathbf{\Phi}_t^\ast \left(\phi(x)-\phi(x')\right) \rangle_{\mathbb{R}_t}\\
		= \rho^2 \langle \left(\phi(x)-\phi(x')\right), \mathbf{\Phi}_t(\rho^2 \mathbf{I}_t + \mathbf{\Phi}_t^\ast \mathbf{\Phi}_t)^{-2}\mathbf{\Phi}_t^\ast \left(\phi(x)-\phi(x')\right) \rangle_\mathcal{H} =\langle \left(\phi(x)-\phi(x')\right),  \tilde{\mathbf{A}}_t \left(\phi(x)-\phi(x')\right) \rangle_\mathcal{H}.
	\end{multline*}

    We now prove the inequality \eqref{eq:hoelder_2norm}. The reverse triangle inequality yields $\lvert \rho \lVert \mathbf{h}_t(x)\rVert_2 - \rho \lVert \mathbf{h}_t(x')\rVert_2\rvert \le \rho \lVert \mathbf{h}_t(x) - \mathbf{h}_t(x')\rVert_2$. 
    
    Exploiting the fact that the spectrum of $\tilde{\mathbf{A}}_t$ is contained in $[0,0.25]$, we obtain
    \begin{align}
        \lvert \rho \lVert \mathbf{h}_t(x)\rVert_2 - \rho \lVert \mathbf{h}_t(x')\rVert_2\rvert & \le \rho \lVert \mathbf{h}_t(x) - \mathbf{h}_t(x')\rVert_2 = \sqrt{\langle \left(\phi(x)-\phi(x')\right),  \tilde{\mathbf{A}}_t \left(\phi(x)-\phi(x')\right) \rangle_\mathcal{H}} \notag\\
        &\le \sqrt{0.25 \langle \phi(x)-\phi(x'), \phi(x)-\phi(x') \rangle_\mathcal{H}} = 0.5\sqrt{k(x,x) + k(x',x') - 2 k(x,x')} \notag\\
        &\le \sqrt{\frac{L}{2}}\lVert x - x' \rVert^{\frac{p}{2}}, \label{eq:hoelder_helper}
    \end{align}
    where the last inequality \eqref{eq:hoelder_helper} follows from Hölder continuity of the kernel $k(\cdot,\cdot)$ via Assumption~\ref{asm:hoelder}, cf. \citep[ Lemma~12]{curi2020efficient}, \citep[Lemma~2.3]{lederer2023gaussian}. 

    Furthermore, inequality \eqref{eq:hoelder_infnorm} then follows from leveraging the reverse triangle inequality and 
    $\lVert \cdot \rVert_\infty \le \lVert \cdot \rVert_2$, i.e., 
    \begin{equation*}
        \lvert \rho \lVert \mathbf{h}_t(x)\rVert_\infty - \rho \lVert \mathbf{h}_t(x')\rVert_\infty\rvert \le \rho \lVert \mathbf{h}_t(x) - \mathbf{h}_t(x')\rVert_\infty \le \rho \lVert \mathbf{h}_t(x) - \mathbf{h}_t(x')\rVert_2 \le \sqrt{\frac{L}{2}}\lVert x - x' \rVert^{\frac{p}{2}}.
    \end{equation*}

    Similarly, inequality \eqref{eq:hoelder_2norm_weight} follows from leveraging the reverse triangle inequality and $\lVert \cdot \rVert_{\mathbf{C}} \le \sqrt{\varsigma_{\mathbf{C}}} \lVert \cdot \rVert_2$, i.e., 
    \begin{equation*}
        \lvert \rho \lVert \mathbf{h}_t(x)\rVert_{\mathbf{C}} - \rho \lVert \mathbf{h}_t(x')\rVert_{\mathbf{C}}\rvert \le \rho \lVert \mathbf{h}_t(x) - \mathbf{h}_t(x')\rVert_{\mathbf{C}} \le \sqrt{\varsigma_{\mathbf{C}}} \rho \lVert \mathbf{h}_t(x) - \mathbf{h}_t(x')\rVert_2 \le \sqrt{\varsigma_{\mathbf{C}}\frac{L}{2}}\lVert x - x' \rVert^{\frac{p}{2}},
    \end{equation*}
    concluding the proof.
\end{proof}

Notably, Lemma~\ref{lem:hoelder} ensures Hölder continuity with a smaller order of $p/2$. In case of stationary kernels, this can be improved to order $p$, see \citep[Corrolary~2.6]{lederer2023gaussian}.

For the proof of Lemma~\ref{lem:hoelder}, we have used the fact that the eigenvalues $\tilde{\lambda}$ of $\tilde{\mathbf{A}}_t = \mathbf{A}_t-\mathbf{A}_t^2$ are bounded by $0 \le \tilde{\lambda} \le 0.25$. Using structural knowledge of $\mathbf{A}_t$, it can be shown that the eigenvalues depend on $\rho$ with $\tilde{\lambda} \to 0$ for $\rho \to 0$: By definition, we have $\mathbf{A}_t = (\mathbf{I}_\mathcal{H} + \rho^{-2}\mathbf{\Phi}_t \mathbf{\Phi}_t^\ast)^{-1}$; see \eqref{eq:kervar_weightmatrix}. Let us denote the $i$-th (non-zero) eigenvalue of $\mathbf{\Phi}_t \mathbf{\Phi}_t^\ast$ as $\lambda_i^\phi$ with $i \in \Lambda^\phi$, where $\Lambda^\phi \subseteq \mathbb{N}_{t}$ is the set we use for indexing\footnote{Since $\mathbf{\Phi}_t \mathbf{\Phi}_t^\ast$ is a compact self-adjoint and finite-rank operator, its spectrum only contains zero and at most $t$ points in $\mathbb{R}_{\geq 0}$ corresponding to the non-zero eigenvalues. Hence to index them, we can use a (finite) subset of $\mathbb{N}_t$.}. Then, the corresponding $i$-th eigenvalue of $\mathbf{A}_t$ is $\lambda_i = (1+\rho^{-2}\lambda_i^\phi)^{-1}$, and the corresponding $i$-th eigenvalue of $\tilde{\mathbf{A}}_t$ is 
\begin{equation*}
    \tilde{\lambda}_i = \lambda_i - \lambda_i^2 = (1+\rho^{-2}\lambda_i^\phi)^{-1} - (1+\rho^{-2}\lambda_i^\phi)^{-2} = \rho^{-2}\lambda_i^\phi(1+\rho^{-2}\lambda_i^\phi)^{-2} = \rho^{2}\lambda_i^\phi(\rho^{2} + \lambda_i^\phi)^{-2}.
\end{equation*} 
For fixed $\lambda_i^\phi>0$, we have $\tilde{\lambda}_i \to 0$ for both $\rho \to 0$ and $\rho \to \infty$, and the maximum value of $0.25$ is attained for $\rho = \sqrt{\lambda_i^\phi}$. This dependence on $\rho$ (or, equivalently, on the noise level $\sigma_M$ when opting for the common choice $\rho = \sigma_M$ \cite{kanagawa2025gaussian}) is not reflected in Lemma~\ref{lem:hoelder}. Thus, the bound on the right-hand side of \eqref{eq:hoelder_2norm}--\eqref{eq:hoelder_2norm_weight} can be improved by scaling it with the data-dependent factor $b_t = \max_{i\in\Lambda^\phi}\sqrt{4\rho^{2}\lambda_i^\phi(\rho^{2} + \lambda_i^\phi)^{-2}} \in [0,1]$, with $b_t = 0$ for $\rho = 0$.

\subsection*{Proof of Theorem~\ref{thm:errorbounds_uniform}:} 
Leveraging Theorem~\ref{thm:errorbounds_nonuniform} and Lemma~\ref{lem:hoelder}, we proceed the proof of Theorem~\ref{thm:errorbounds_uniform} as follows.
\begin{proof}
    Choose any regularization parameter $\rho > 0$, input $x \in \mathcal{X}$, time/number of data points $t \in \mathbb{N}$, confidence parameter $\delta \in (0,1)$, and grid constants $\zeta_t,\,\zeta_{t}>0$. For all cases \textbf{(a)}--\textbf{(d)}, the assertions follow from Lemma~\ref{lem:errorbound_general} by deriving suitable functions ${\eta}_t^M$ satisfying the probabilistic uniform bound $ \mathbb{P}\left[ \forall x \in \mathcal{X},\,t\in\mathbb{N}:\,\lvert\mathbf{M}_t^{\top}\mathbf{h}_t(x)\rvert \le {\eta}^{M}_t(x)\right] \geq 1 - \delta$ from \eqref{eq:noisebound_uniform}.
    
    \paragraph{(a)}
    From Theorem~\ref{thm:errorbounds_nonuniform}, the probabilistic error bound \eqref{eq:errorbound_subG_nonuniform} with $\mathbf{C}_t=\sigma_M^2\mathbf{I}_t$ holds for any confidence parameter $\delta_1 \in (0,1)$ and fixed $x \in \mathcal{X}$ and $t \in \mathbb{N}$. Setting $\delta_1 = \delta_2/\pi_t$ for some $\delta_2 \in (0,1)$ and $\{\pi_t\}_{t=1}^{\infty}$ satisfying $\sum_{t=1}^{\infty}\pi_t^{-1} = 1$ (e.g., $\pi_t = \pi^2 t^2 / 6$), and then taking the union bound over all times \citep{Srinivas2009}, we obtain the time-uniform counterpart of \eqref{eq:errorbound_subG_nonuniform}, i.e. 
    \begin{equation} \label{eq:errorbound_subG_timeuniform}
        \forall x \in \mathcal{X}:~~ \prob{\forall t \in \mathbb{N}:~~ \lvert \mathbf{M}_t^{\top}\mathbf{h}_t(x) \rvert \le \sqrt{2\ln(2\pi_t/\delta_2)} \sigma_M \norm{\mathbf{h}_t(x)}_{2}} \ge 1-\delta_2.
    \end{equation}
    
    In order to lift this bound~\eqref{eq:errorbound_subG_timeuniform} to hold uniformly over all $x \in \mathcal{X}$, we leverage a discretization approach \cite{Srinivas2009,Lederer2019}: Consider a discretized version $\mathcal{X}_\zeta$ of the input domain $\mathcal{X}$ with $\lvert \mathcal{X}_\zeta \rvert$ grid points and grid width $\zeta = \zeta_t$ such that ${\max_{x \in \mathcal{X}} \min_{x'\in\mathcal{X}_\zeta} \lVert x - x' \rVert} \le \zeta$. The minimum number $\lvert \mathcal{X}_\zeta \rvert$ of grid points such that this condition is satisfied is given by the $\zeta$-covering number $Z(\zeta,\mathcal{X})$ of $\mathcal{X}$. 
    We aim to extend the time-uniform bound \eqref{eq:errorbound_subG_timeuniform} such that it holds uniformly over $x \in \overline{\mathcal{X}} \doteq\{x \in \mathbb{R}^{n_x}~|~\lVert x - [x]_\zeta\rVert \le \zeta,~ [x]_\zeta \in \mathcal{X}_\zeta\}$, with $[x]_\zeta \in \mathcal{X}_\zeta$ denoting the closest point in $\mathcal{X}_\zeta$ to $x$. Then, since $\mathcal{X} \subseteq  \overline{\mathcal{X}}$ by definition of $ \overline{\mathcal{X}}$, the desired bound also holds uniformly over all $x \in \mathcal{X}$.
    
    To this end, setting $\delta_2 = \delta_3/Z(\zeta,\mathcal{X})$ for some $\delta_3 \in (0,1)$ and applying \eqref{eq:errorbound_subG_timeuniform} and the union bound over the discretized domain $\mathcal{X}_\zeta$, we obtain
    \begin{equation} \label{eq:bound_discretized}
        \prob{\forall x \in \mathcal{X}_\zeta,~~ \forall t \in \mathbb{N}:~~ \lvert \mathbf{M}_t^{\top}\mathbf{h}_t(x) \rvert \le \sqrt{2\ln(2\pi_t Z(\zeta,\mathcal{X})/\delta_3)} \sigma_M \norm{\mathbf{h}_t(x)}_{2}} \ge 1-\delta_3.
    \end{equation}     
    
    Using the triangle inequality and the Hölder inequality, we have 
    \begin{align}
        \lvert \mathbf{M}_t^{\top}\mathbf{h}_t(x) \rvert &\le \lvert \mathbf{M}_t^{\top}\mathbf{h}_t([x]_\zeta) \rvert + \lvert \mathbf{M}_t^{\top}\left(\mathbf{h}_t(x) - \mathbf{h}_t([x]_\zeta)\right)\rvert \notag\\
        &\le \lvert \mathbf{M}_t^{\top}\mathbf{h}_t([x]_\zeta) \rvert +  \norm{\mathbf{M}_t^{\top} \left(\mathbf{K}_t+\rho^2\mathbf{I}_t\right)^{-1}}_1 \norm{\mathbf{k}_t(x) - \mathbf{k}_t([x]_\zeta)}_{\infty} \notag\\
        &\le \lvert \mathbf{M}_t^{\top}\mathbf{h}_t([x]_\zeta) \rvert +  L \zeta^p \norm{\mathbf{M}_t^{\top} \left(\mathbf{K}_t+\rho^2\mathbf{I}_t\right)^{-1}}_1, \label{eq:bound_helper}
    \end{align}
    where the last inequality \eqref{eq:bound_helper} follows from $\norm{\mathbf{k}_t(x) - \mathbf{k}_t([x]_\zeta)}_{\infty} \le L\norm{x - [x]_\zeta}^p \le L\zeta^p$, leveraging Hölder continuity of the kernel $k(\cdot,\cdot)$ via Assumption~\ref{asm:hoelder}.
    The term $\lvert \mathbf{M}_t^{\top}\mathbf{h}_t([x]_\zeta) \rvert$ in \eqref{eq:bound_helper} obeys the uniform bound~\eqref{eq:bound_discretized}, which we can express in terms of $\sigma_M\norm{\mathbf{h}_t(x)}_{2}$ by leveraging Lemma~\ref{lem:hoelder} and $\norm{x - [x]_\zeta}^p \le \zeta^p$, i.e.,
    \begin{align}
       \sigma_M\norm{\mathbf{h}_t([x]_\zeta)}_{2} &= \sigma_M\norm{\mathbf{h}_t(x)}_{2} + \frac{\sigma_M}{\rho}\left(\rho \norm{\mathbf{h}_t([x]_\zeta)}_{2} - \rho \norm{\mathbf{h}_t(x)}_{2}\right) \notag\\
        &\overset{\eqref{eq:hoelder_2norm}}{\le} \sigma_M\norm{\mathbf{h}_t(x)}_{2} + \frac{\sigma_M}{\rho} \sqrt{\frac{L}{2}} \norm{x - [x]_\zeta}^{\frac{p}{2}} \notag\\
        &\le \sigma_M \norm{\mathbf{h}_t(x)}_{2} + \frac{\sigma_M}{\rho} \sqrt{\frac{L}{2}} \zeta^{\frac{p}{2}}\label{eq:WKV_discrbound}.
    \end{align}

    Note that inequalities similar to \eqref{eq:WKV_discrbound} can likewise be derived for the norms $\norm{\mathbf{h}_t([x]_\zeta)}_{\infty}$ and $\norm{\mathbf{h}_t([x]_\zeta)}_{\mathbf{C}}$ with $\mathbf{C}\succeq \bm{0}$ using \eqref{eq:hoelder_infnorm} and \eqref{eq:hoelder_2norm_weight} from Lemma~\ref{lem:hoelder}.
    
    We now derive a probabilistic bound for the term $\norm{\mathbf{M}_t^{\top} \left(\mathbf{K}_t+\rho^2\mathbf{I}_t\right)^{-1}}_1$ in \eqref{eq:bound_helper}. First, let us denote $\left(\mathbf{K}_t+\rho^2\mathbf{I}_t\right)^{-1} \doteq \mat{\mathbf{a}_1 & \dots & \mathbf{a}_t}$. Then, we can write $\norm{\mathbf{M}_t^{\top} \left(\mathbf{K}_t+\rho^2\mathbf{I}_t\right)^{-1}}_1 = \sum_{i=1}^t \lvert \mathbf{M}_t^{\top} \mathbf{a}_i \rvert $. By \citep[Lemma~2]{ao2025stochastic}, it holds that 
    \begin{equation*}
        \prob{\lvert \mathbf{M}_t^{\top} \mathbf{a}_i \rvert \le \sqrt{2\ln(2/\delta_4)} \sigma_M\norm{\mathbf{a}_i}_{2}} \ge 1-\delta_4
    \end{equation*}
    for all $\delta_4 \in (0,1)$, cf.~Theorem~\ref{thm:errorbounds_nonuniform}(a). Hence, choosing $\delta_4 = \delta_5/(\pi_t t)$ for some $\delta_5 \in (0,1)$ and applying the union bound over the sum of $t$ elements and all $t \in \mathbb{N}$, we obtain
    \begin{equation} \label{eq:noisediscbnd}
        \prob{\forall t \in \mathbb{N}:~~ \sum_{i=1}^t \lvert \mathbf{M}_t^{\top} \mathbf{a}_i \rvert \le \sqrt{2\ln(2\pi_t t/\delta_5)} \sigma_M t} \ge 1-\delta_5,
    \end{equation}
    leveraging $\sum_{i=1}^t \norm{\mathbf{a}_i}_2 \le \sum_{i=1}^t 1 = t$ since the matrix $\left(\mathbf{K}_t+\rho^2\mathbf{I}_t\right)^{-1}$ is positive definite with singular values upper-bounded by $1$. 

    Lastly, using \eqref{eq:bound_discretized} and \eqref{eq:noisediscbnd} with confidence $\delta_3 = \delta _5 = \delta/2$ and applying the union bound, via \eqref{eq:bound_helper}, \eqref{eq:WKV_discrbound}, and \eqref{eq:hoelder_2norm_weight}, we obtain
    \begin{equation} \label{eq:SGnoisebound}
        \prob{\forall x \in \mathcal{X},~~ \forall t \in \mathbb{N}:~~ \lvert \mathbf{M}_t^{\top}\mathbf{h}_t(x) \rvert \le  \eta_t^{\mathcal{SG}}(x) = \beta^{\mathcal{SG}}_{t} \norm{\mathbf{h}_t(x)}_{2}  + \Delta^{\mathcal{SG}}_{t} ~} \ge 1-\delta,
    \end{equation}
    with parameters $\beta^{\mathcal{SG}}_{t} = \sqrt{2\ln(4\pi_t Z(\zeta_{t},\mathcal{X})/\delta)}$ and $\Delta^{\mathcal{SG}}_{t} = \beta^{\mathcal{SG}}_{t}\sigma \sqrt{L/({2\rho^2})}\zeta_{t}^{p/2} + \sqrt{2\ln(4\pi_t t/\delta)} \sigma t L \zeta_{t}^{p}$, as in Table~\ref{tab:bounds}. The assertion then follows from Lemma~\ref{lem:errorbound_general} by combining \eqref{eq:bound_deterministic} and the derived probabilistic uniform bound \eqref{eq:SGnoisebound}.

    For the remaining cases \textbf{(a)}--\textbf{(d)}, the assertions follow similar arguments presented here by leveraging \eqref{eq:bound_helper}, \eqref{eq:WKV_discrbound}, Hölder continuity of norms of $\mathbf{h}_t(x)$ via Lemma~\ref{lem:hoelder}, and the union bound; however, we need to derive new probabilistic bounds similar to \eqref{eq:bound_discretized} and \eqref{eq:noisediscbnd} for the given distribution class.
    
    \paragraph{(b)}     
    The Hoeffding-type bound~\eqref{eq:errorbound_bnd_hoeffding_uniform} follows directly from part \textbf{(a)} using the variance proxy $\sigma_M = \overline{m}$, yielding
    \begin{equation} \label{eq:BNDnoisebound_hoeffding}
        \prob{\forall x \in \mathcal{X},~~ \forall t \in \mathbb{N}:~~ \lvert \mathbf{M}_t^{\top}\mathbf{h}_t(x) \rvert \le \eta_{1,t}^{\mathrm{bnd}}(x) = \beta^\mathrm{bnd}_{1,t}\,\overline{m} \lVert \mathbf{h}_t(x) \rVert_2 + \Delta^{\mathrm{bnd}}_{1,t} ~} \ge 1-\delta,
    \end{equation}
    with parameters $\beta^\mathrm{bnd}_{1,t}$ and $\Delta^{\mathrm{bnd}}_{1,t}$ as in Table~\ref{tab:bounds}(b).
    
    For the Berstein-type bound~\eqref{eq:errorbound_bnd_bernstein_uniform}, we adjust the following steps from part \textbf{(a)}: The uniform counterpart of \eqref{eq:errorbound_bnd_bernstein} over time $t\in\mathbb{N}$ and discretized input domain $\mathcal{X}_\zeta$ results in (cf. \eqref{eq:bound_discretized})
    \begin{equation} \label{eq:bound_discretized_bernstein}
        \prob{\forall x \in \mathcal{X}_\zeta,~~ \forall t \in \mathbb{N}:~~ \lvert \mathbf{M}_t^{\top}\mathbf{h}_t(x) \rvert \le \frac{1}{3}\hat{\beta}_{t}~ \overline{m} \lVert\mathbf{h}_t(x)\rVert_\infty + \sqrt{\hat{\beta}_{t}}~ \overline{\sigma} \lVert \mathbf{h}_t(x) \rVert_2} \ge 1-\delta_3,
    \end{equation}
    with $\delta_3 \in (0,1)$ and $\hat{\beta}_{t} = 2\ln(2\pi_tZ(\zeta_t,\mathcal{X})/\delta_3)$.

    Regarding the term $\norm{\mathbf{M}_t^{\top} \left(\mathbf{K}_t+\rho^2\mathbf{I}_t\right)^{-1}}_1 = \sum_{i=1}^t \lvert \mathbf{M}_t^{\top} \mathbf{a}_i \rvert$ in \eqref{eq:bound_helper} with $\left(\mathbf{K}_t+\rho^2\mathbf{I}_t\right)^{-1} = \mat{\mathbf{a}_1 & \dots & \mathbf{a}_t}$, we apply Bernstein's inequality for the sum of bounded random variables \citep[Theorem~2.8.4]{vershynin2018high} to obtain 
    \begin{equation*}
        \prob{\lvert \mathbf{M}_t^{\top} \mathbf{a}_i \rvert \le \frac{2}{3}\ln(2/\delta_4)~ \overline{m} \lVert\mathbf{a}_i\rVert_\infty + \sqrt{2\ln(2/\delta_4)}~ \overline{\sigma} \lVert \mathbf{a}_i \rVert_2 } \ge 1-\delta_4
    \end{equation*}
    for all $\delta_4 \in (0,1)$, cf.~Theorem~\ref{thm:errorbounds_nonuniform}(b). Hence, choosing $\delta_4 = \delta_5/(\pi_t t)$ for some $\delta_5 \in (0,1)$ and applying the union bound over the sum of $t$ elements and all $t \in \mathbb{N}$, we obtain (cf. \eqref{eq:noisediscbnd})
    \begin{equation} \label{eq:noisediscbnd_bernstein}
        \prob{\forall t \in \mathbb{N}:~~ \sum_{i=1}^t \lvert \mathbf{M}_t^{\top} \mathbf{a}_i \rvert \le \frac{2}{3}\ln(2\pi_t t/\delta_5)~ \overline{m} t + \sqrt{2\ln(2\pi_t t/\delta_5)}~ \overline{\sigma}t} \ge 1-\delta_5,
    \end{equation}
    leveraging $\sum_{i=1}^t \norm{\mathbf{a}_i}_{\infty} \le \sum_{i=1}^t \norm{\mathbf{a}_i}_2 \le \sum_{i=1}^t 1 = t$ since $\left(\mathbf{K}_t+\rho^2\mathbf{I}_t\right)^{-1}$ is positive definite with singular values upper-bounded by $1$. 

    Lastly, using \eqref{eq:bound_discretized_bernstein} and \eqref{eq:noisediscbnd_bernstein} with $\delta_3 = \delta _5 = \delta/2$ and applying the union bound, via \eqref{eq:bound_helper}, \eqref{eq:WKV_discrbound}, and Lemma~\ref{lem:hoelder}, we obtain (cf. \eqref{eq:SGnoisebound})
    \begin{equation} \label{eq:BNDnoisebound_bernstein}
        \prob{\forall x \in \mathcal{X},~~ \forall t \in \mathbb{N}:~~ \lvert \mathbf{M}_t^{\top}\mathbf{h}_t(x) \rvert \le \eta_{2,t}^{\mathrm{bnd}}(x) = \beta^\mathrm{bnd}_{1,t}\,\overline{\sigma} \lVert \mathbf{h}_t(x) \rVert_2 + \beta_{2,t}^{\mathrm{bnd}} \,\overline{m} \lVert\mathbf{h}_t(x)\rVert_\infty + \Delta^{\mathrm{bnd}}_{2,t} ~} \ge 1-\delta,
    \end{equation}
    with parameters $\beta^\mathrm{bnd}_{1,t}$, $\beta^\mathrm{bnd}_{2,t}$, and $\Delta^{\mathrm{bnd}}_{2,t}$ as in Table~\ref{tab:bounds}(b). The assertion then follows from Lemma~\ref{lem:errorbound_general} by combining \eqref{eq:bound_deterministic} and the derived probabilistic uniform bound \eqref{eq:BNDnoisebound_bernstein}.

    Since both the Hoeffding-type bound~\eqref{eq:BNDnoisebound_hoeffding} and the Bernstein-type bound~\eqref{eq:BNDnoisebound_bernstein} are valid probabilistic uniform bounds \eqref{eq:noisebound_uniform}, we can employ ${\eta}^{M}_t(x) = \min_{j\in\{1,2\}}{\eta}^\mathrm{bnd}_{j,t}(x)$ to obtain the tightest bound.
    
    \paragraph{(c)} 
    We adjust the following steps from part \textbf{(a)}: The uniform counterpart of \eqref{eq:errorbound_subE_nonuniform} over time $t\in\mathbb{N}$ and discretized input domain $\mathcal{X}_\zeta$ results in (cf. \eqref{eq:bound_discretized})
    \begin{equation} \label{eq:bound_discretized_subE}
        \prob{\forall x \in \mathcal{X}_\zeta,~~ \forall t \in \mathbb{N}:~~ \lvert \mathbf{M}_t^{\top}\mathbf{h}_t(x) \rvert \le \max\left\{\hat{\beta}_{t}\, \alpha_M \lVert \mathbf{h}_t(x)\rVert_\infty,~ \sqrt{\hat{\beta}_{t}}\, \nu_M \lVert \mathbf{h}_t(x)\rVert_2\right\}} \ge 1-\delta_3,
    \end{equation}
    with $\delta_3 \in (0,1)$ and $\hat{\beta}_{t} = 2\ln(2\pi_tZ(\zeta_t,\mathcal{X})/\delta_3)$.

    Regarding the term $\norm{\mathbf{M}_t^{\top} \left(\mathbf{K}_t+\rho^2\mathbf{I}_t\right)^{-1}}_1 = \sum_{i=1}^t \lvert \mathbf{M}_t^{\top} \mathbf{a}_i \rvert$ in \eqref{eq:bound_helper} with $\left(\mathbf{K}_t+\rho^2\mathbf{I}_t\right)^{-1} = \mat{\mathbf{a}_1 & \dots & \mathbf{a}_t}$, we apply Bernstein's inequality for the sum of weighted sub-exponential random variables \citep[Theorem~2.8.2]{vershynin2018high}, \citep[Proposition~2.9]{wainwright2019high} to obtain 
    \begin{equation*}
        \prob{\lvert \mathbf{M}_t^{\top} \mathbf{a}_i \rvert \le  \max\left\{2\ln(2/\delta_4)\, \alpha_M \lVert \mathbf{a}_i\rVert_\infty,~ \sqrt{2\ln(2/\delta_4)}\, \nu_M \lVert \mathbf{a}_i\rVert_2\right\}} \ge 1-\delta_4
    \end{equation*}
    for all $\delta_4 \in (0,1)$, cf.~Theorem~\ref{thm:errorbounds_nonuniform}(c). Hence, choosing $\delta_4 = \delta_5/(\pi_t t)$ for some $\delta_5 \in (0,1)$ and applying the union bound over the sum of $t$ elements and all $t \in \mathbb{N}$, we obtain (cf. \eqref{eq:noisediscbnd})
    \begin{equation} \label{eq:noisediscbnd_subE}
        \prob{\forall t \in \mathbb{N}:~~ \sum_{i=1}^t \lvert \mathbf{M}_t^{\top} \mathbf{a}_i \rvert \le \max\left\{2\ln(2\pi_t t/\delta_5)\, \alpha_M,~ \sqrt{2\ln(2\pi_t t/\delta_5)}\, \nu_M\right\} t} \ge 1-\delta_5,
    \end{equation}
    leveraging $\norm{\mathbf{a}_i}_{\infty} \le \norm{\mathbf{a}_i}_2 \le 1 $ and $\sum_{i=1}^t 1 = t$. 

    Now, using \eqref{eq:bound_discretized_subE} and \eqref{eq:noisediscbnd_subE} with $\delta_3 = \delta _5 = \delta/2$ and applying the union bound, via \eqref{eq:bound_helper}, \eqref{eq:WKV_discrbound}, and Lemma~\ref{lem:hoelder}, we obtain $\prob{\forall x \in \mathcal{X},~~ \forall t \in \mathbb{N}:~~ \lvert \mathbf{M}_t^{\top}\mathbf{h}_t(x) \rvert \le \hat{\eta}_{t}^{\mathcal{SE}}(x)} \ge 1-\delta$
    with
    \begin{align*}
        \hat{\eta}_{t}^{\mathcal{SE}}(x) \doteq&\max\left\{\beta^{\mathcal{SE}}_{t}\, \alpha_M \left(\lVert \mathbf{h}_t(x)\rVert_\infty + \frac{1}{\rho}\sqrt{\frac{L}{2}}\zeta^{\frac{p}{2}}_t\right), \sqrt{\beta^{\mathcal{SE}}_{t}}\, \nu_M \left(\lVert \mathbf{h}_t(x)\rVert_2 + \frac{1}{\rho}\sqrt{\frac{L}{2}}\zeta^{\frac{p}{2}}_t\right)\right\} \\ &~~~+ \max\left\{2\ln(4\pi_t t/\delta)\, \alpha_M,~ \sqrt{2\ln(4\pi_t t/\delta)}\, \nu_M\right\} t L \zeta_t^p,
    \end{align*}
    where $\beta^{\mathcal{SE}}_{t} = 2\ln(4\pi_tZ(\zeta_t,\mathcal{X})/\delta)$ as in Table~\ref{tab:bounds}(c).
    To simplify the expression, we employ the upper-bound ${\eta}_{t}^{\mathcal{SE}}(x) \ge \hat{\eta}_{t}^{\mathcal{SE}}(x)$ from \eqref{eq:errorbound_subE_uniform} such that (cf. \eqref{eq:SGnoisebound})
    \begin{equation} \label{eq:SEnoisebound}
        \prob{\forall x \in \mathcal{X},~~ \forall t \in \mathbb{N}:~~ \lvert \mathbf{M}_t^{\top}\mathbf{h}_t(x) \rvert \le {\eta}^{\mathcal{SE}}_t(x) = \max\{\beta^{\mathcal{SE}}_{t}\, \alpha_M \lVert \mathbf{h}_t(x)\rVert_\infty ,~ \sqrt{\beta^{\mathcal{SE}}_{t}}\, \nu_M \norm{\mathbf{h}_t(x)}_{2}\} + \Delta^{\mathcal{SE}}_{t}} \ge 1-\delta,
    \end{equation}    
    with parameter $\Delta^{\mathcal{SE}}_{t}$ as in Table~\ref{tab:bounds}(c). The assertion then follows from Lemma~\ref{lem:errorbound_general} by combining \eqref{eq:bound_deterministic} and the derived probabilistic uniform bound \eqref{eq:SEnoisebound}.
    
    \paragraph{(d)} 
    We adjust the following steps from part \textbf{(a)}: The uniform counterpart of \eqref{eq:errorbound_L2_nonuniform} over time $t\in\mathbb{N}$ and discretized input domain $\mathcal{X}_\zeta$ results in (cf. \eqref{eq:bound_discretized})
    \begin{equation} \label{eq:bound_discretized_L2}
        \prob{\forall x \in \mathcal{X}_\zeta,~~ \forall t \in \mathbb{N}:~~ \lvert \mathbf{M}_t^{\top}\mathbf{h}_t(x) \rvert \le \sqrt{\pi_t Z(\zeta_t,\mathcal{X})/\delta_3}~ \sigma_M \lVert \mathbf{h}_t(x)\rVert_2} \ge 1-\delta_3,
    \end{equation}
    with $\delta_3 \in (0,1)$.

    Since the noise sequence $M_1,\,\dots,\,M_t$ is i.i.d., we have $\mathbf{M}_t^{\top}\mathbf{a} \in \mathcal{L}^2(\sigma^2_M\lVert \mathbf{a}\rVert^2_2)$ for all $\mathbf{a} \in \mathbb{R}^t$. Thus, regarding the term $\norm{\mathbf{M}_t^{\top} \left(\mathbf{K}_t+\rho^2\mathbf{I}_t\right)^{-1}}_1 = \sum_{i=1}^t \lvert \mathbf{M}_t^{\top} \mathbf{a}_i \rvert$ in \eqref{eq:bound_helper} with $\left(\mathbf{K}_t+\rho^2\mathbf{I}_t\right)^{-1} = \mat{\mathbf{a}_1 & \dots & \mathbf{a}_t}$, we apply the Chebychev inequality~\citep[Corollary~1.2.5]{vershynin2018high} to obtain 
    \begin{equation*}
        \prob{\lvert \mathbf{M}_t^{\top} \mathbf{a}_i \rvert \le \sqrt{1/\delta_4}~ \sigma_M \lVert \mathbf{a}_i \rVert_2 } \ge 1-\delta_4
    \end{equation*}
    for all $\delta_4 \in (0,1)$, cf.~Theorem~\ref{thm:errorbounds_nonuniform}(d). Hence, choosing $\delta_4 = \delta_5/(\pi_t t)$ for some $\delta_5 \in (0,1)$ and applying the union bound over the sum of $t$ elements and all $t \in \mathbb{N}$, we obtain (cf. \eqref{eq:noisediscbnd})
    \begin{equation} \label{eq:noisediscbnd_L2}
        \prob{\forall t \in \mathbb{N}:~~ \sum_{i=1}^t \lvert \mathbf{M}_t^{\top} \mathbf{a}_i \rvert \le \sqrt{\pi_t t/\delta_5}~ \sigma_Mt} \ge 1-\delta_5,
    \end{equation}
    leveraging $\sum_{i=1}^t \norm{\mathbf{a}_i}_2 \le \sum_{i=1}^t 1 = t$ since $\left(\mathbf{K}_t+\rho^2\mathbf{I}_t\right)^{-1}$ is positive definite with singular values upper-bounded by $1$. 

    Lastly, using \eqref{eq:bound_discretized_L2} and \eqref{eq:noisediscbnd_L2} with $\delta_3 = \delta _5 = \delta/2$ and applying the union bound, via \eqref{eq:bound_helper}, \eqref{eq:WKV_discrbound}, and Lemma~\ref{lem:hoelder}, we obtain (cf. \eqref{eq:SGnoisebound})
    \begin{equation} \label{eq:L2noisebound}
        \prob{\forall x \in \mathcal{X},~~ \forall t \in \mathbb{N}:~~ \lvert \mathbf{M}_t^{\top}\mathbf{h}_t(x) \rvert \le {\eta}^{\mathcal{L}2}_{t}(x) =  \beta^{\mathcal{L}2}_{t}~ \sigma_M \lVert \mathbf{h}_t(x)\rVert_2 + \Delta^{\mathcal{L}2}_{t}} \ge 1-\delta,
    \end{equation}
    with parameters $\beta^{\mathcal{L}2}_{t}$ and $\Delta^{\mathcal{L}2}_{t}$ as in Table~\ref{tab:bounds}(d). The assertion then follows from Lemma~\ref{lem:errorbound_general} by combining \eqref{eq:bound_deterministic} and the derived probabilistic uniform bound \eqref{eq:L2noisebound}.
\end{proof}

\subsection{Proof of Proposition~\ref{prop:cor_subG}} \label{app:proof_cor-subG}
\begin{proof}
    The proof follows the arguments presented in the proof of Theorem~\ref{thm:errorbounds_uniform}(a) (see Appendix~\ref{app:proof_uniform}), specifically by leveraging \eqref{eq:bound_helper}, \eqref{eq:WKV_discrbound}, Hölder continuity of the norm $\lVert \mathbf{h}_t(x)\rVert_{\mathbf{C}_t}$ via \eqref{eq:hoelder_2norm_weight} from Lemma~\ref{lem:hoelder}, and the union bound. However, we need to derive new probabilistic bounds similar to \eqref{eq:bound_discretized} and \eqref{eq:noisediscbnd} considering $\mathbf{M}_t \in \mathcal{SG}(\mathbf{C}_t)$ for some matrix variance proxy $\mathbf{C}_t\succeq\bm{0}$ instead of $\mathbf{C}_t = \mathbf{I}_t$, cf. Theorem~\ref{thm:errorbounds_nonuniform}(a).

    The uniform counterpart of \eqref{eq:errorbound_subG_nonuniform} over time $t\in\mathbb{N}$ and discretized input domain $\mathcal{X}_\zeta$ for general $\mathbf{C}_t\succeq\bm{0}$ results in (cf. \eqref{eq:bound_discretized})
    \begin{equation} \label{eq:bound_discretized_subG_cor}
        \prob{\forall x \in \mathcal{X}_\zeta,~~ \forall t \in \mathbb{N}:~~ \lvert \mathbf{M}_t^{\top}\mathbf{h}_t(x) \rvert \le \sqrt{2\ln(2\pi_t Z(\zeta,\mathcal{X})/\delta_3)} \norm{\mathbf{h}_t(x)}_{\mathbf{C}_t}} \ge 1-\delta_3,
    \end{equation}
    with $\delta_3 \in (0,1)$.

    Regarding the term $\norm{\mathbf{M}_t^{\top} \left(\mathbf{K}_t+\rho^2\mathbf{I}_t\right)^{-1}}_1 = \sum_{i=1}^t \lvert \mathbf{M}_t^{\top} \mathbf{a}_i \rvert$ in \eqref{eq:bound_helper} with $\left(\mathbf{K}_t+\rho^2\mathbf{I}_t\right)^{-1} = \mat{\mathbf{a}_1 & \dots & \mathbf{a}_t}$, we apply \citep[Lemma~2]{ao2025stochastic} to obtain 
    \begin{equation*}
        \prob{\lvert \mathbf{M}_t^{\top} \mathbf{a}_i \rvert \le \sqrt{2\ln(2/\delta_4)} \norm{\mathbf{a}_i}_{\mathbf{C}_t} } \ge 1-\delta_4
    \end{equation*}
    for all $\delta_4 \in (0,1)$, cf.~Theorem~\ref{thm:errorbounds_uniform}(a). Hence, choosing $\delta_4 = \delta_5/(\pi_t t)$ for some $\delta_5 \in (0,1)$ and applying the union bound over the sum of $t$ elements and all $t \in \mathbb{N}$, we obtain (cf. \eqref{eq:noisediscbnd})
    \begin{equation} \label{eq:noisediscbnd_subG_cor}
        \prob{\forall t \in \mathbb{N}:~~ \sum_{i=1}^t \lvert \mathbf{M}_t^{\top} \mathbf{a}_i \rvert \le \sqrt{2\ln(2\pi_t t/\delta_5)} \sqrt{\varsigma_{\mathbf{C}_t}} \,t} \ge 1-\delta_5,
    \end{equation}
    leveraging $\sum_{i=1}^t \norm{\mathbf{a}_i}_{\mathbf{C}_t} \le \sum_{i=1}^t \sqrt{\varsigma_{\mathbf{C}_t}} \norm{\mathbf{a}_i}_{2} \le \sum_{i=1}^t \sqrt{\varsigma_{\mathbf{C}_t}} = \sqrt{\varsigma_{\mathbf{C}_t}}\, t$, where $\varsigma_{\mathbf{C}_t}$ is the maximum singular value of $\mathbf{C}_t$ and since $\left(\mathbf{K}_t+\rho^2\mathbf{I}_t\right)^{-1}$ is positive definite with singular values upper-bounded by $1$. 

    Lastly, using \eqref{eq:bound_discretized_subG_cor} and \eqref{eq:noisediscbnd_L2} with $\delta_3 = \delta _5 = \delta/2$ and applying the union bound, via \eqref{eq:bound_helper}, \eqref{eq:WKV_discrbound}, and Lemma~\ref{lem:hoelder}, we obtain (cf. \eqref{eq:SGnoisebound})
    \begin{equation} \label{eq:SGnoisebound_cor}
        \prob{\forall x \in \mathcal{X},~~ \forall t \in \mathbb{N}:~~ \lvert \mathbf{M}_t^{\top}\mathbf{h}_t(x) \rvert \le \overline{\eta}^{\mathcal{SG}}_t(x) \doteq  \beta^{\mathcal{SG}}_{t}\, \norm{\mathbf{h}_t(x)}_{\mathbf{C}_t} + \Delta_{t}(\sqrt{\varsigma_{\mathbf{C}_t}})} \ge 1-\delta,
    \end{equation}
    with parameters $\beta^{\mathcal{SG}}_{t}$ and $\Delta_{t}(\cdot)$ as in Table~\ref{tab:bounds}. The assertion then follows from Lemma~\ref{lem:errorbound_general} by combining \eqref{eq:bound_deterministic} and the derived probabilistic uniform bound \eqref{eq:L2noisebound}.
\end{proof}

\subsection{Proof of Corollary~\ref{cor:cond_subG}} \label{app:proof_cond-subG}
\begin{proof}
    The claim follows from the proof of Theorem~\ref{thm:errorbounds_uniform}(a) (see Appendix~\ref{app:proof_uniform}) by showing that the noise term $\mathbf{M}_t^{\top}\mathbf{h}_t(x)$ satisfies $\mathbf{M}_t^{\top}\mathbf{h}_t(x) \in \mathcal{SG}(\sigma_M^2\lVert\mathbf{h}_t(x)\rVert_2^2)$ in case of conditionally sub-Gaussian $M_1,\dots,M_t$:

    Choose any $\rho > 0$, $x \in \mathcal{X}$, and $t \in \mathbb{N}$. By denoting $\mathbf{h}_t(x) = [h_1(x), \dots, h_t(x)]^{\top}$, we can write $\mathbf{M}_t^{\top}\mathbf{h}_t(x) = \sum_{i=1}^t h_i(x) M_i$. For any $a\in \mathbb{R}/\{0\}$, the $\sigma$-algebra of $aM_i$ and $M_i$ are the same, whereas for $a = 0$, the $\sigma$-algebra of $0 \cdot M_i$ is trivially contained in the $\sigma$-algebra of $M_i$ for any $i \in \mathbb{N}_t$. With $M_1 \in \mathcal{SG}(\sigma_M^2)$, we have $h_1(x)M_1 \in \mathcal{SG}(\sigma_M^2h_1(x)^2)$ due to \citep[Theorem~1a]{ao2025stochastic}. Furthermore, since $M_2$ is conditionally $\sigma_M$-sub-Gaussian conditioned on the $\sigma$-algebra of $M_1$, we have $\sum_{i=1}^2 h_i(x) M_i \in \mathcal{SG}(\sigma_M^2(h_1(x)^2+h_2(x)^2))$ due to \citep[Theorem~1b]{ao2025stochastic}. After recursive application of \citep[Theorem~1]{ao2025stochastic} on all partial sums of $\mathbf{M}_{j}^{\top}\mathbf{h}_{j}(x) = \sum_{i=1}^{j} h_i(x) M_i$ with $j\in\mathbb{N}_t$, we obtain $\mathbf{M}_t^{\top}\mathbf{h}_t(x) \in \mathcal{SG}(\sigma_M^2\lVert\mathbf{h}_t(x)\rVert_2^2)$. The claim then follows from the proof of Theorem~\ref{thm:errorbounds_uniform}(a).
\end{proof}

\subsection{Proof of Proposition~\ref{prop:errorbound_HT}} \label{app:proof_heavytail}
\begin{proof}
    Let us first decompose the truncated output data $\hat{\mathbf{y}}_t$ into noise-free component $\mathbf{f}_t$ and truncated noise component $\hat{\mathbf{M}}_t(\omega)$, i.e., $\hat{\mathbf{y}}_t = \mathbf{f}_t + \hat{\mathbf{M}}_t(\omega)$. By the triangle inequality, the regression error $\lvert f(x) - \hat{\mu}_t(x) \rvert$ is bounded by
    \begin{align}
         \lvert f(x) - \hat{\mu}_t(x) \rvert =  \lvert f(x)  - \left(\mathbf{f}_t + \hat{\mathbf{M}}_t(\omega)\right)^{\top}\mathbf{h}_t(x) \rvert &\le \lvert f(x) - \mathbf{f}^{\top}\mathbf{h}_t(x) \rvert + \lvert \hat{\mathbf{M}}_t(\omega)^{\top}\mathbf{h}_t(x)\rvert \notag \\
         &\overset{\eqref{eq:bound_deterministic}}{\le} B \tilde{\sigma}_t(x) + \lvert \hat{\mathbf{M}}_t(\omega)^{\top}\mathbf{h}_t(x)\rvert. \label{eq:l2bound_step1}
    \end{align}
    Moreover, applying the triangle inequality on the noise term $\lvert \hat{\mathbf{M}}_t(\omega)^{\top}\mathbf{h}_t(x)\rvert =  \lvert (\hat{\mathbf{M}}_t(\omega) - \mean{\hat{\mathbf{M}}_t}+ \mean{\hat{\mathbf{M}}_t})^{\top}\mathbf{h}_t(x)\rvert$, we obtain
    \begin{align}
         \lvert \hat{\mathbf{M}}_t(\omega)^{\top}\mathbf{h}_t(x)\rvert \le \lvert (\hat{\mathbf{M}}_t(\omega) - \mean{\hat{\mathbf{M}}_t})^{\top}\mathbf{h}_t(x)\rvert + \lvert\mean{\hat{\mathbf{M}}_t}^{\top}\mathbf{h}_t(x)\rvert. \label{eq:l2bound_step2}
    \end{align}
    Since $\hat{\mathbf{M}}_t - \mean{\hat{\mathbf{M}}_t} = \hat{\mathbf{Y}}_t - \mean{\hat{\mathbf{Y}}_t}$ is zero-mean and $\lvert\hat{Y}_i-\mean{\hat{y}_i}\rvert \le 2 b_t$ holds for all $i\in\mathbb{N}_t$ by design, we have $\hat{\mathbf{M}}_t - \mean{\hat{\mathbf{M}}_t} \in \mathcal{SG}(4b_t^2)$. Thus, Theorem~\ref{thm:errorbounds_uniform}(a) yields
    \begin{equation} 
        \prob{\forall x \in \mathcal{X},~~ \forall t \in \mathbb{N}:~~ \lvert (\hat{\mathbf{M}}_t - \mean{\hat{\mathbf{M}}_t})^{\top}\mathbf{h}_t(x) \rvert \le  \sqrt{2\ln(4\pi_t Z(\zeta_{t},\mathcal{X})/\delta)} 2b_t\norm{\mathbf{h}_t(x)}_{2}  + \Delta^{\mathrm{HT}}_{t} ~} \ge 1-\delta. \label{eq:l2bound_step3}
    \end{equation}
    Furthermore, note that $\mean{\hat{M}_i} = \mean{\hat{Y}_i} - f(x_i) = \mean{Y_i \mathbbm{1}_{\lvert Y_i\rvert \le b_i}} - f(x_i)  = - \mean{Y_i \mathbbm{1}_{\lvert Y_i\rvert > b_i}}$ for all $i\in\mathbb{N}_t$ \cite{Chowdhury2019bayesian}. Thus, applying the Cauchy-Schwarz inequality  yields
    \begin{align}
        \lvert\mean{\hat{\mathbf{M}}_t}^{\top}\mathbf{h}_t(x)\rvert \le \lVert \mean{\hat{\mathbf{M}}_t} \rVert_2 ~ \lVert \mathbf{h}_t(x)\rVert_2 &= \sqrt{\sum_{i=1}^t\mean{Y_i \mathbbm{1}_{\lvert Y_i\rvert > b_i}}^2}~ \lVert \mathbf{h}_t(x)\rVert_2 \notag \\
        & \le \sqrt{\sum_{i=1}^t\frac{1}{b_i^{2a}}\mean{\lvert Y_i  \rvert^{1+a}}^2}~ \lVert \mathbf{h}_t(x)\rVert_2 \le \ol{v}\sqrt{\sum_{i=1}^t\frac{1}{b_i^{2a}}}~ \lVert \mathbf{h}_t(x)\rVert_2\label{eq:l2bound_step4},
    \end{align}
    where \eqref{eq:l2bound_step4} follows from $\lvert Y_i\rvert\mathbbm{1}_{\lvert Y_i\rvert > b_i} \le \lvert Y_i\rvert \frac{\lvert Y_i\rvert^a}{b_i^a}$ for all $a > 0$, and from $\mean{\lvert Y_i\rvert^{1+a}} \le \ol{v}$. Using $b_t = \ol{v}^{\frac{1}{1+a}} t^{\frac{1}{2(1+a)}}$, we further obtain (cf. \citep[Equation~12]{Chowdhury2019bayesian})
    \begin{align}
       \lvert\mean{\hat{\mathbf{M}}_t}^{\top}\mathbf{h}_t(x)\rvert &\le \ol{v}^{1-\frac{a}{1+a}}\sqrt{\sum_{i=1}^t i^{-\frac{a}{1+a}}}~ \lVert \mathbf{h}_t(x)\rVert_2 \le \ol{v}^{\frac{1}{1+a}}\sqrt{\int_{0}^t \tau^{-\frac{a}{1+a}}\mathrm{d}\tau}~ \lVert \mathbf{h}_t(x)\rVert_2 = \sqrt{1+a} ~b_t~ \lVert \mathbf{h}_t(x)\rVert_2\label{eq:l2bound_step5}
    \end{align}    
    The assertion then follows by combining \eqref{eq:l2bound_step1}--\eqref{eq:l2bound_step3} and \eqref{eq:l2bound_step5}.
\end{proof}

\section{Parameter estimation error bound} \label{app:parambounds}
Consider the case in which a finite-dimensional representation $f(x) = \bm{\theta}^\top \bm{\phi}(x)$ of the function $f$ exists, with unknown parameter vector $\bm{\theta} \in \mathbb{R}^{n_\phi}$ and known vector of basis functions $\bm{\phi}: \mathcal{X} \to \mathbb{R}^{n_\phi}$. For this special case, the corresponding kernel function \eqref{eq:k_phi} results in $k(x,x') = \bm{\phi}(x)^{\top}\bm{\phi}(x')$, whereas $\bm{\Phi}_t$ from \eqref{eq:phi_operator} and $\mathbf{A}_t$, $\tilde{\mathbf{A}}_t$ from \eqref{eq:kervar_operators} result in the matrices
\begin{align*}
    \bm{\Phi}_t &= \left[\bm{\phi}(x_1) \, \ldots \, \bm{\phi}(x_t)\right] \in \mathbb{R}^{n_\phi \times t},\\
    \mathbf{A}_t &= \mathbf{I}_{n_\phi} - \mathbf{\Phi}_t  (\rho^2 \mathbf{I}_{t} + \mathbf{\Phi}_t ^\top\mathbf{\Phi}_t)^{-1}\mathbf{\Phi}_t^\top = \rho^{2}(\rho^{2}\mathbf{I}_{n_{\phi}} + \bm{\Phi}_t \bm{\Phi}_t^\top)^{-1} = (\mathbf{I}_{n_{\phi}} + \rho^{-2}\bm{\Phi}_t \bm{\Phi}_t^\top)^{-1}  \in \mathbb{R}^{n_\phi \times n_\phi},\\
    \tilde{\mathbf{A}}_t &= \mathbf{A}_t - \mathbf{A}_t^2 = \rho^2 \mathbf{\Phi}_t (\rho^2 \mathbf{I}_t + \mathbf{\Phi}_t^\top \mathbf{\Phi}_t)^{-2}\mathbf{\Phi}_t^{\top} = \rho^2 (\rho^2 \mathbf{I}_{n_\phi} + \mathbf{\Phi}_t \mathbf{\Phi}_t^\top)^{-1} \mathbf{\Phi}_t \mathbf{\Phi}_t^\top (\rho^2 \mathbf{I}_{n_\phi} + \mathbf{\Phi}_t \mathbf{\Phi}_t^\top)^{-1} \in \mathbb{R}^{n_\phi \times n_\phi},
\end{align*}
based on the data \eqref{eq:dataset} at time $t$. This allows us to express the kernel-based estimate \eqref{eq:kermean} as $\mu_t(x) = \hat{\bm{\theta}}_t^\top \bm{\phi}(x)$, with the parameter estimate $\hat{\bm{\theta}}_t = \bm{\Phi}_t(\bm{\Phi}_t^\top \bm{\Phi}_t+\rho^2 I_t)^{-1}\mathbf{y}_t \in \mathbb{R}^{n_\phi}$. For this finite-dimensional case, beyond probabilistic uniform error bounds of the form~\eqref{eq:errorbound_uniform}, one might be interested in parameter estimation error bounds of the form
\begin{equation}\label{eq:errorbound_params}
          \mathbb{P}\left[ \forall x \in \mathcal{X},\,t\in\mathbb{N}:\,\lVert\bm{\theta} - \hat{\bm{\theta}}_t\rVert_2 \le {\eta}^{\theta}_t\right] \geq 1 - \delta,
\end{equation}
with an error bound ${\eta}^{\theta}_t \in \mathbb{R}_{\ge0}$ and confidence $\delta \in (0,1)$. The following result yields such parameter estimation error bounds under non-Gaussian distributions of the noise $M_i$, $i \in \mathbb{N}_{t}$ (see Section~\ref{sec:randvar}), as corresponding to Theorem~\ref{thm:errorbounds_uniform}.
\begin{theorem} \label{thm:errorbounds_params}
    Consider data \eqref{eq:dataset} generated via \eqref{eq:system} under Assumption~\ref{asm:system} and $f(x) = \bm{\theta}^\top \bm{\phi}(x)$ with realizations $M_1(\omega), \dots,  M_t(\omega)$ of i.i.d. noise $M_i$, $i \in \mathbb{N}_t$ and some $\{\pi_t\}_{t=1}^{\infty}$ that satisfies $\sum_{t=1}^{\infty}\pi_t^{-1} = 1$ (e.g., $\pi_t = \pi^2 t^2 / 6$).\\     
    Then, for every $\rho > 0$ and $\delta \in (0,1)$, the parameter estimation error $\lVert\bm{\theta} - \hat{\bm{\theta}}_t\rVert_2$ is bounded by~\eqref{eq:errorbound_params} with 
    $${\eta}^\theta_t = B \lVert \mathbf{A}_t \rVert_2 + \frac{\sqrt{n_\phi}}{\rho} \gamma^M_t(\delta) \sqrt{\lVert \tilde{\mathbf{A}}_t \rVert_2},$$ 
    where $\gamma^{M}_t(\delta)$ defined as follows:
    
    \paragraph{(a)} If $M_i \in \mathcal{SG}(\sigma_M^2)$, then $\gamma^M_t(\delta) = \gamma^{\mathcal{SG}}_t(\delta) \doteq  \sqrt{2\ln(2\pi_t n_\phi/\delta)}~\sigma_M$.
   
    \paragraph{(b)} If $M_i \in \mathcal{L}^{\infty}(\overline{m})$ and $\var{M_i} \le \overline{\sigma}^2$, then $\gamma^M_t(\delta) = \gamma^{\mathrm{bnd}}_t(\delta) \doteq \frac{2}{3}\ln(2\pi_t n_\phi/\delta)~ \overline{m} + \sqrt{2\ln(2\pi_t n_\phi/\delta)}~ \overline{\sigma}$ .
   
    \paragraph{(c)} If $M_i \in \mathcal{SE}(\nu_M^2,\alpha_M)$, then $\gamma^{M}_t(\delta) =  \gamma^{\mathcal{SE}}_t(\delta)\doteq \max\left\{2\ln(2\pi_t n_\phi/\delta)\, \alpha_M,~ \sqrt{2\ln(2\pi_t n_\phi/\delta)}~ \nu_M\right\}$.
    
    \paragraph{(d)} If $M_i \in \mathcal{L}^2(\sigma_M^2)$, then $\gamma^{M}_t(\delta) = \gamma^{\mathcal{L}2}_{t}(\delta) \doteq \sqrt{\pi_t n_\phi/\delta}~ \sigma_M$.
\end{theorem}

\begin{proof}
    Consider any $\rho > 0$. First, let us decompose the available output data $\mathbf{y}_t$ (generated via \eqref{eq:system}) into noise-free component $\mathbf{f}_t$ and noise component $\mathbf{M}_t(\omega)$, i.e., $\mathbf{y}_t = \mathbf{f}_t + \mathbf{M}_t(\omega)$. By the triangle inequality, the parameter estimation error is bounded by
    \begin{align}
         \lVert\bm{\theta} - \hat{\bm{\theta}}_t\rVert_2 &=  \lVert \bm{\theta}  - \mathbf{H}_t\left(\mathbf{f}_t + \mathbf{M}_t(\omega)\right) \rVert_2 
         \le \lVert \bm{\theta}  - \mathbf{H}_t\mathbf{f}_t \rVert_2 + \lVert \mathbf{H}_t\mathbf{M}_t(\omega)\rVert_2, \label{eq:bound_triangle_param}
    \end{align}
    with $\mathbf{H}_t \doteq \bm{\Phi}_t(\rho^2 I_t + \bm{\Phi}_t^\top \bm{\Phi}_t)^{-1} \in \mathbb{R}^{n_\phi \times t}$. Note that we can write $\mathbf{A}_t = \mathbf{I}_{n_\phi} - \mathbf{H}_t\bm{\Phi}_t^\top$ and $\tilde{\mathbf{A}}_t = \rho^2\mathbf{H}_t \mathbf{H}_t^\top$. 
    In order to prove the assertion, we will derive individual bounds for the two terms on the right-hand side of \eqref{eq:bound_triangle_param}. 
    
    Using $\mathbf{f}_t = \bm{\Phi}_t^\top \bm{\theta}$, the Cauchy-Schwarz inequality, and $\lVert \bm{\theta}\rVert_2 = \lVert \bm{\theta}^\top\bm{\phi}(x)\rVert_{\mathcal{H}_k} \le B$ via Assumption~\ref{asm:system}, we obtain
    \begin{align}
        \lVert \bm{\theta}  - \mathbf{H}_t\mathbf{f}_t \rVert_2 = \lVert\left(\mathbf{I}_{n_\phi} - \mathbf{H}_t\bm{\Phi}_t^\top \right) \bm{\theta}\rVert_2 \le \lVert \left(\mathbf{I}_{n_\phi} - \mathbf{H}_t\bm{\Phi}_t^\top \right)\rVert_2 ~ \lVert \bm{\theta} \rVert_2 \le B \lVert \mathbf{A}_t\rVert_2.\label{eq:bound_rkhs_param}
    \end{align}
    To analyze the noise term $\lVert \mathbf{H}_t\mathbf{M}_t(\omega)\rVert_2$, let us denote $\mathbf{H}_t \doteq \mat{\mathbf{b}_1,\dots,\mathbf{b}_{n_\phi}}^\top$, and note that $\lVert\mathbf{b}_j\rVert_2 \le \sqrt{\lVert \mathbf{H}_t \mathbf{H}_t^\top \rVert_2} = (1/\rho) \sqrt{ \lVert \tilde{\mathbf{A}}_t \rVert_2}$ by the definition of the spectral norm. By the properties of the $2$-norm and $\infty$-norm, we can then write 
    \begin{align}
        \lVert \mathbf{H}_t\mathbf{M}_t(\omega)\rVert_2 \le \sqrt{n_\phi}\lVert\mathbf{H}_t\mathbf{M}_t(\omega)\rVert_\infty = \sqrt{n_\phi} \max_{j\in\mathbb{N}_{n_\phi}}\{\lvert\mathbf{b}_j^\top \mathbf{M}_t(\omega)\rvert\} .\label{eq:bound_noise_param_helper}
    \end{align}
    Thus, the assertion follows from \eqref{eq:bound_triangle_param}, \eqref{eq:bound_rkhs_param}, and \eqref{eq:bound_noise_param_helper} in conjunction with a probabilistic bound
    \begin{align}
        \mathbb{P}\left[ \forall x \in \mathcal{X},\,t\in\mathbb{N},\,j\in\mathbb{N}_{n_\phi}:\,\lvert\mathbf{b}_j^\top \mathbf{M}_t\rvert \le \frac{\gamma^M_t(\delta)}{\rho}  \sqrt{\lVert \tilde{\mathbf{A}}_t \rVert_2}\right] \geq 1 - \delta,\label{eq:bound_noise_param}
    \end{align}
    for some $\delta \in (0,1)$ and a scaling factor $\gamma^M_t(\delta)$ that depends on the noise class. In the remainder of the proof, we will derive appropriate $\gamma^M_t(\delta)$ considering the noise classes from (a)--(d). Note that uniformity of \eqref{eq:bound_noise_param} over all $x \in \mathcal{X}$ holds automatically since all terms are independent of the input $x$ (cf. discussion in Appendix~\ref{app:compare_errorbounds}).
    
    \paragraph{(a)} If $M_i \in \mathcal{SG}(\sigma_M^2)$ for all $i \in \mathbb{N}_t$, then, by \citep[Lemma~2]{ao2025stochastic}, it holds that 
    \begin{equation*}
        \prob{\lvert\mathbf{b}_j^\top \mathbf{M}_t\rvert \le \sqrt{2\ln(2/\tilde{\delta})} \sigma_M \lVert\mathbf{b}_j\rVert_2} \ge 1-\tilde{\delta}
    \end{equation*}
    for all $\tilde{\delta} \in (0,1)$, cf.~Theorem~\ref{thm:errorbounds_nonuniform}(a). Hence, leveraging $\lVert\mathbf{b}_j\rVert_2 \le (1/\rho) \sqrt{ \lVert \tilde{\mathbf{A}}_t \rVert_2}$, choosing $\tilde{\delta} = \delta/(\pi_t n_\phi)$, and applying the union bound over all $j\in\mathbb{N}_{n_\phi}$ and all $t \in \mathbb{N}$, we obtain \eqref{eq:bound_noise_param} with $\gamma^M_t(\delta) = \gamma^{\mathcal{SG}}_t(\delta) = \sqrt{2\ln(2\pi_t n_\phi/\delta)}~ \sigma_M$.
    
    \paragraph{(b)} If $M_i \in \mathcal{L}^{\infty}(\overline{m})$ and $\var{M_i} \le \overline{\sigma}^2$ for all $i \in \mathbb{N}_t$, then, by Bernstein's inequality for the sum of bounded random variables \citep[Theorem~2.8.4]{vershynin2018high}, it holds that
    \begin{equation*}
        \prob{\lvert\mathbf{b}_j^\top \mathbf{M}_t\rvert \le \frac{2}{3}\ln(2/\tilde{\delta})~ \overline{m} \lVert\mathbf{b}_j\rVert_\infty + \sqrt{2\ln(2/\tilde{\delta})} \overline{\sigma} \lVert \mathbf{b}_j \rVert_2 } \ge 1-\tilde{\delta}
    \end{equation*}
    for all $\tilde{\delta} \in (0,1)$, cf.~Theorem~\ref{thm:errorbounds_nonuniform}(b). Hence, leveraging $\lVert\mathbf{b}_j\rVert_{\infty} \le \lVert\mathbf{b}_j\rVert_2 \le (1/\rho) \sqrt{ \lVert \tilde{\mathbf{A}}_t \rVert}$, choosing $\tilde{\delta} = \delta/(\pi_t n_\phi)$, and applying the union bound over all $j\in\mathbb{N}_{n_\phi}$ and all $t \in \mathbb{N}$, we obtain \eqref{eq:bound_noise_param} with $\gamma^M_t(\delta) = \gamma^{\mathrm{bnd}}_t(\delta) = \frac{2}{3}\ln(2\pi_t n_\phi/\delta)~ \overline{m} + \sqrt{2\ln(2\pi_t n_\phi/\delta)}~\overline{\sigma}$.
    
    \paragraph{(c)} If $M_i \in \mathcal{SE}(\nu_M^2,\alpha_M)$ for all $i \in \mathbb{N}_t$, then, by Bernstein's inequality for the sum of weighted sub-exponential random variables \citep[Theorem~2.8.2]{vershynin2018high}, \citep[Proposition~2.9]{wainwright2019high}, it holds that
    \begin{equation*}
        \prob{\lvert\mathbf{b}_j^\top \mathbf{M}_t\rvert \le  \max\left\{2\ln(2/\tilde{\delta})\, \alpha_M \lVert \mathbf{b}_j\rVert_\infty,~ \sqrt{2\ln(2/\tilde{\delta})}\, \nu_M \lVert \mathbf{b}_j\rVert_2\right\}} \ge 1-\tilde{\delta}
    \end{equation*}
    for all $\tilde{\delta} \in (0,1)$, cf.~Theorem~\ref{thm:errorbounds_nonuniform}(c). Hence, leveraging $\lVert\mathbf{b}_j\rVert_{\infty} \le \lVert\mathbf{b}_j\rVert_2 \le (1/\rho) \sqrt{ \lVert \tilde{\mathbf{A}}_t\rVert_2}$, choosing $\tilde{\delta} = \delta/(\pi_t n_\phi)$, and applying the union bound over all $j\in\mathbb{N}_{n_\phi}$ and all $t \in \mathbb{N}$, we obtain \eqref{eq:bound_noise_param} with $\gamma^{M}_t(\delta) = \gamma^{\mathcal{SE}}_t(\delta) = \max\left\{2\ln(2\pi_t n_\phi/\delta)\, \alpha_M,~ \sqrt{2\ln(2\pi_t n_\phi/\delta)}~ \nu_M\right\}$.
   
    \paragraph{(d)} If $M_i \in \mathcal{L}^2(\sigma_M^2)$ for all $i \in \mathbb{N}_t$, then, by the Chebychev inequality~\citep[Corollary~1.2.5]{vershynin2018high}, it holds that
    \begin{equation*}
        \prob{\lvert\mathbf{b}_j^\top \mathbf{M}_t\rvert \le \sqrt{1/\tilde{\delta}}~ \sigma_M \lVert \mathbf{b}_j \rVert_2 } \ge 1-\tilde{\delta}
    \end{equation*}
    for all $\tilde{\delta} \in (0,1)$, cf.~Theorem~\ref{thm:errorbounds_nonuniform}(d). Hence, leveraging $\lVert\mathbf{b}_j\rVert_2 \le (1/\rho) \sqrt{ \lVert \tilde{\mathbf{A}}_t \rVert_2}$, choosing $\tilde{\delta} = \delta/(\pi_t n_\phi)$, and applying the union bound over all $j\in\mathbb{N}_{n_\phi}$ and all $t \in \mathbb{N}$, we obtain \eqref{eq:bound_noise_param} with $\gamma^{M}_t(\delta) = \gamma^{\mathcal{L}2}_{t}(\delta) = \sqrt{\pi_t n_\phi/\delta}~ \sigma_M$.
\end{proof}

Notably, in this work, we do not make assumptions on the richness of the available data \eqref{eq:dataset}, which hinders the convergence analysis of the parameter estimation error from Theorem~\ref{thm:errorbounds_params}. For such an analysis, a persistent excitation assumption is required, such as $\bm{\Phi}_t \bm{\Phi}_t^\top \succ ct\mathbf{I}_{n_\phi}$ for some $c>0$, cf. \citep[Definition~3]{musavi2024identification}. Then, it follows that the terms $\lVert\mathbf{A}_t\rVert_2$ and $\sqrt{\lVert\tilde{\mathbf{A}}_t\rVert_2} \le \sqrt{\lVert\mathbf{A}_t\rVert_2}$ decay in the rates of $\mathcal{O}(1/t)$ and $\mathcal{O}(1/\sqrt{t})$, respectively, using $\mathbf{A}_t \succ\tilde{\mathbf{A}}_t=\mathbf{A}_t-\mathbf{A}_t^2$. For the case of $\pi_t = 1$ (i.e., the parameter estimation error bound from Theorem~\ref{thm:errorbounds_params} holds nonuniform in time), we recover common convergence rates reported in the literature \cite{simchowitz2018learning,musavi2024identification}.

\section{Comparison of Derivation of Error Bounds}\label{app:compare_errorbounds}

This appendix provides a more technical discussion on the differences between the derivation of the proposed probabilistic uniform error bounds (Theorem~\ref{thm:errorbounds_uniform}) and related bounds from the literature \cite{abbasi2013online,Fiedler2021}. 
In particular, under conditionally $\sigma_M$-sub-Gaussian noise $M_1,\dots,M_t$, \citep[Theorem~3.11]{abbasi2013online} yields the probabilistic uniform error bound (cf. \citep[Equation~(7)]{fiedler2024safety}
\begin{equation} \label{eq:abbasi-yadkori}
    \prob{\forall x \in \mathcal{X},~ t \in \mathbb{N}:~\lvert f(x) - \mu_t(x)\rvert \le B\sigma_t(x) + \frac{\sigma_M}{\rho}\sqrt{2\ln\left(\det\left(\mathbf{I}_t + \rho^{-2}\mathbf{K}_t\right)\right) + 2\ln(1/\delta)}~\sigma_t(x) } \ge 1-\delta,
\end{equation}
valid for any $\rho >0$ and $\delta \in (0,1)$, with $\sigma_t(x)$ from \eqref{eq:kervar}. In contrast, for i.i.d. sub-Gaussian noise $M_1,\dots,M_t$ with $M_i \in \mathcal{SG}(\sigma_M)$, $i\in\mathbb{N}_t$, \citep[Proposition~2]{Fiedler2021} yields the probabilistic uniform error bound
\begin{equation} \label{eq:fiedler}
    \forall t \in \mathbb{N}:~\prob{\forall x \in \mathcal{X}:~ \lvert f(x) - \mu_t(x)\rvert \le B\sigma_t(x) + \sqrt{t + 2\sqrt{t\ln(1/\delta)} + 2\ln(1/\delta)}~\sigma_M \lVert \mathbf{h}_t(x)\rVert_2 } \ge 1-\delta,
\end{equation}
valid for any $\rho >0$ and $\delta \in (0,1)$, where $\mathbf{h}_t(x) = (\mathbf{K}_t+\rho^2\mathbf{I}_t)^{-1}\mathbf{k}_t(x)$. Note that a time-uniform version of \eqref{eq:fiedler} is straightforwardly obtained from the union bound by setting $\delta \leftarrow \delta/\pi_t$ with some $\{\pi_t\}_{t=1}^{\infty}$ such that $\sum_{t=1}^{\infty}\pi_t^{-1} = 1$, e.g., $\pi_t = \pi^2 t^2 / 6$ \cite{Srinivas2009}.

For the proposed bounds from Theorem~\ref{thm:errorbounds_uniform} as well as for the comparison bounds, the derivation starts with the triangle inequality~\eqref{eq:bound_triangle}, i.e.,
\begin{equation*}
    \lvert f(x) - \mu_t(x) \rvert \le \lvert f(x) - \mathbf{f}^{\top}\mathbf{h}_t(x) \rvert + \lvert \mathbf{M}_t(\omega)^{\top}\mathbf{h}_t(x) \rvert,
\end{equation*}
where the noise-free component $\mathbf{f}_t$ and the noise component $\mathbf{M}_t(\omega)$ stem from the decomposition of the output data $\mathbf{y}_t$, i.e., $\mathbf{y}_t = \mathbf{f}_t + \mathbf{M}_t(\omega)$ (see Appendix~\ref{app:proof_generalbound}). In order to upper-bound the term $\lvert f(x) - \mathbf{f}^{\top}\mathbf{h}_t(x) \rvert$, representing the uncertainty that stems from lack of exploration of the function space, the comparison methods \cite{abbasi2013online,Fiedler2021} employ $\lvert f(x) - \mathbf{f}^{\top}\mathbf{h}_t(x) \rvert \le B \sigma_t(x)$ \citep[Theorem~2]{Chowdhury2017}, where $B\ge\norm{f}_{\mathcal{H}_k}$ is an upper-bound on the RKHS norm of the unknown function, see Assumption~\ref{asm:system}. In contrast, we employ the sharper bound  $\lvert f(x) - \mathbf{f}^{\top}\mathbf{h}_t(x) \rvert \le B \tilde{\sigma}_t(x) = \sqrt{\sigma^2_{t}(x) - \rho^2 \norm{\mathbf{h}_t(x)}_2^2}$, see Appendix~\ref{app:proof_generalbound}. Via Lemma~\ref{lem:errorbound_general}, probabilistic uniform error bounds are then obtained when a bound~\eqref{eq:noisebound_uniform} on the noise term $\lvert\mathbf{M}_t^{\top}\mathbf{h}_t(x)\rvert$ is available, i.e., 
\begin{equation*}
    \mathbb{P}\left[ \forall x \in \mathcal{X},\,t\in\mathbb{N}:\,\lvert\mathbf{M}_t^{\top}\mathbf{h}_t(x)\rvert \le {\eta}^{M}_t(x)\right] \geq 1 - \delta.
\end{equation*}

Notably, the crucial difference between the proposed approach and the comparison methods lies in the derivation of this probabilistic bound~\eqref{eq:noisebound_uniform} on the noise term. For instance, the Cauchy-Schwarz inequality is applied in \cite{abbasi2013online} to separate the noise term $\lvert\mathbf{M}_t(\omega)^{\top}\mathbf{h}_t(x)\rvert$ as
\begin{equation}
    \lvert\mathbf{M}_t(\omega)^{\top}\mathbf{h}_t(x)\rvert = \lvert\mathbf{M}_t(\omega)^{\top}(\mathbf{K}_t+\rho^2\mathbf{I}_t)^{-1}\mathbf{k}_t(x)\rvert \le \sqrt{\mathbf{M}_t(\omega)^{\top}\mathbf{K}_t(\mathbf{K}_t+\rho^2\mathbf{I}_t)^{-1}\mathbf{M}_t(\omega)}~\sigma_t(x).\label{eq:cauchy_abbasi}
\end{equation}
Furthermore, via a self-normalized martingale inequality (cf. \cite{Chowdhury2017}), Abbasi-Yadkori~\yrcite{abbasi2013online} obtains the probabilistic bound
\begin{equation}
    \prob{\forall t \in \mathbb{N}:~\sqrt{\mathbf{M}_t^{\top}\mathbf{K}_t(\mathbf{K}_t+\rho^2\mathbf{I}_t)^{-1}\mathbf{M}_t} \le \frac{\sigma_M}{\rho}\sqrt{2\ln\left(\det\left(\mathbf{I}_t + \rho^{-2}\mathbf{K}_t\right)\right) + 2\ln(1/\delta)}} \ge 1-\delta. \label{eq:prob_abbasi}
\end{equation}
Importantly, since the noise term $\sqrt{\mathbf{M}_t^{\top}\mathbf{K}_t(\mathbf{K}_t+\rho^2\mathbf{I}_t)^{-1}\mathbf{M}_t}$ is independent of the decision $x$, the probabilistic bound~\eqref{eq:prob_abbasi} automatically holds uniformly over all $x \in \mathcal{X}$. Thus, the combination of the Cauchy-Schwarz-based inequality~\eqref{eq:cauchy_abbasi} and the probabilistic bound~\eqref{eq:prob_abbasi} yields the probabilistic uniform error bound~\eqref{eq:abbasi-yadkori}.

Similarly, Fiedler et al.~\yrcite{Fiedler2021} apply the Cauchy-Schwartz inequality to separate the noise term $\lvert\mathbf{M}_t(\omega)^{\top}\mathbf{h}_t(x)\rvert$ as
\begin{equation}
    \lvert\mathbf{M}_t(\omega)^{\top}\mathbf{h}_t(x)\rvert \le \lVert \mathbf{M}_t(\omega) \rVert_2~ \lVert \mathbf{h}_t(x) \rVert_2.\label{eq:cauchy_fiedler}
\end{equation}
Then, the concentration inequality from \citep[Theorem~2.1]{hsu2012tail} is leveraged to obtain the probabilistic bound
\begin{equation}
    \forall t \in \mathbb{N}:~\prob{\lVert \mathbf{M}_t \rVert_2 \le \sqrt{t + 2\sqrt{t\ln(1/\delta)} + 2\ln(1/\delta)}~\sigma_M} \ge 1-\delta. \label{eq:prob_fiedler}
\end{equation}
Again, the noise term $\lVert \mathbf{M}_t \rVert_2$ is independent of the decision $x$, thus, the probabilistic bound~\eqref{eq:prob_fiedler} holds uniformly over all $x \in \mathcal{X}$. Therefore, the combination of the Cauchy-Schwarz-based inequality~\eqref{eq:cauchy_fiedler} and the probabilistic bound~\eqref{eq:prob_abbasi} yields the probabilistic uniform error bound~\eqref{eq:fiedler}.

In contrast, for the proposed bounds in Theorem~\ref{thm:errorbounds_uniform}, we do \textit{not} apply the Cauchy-Schwarz inequality to obtain inequalities of the form \eqref{eq:cauchy_abbasi} or \eqref{eq:cauchy_fiedler} to separate the decision-dependent term $\mathbf{h}_t(x)$ and the noise $\mathbf{M}_t$. Instead, we treat the noise term $\lvert\mathbf{M}_t^{\top}\mathbf{h}_t(x)\rvert$ as a (decision-dependent) weighted sum of random variables $M_1,\dots,M_t$ (equivalently, a projection of the vector-valued random variable $\mathbf{M}_t$), and thus directly apply distribution-specific concentration inequalities \cite{vershynin2018high,wainwright2019high,ao2025stochastic}. 
While this avoids the use of the conservative Cauchy-Schwarz inequality, the dependence of the noise term $\lvert\mathbf{M}_t^{\top}\mathbf{h}_t(x)\rvert$ on the decision $x \in \mathcal{X}$ renders the probabilistic noise bounds \textit{nonuniform} as in~\eqref{eq:noisebound_nonuniform}, i.e., 
\begin{equation*}
    \forall x \in \mathcal{X}~,t\in\mathbb{N}:~ \mathbb{P}\left[\lvert\mathbf{M}_t^{\top}\mathbf{h}_t(x)\rvert \le \tilde{\eta}_t^{M}(x)\right] \geq 1 - \delta,
\end{equation*}
for some $\tilde{\eta}^M_t$ and $\delta\in(0,1)$; see Theorem~\ref{thm:errorbounds_nonuniform}. 
Therefore, we employ a discretization approach to derive uniform bounds as in \eqref{eq:errorbound_uniform}; see the proof of Theorem~\ref{thm:errorbounds_uniform} in Appendix~\ref{app:proof_uniform} for details. We remark that the error bounds proposed by \cite{reed2025error,molodchyk2025towards} follow a similar strategy regarding the treatment of the noise term $\lvert\mathbf{M}_t^{\top}\mathbf{h}_t(x)\rvert$ and concentration inequalities; however, both works yield \textit{nonuniform}\footnote{Note that, contrary to its name, \citep[Corollary~1]{reed2025error} provides a nonuniform probabilistic error bound similar to \eqref{eq:errorbound_nonuniform}, instead of the desired probabilistic uniform error bound of the form~\eqref{eq:errorbound_uniform}.} probabilistic error bounds as in \eqref{eq:errorbound_nonuniform}, thus failing to provide strong safety guarantees as for uniform error bounds~\eqref{eq:errorbound_uniform}.

Numerical experiments (see Section~\ref{sec:eval} and Appendix~\ref{app:addeval}) suggest that the proposed approach introduces less conservatism than the comparison approaches discussed above that directly apply the Cauchy-Schwarz inequality to the noise term $\lvert\mathbf{M}_t^{\top}\mathbf{h}_t(x)\rvert$ as in \eqref{eq:cauchy_abbasi} and \eqref{eq:cauchy_fiedler}. Additionally, the proposed approach enables the user to adjust the error bounds to the application at hand, particularly by offering flexibility in the noise model and by allowing the user to shape the bounds via the grid constant $\zeta_t$.

\section{Additional Material for Numerical Evaluation} \label{app:addeval}

This appendix provides additional information and implementation details on the numerical experiments conducted in Section~\ref{sec:eval}.

\paragraph{Bounded Noise}
In Sections~\ref{sec:eval_confidence}--\ref{sec:eval_control}, we consider bounded noise. Thus, we compare our proposed error bounds from Theorem~\ref{thm:errorbounds_uniform}(a)--(b) with the following two related bounds from the literatue: i) the bound from \citep[Theorem~3.11]{abbasi2013online} as given in \eqref{eq:abbasi-yadkori}, i.e., 
\begin{equation*}
    \prob{\forall x \in \mathcal{X},~ t \in \mathbb{N}:~\lvert f(x) - \mu_t(x)\rvert \le B\sigma_t(x) + \frac{\sigma_M}{\rho}\sqrt{2\ln\left(\det\left(\mathbf{I}_t + \rho^{-2}\mathbf{K}_t\right)\right) + 2\ln(1/\delta)}~\sigma_t(x) } \ge 1-\delta,
\end{equation*}
and ii) the bound from \citep[Proposition~2]{Fiedler2021}, 
adjusted to hold uniformly in both $x\in\mathcal{X}$ and in $t\in\mathbb{N}$ via the union bound using $\delta\leftarrow \delta/\pi_t$ with $\pi_t = \pi^2 t^2 / 6$ \cite{Srinivas2009}, resulting in
\begin{equation*}
    \prob{\forall x \in \mathcal{X},~t \in \mathbb{N}:~ \lvert f(x) - \mu_t(x)\rvert \le B\sigma_t(x) + \sqrt{t + 2\sqrt{t\ln(\pi_t/\delta)} + 2\ln(\pi_t/\delta)}~\sigma_M \lVert \mathbf{h}_t(x)\rVert_2 } \ge 1-\delta.
\end{equation*}
This modification is implemented to ensure a fair comparison between all methods, since the focus lies on uniformity in both $x\in\mathcal{X}$ and $t \in\mathbb{N}$.

\begin{figure*}[!t]
    \includegraphics[page=5, clip, trim=2cm 7.5cm 2cm 8cm, width=\textwidth]{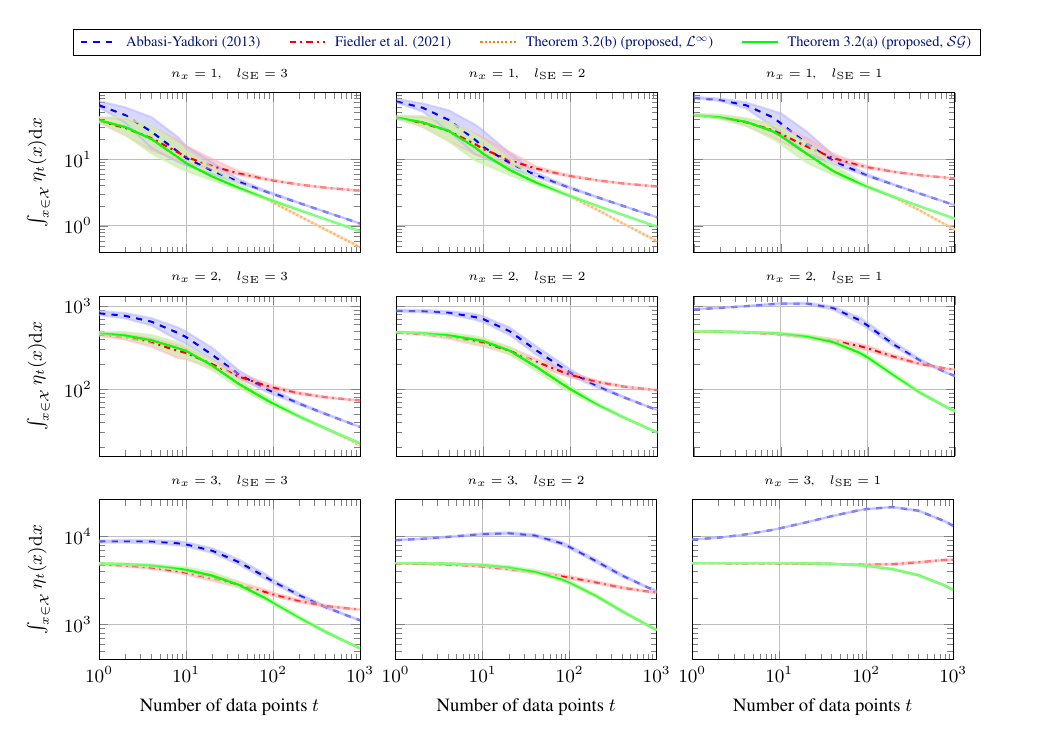}
    \caption{Comparison of the size of the uncertainty region via the integral of the probabilistic uniform error bound $\overline{\eta}_t$ \eqref{eq:errorbound_uniform} over the input domain $\mathcal{X} \subset \mathbb{R}^{n_x}$ for the kernel $k(x,x') = \exp{-\lVert x-x'\rVert_2^2/l_{\mathrm{SE}}^2}$ and increasing number of data points $t$ with input dimension $n_x \in\{1,\,2,\,3\}$ and lengthscale $l_{\mathrm{SE}} \in \{1,\,2,\,3\}$. The shaded areas show the $5\%$- to $95\%$ percentile range over $100$ Monte Carlo data collection runs.}
    \label{fig:size_errbnd}
\end{figure*}

Figure~\ref{fig:size_errbnd} extends Figure~\ref{fig:size_errbnd_nx} by also showing the comparative behavior of the discussed bounds for varying lengthscales $l_{\mathrm{SE}} \in \{1,2,3\}$ of the employed squared exponential kernel. The case of small lengthscales is particularly interesting since choosing smaller lengthscales can compensate for model misspecifications (see \citep[Section 3.2]{Fiedler2021}). As evident from Figure~\ref{fig:size_errbnd}, the relative improvement of the proposed bounds over Abbasi-Yadkori's bound increases with decreasing lengthscale. The reason for this is that Abbasi-Yadkori's bound~\eqref{eq:abbasi-yadkori} expresses the uncertainty solely via the GP posterior variance $\sigma_t(x)$ from~\eqref{eq:kervar}, thus overestimating the uncertainty induced by the noise corruption in the data in yet unexplored regions; see Figure~\ref{fig:variance_comparison}. In contrast, the proposed bounds leverage the favorable kernel- and data-dependent terms $\rho^2 \lVert \mathbf{h}_t\rVert^2_2$ and $\rho^2 \lVert \mathbf{h}_t\rVert^2_{\infty}$ to quantify the noise-induced uncertainty in the data.

Additionally, the proposed bound from Theorem~\ref{thm:errorbounds_uniform}(b) improves upon the comparison bounds specifically for a large amount of data and for $n_x=1$. 
This observation can be explained as follows: 
In contrast to the Hoeffding-type bound~\eqref{eq:errorbound_bnd_hoeffding_uniform} that is solely based on the term $\rho \lVert \mathbf{h}_t\rVert_2$, the Bernstein-type bound~\eqref{eq:errorbound_bnd_bernstein_uniform} consists of a weighted sum of the terms $\rho \lVert \mathbf{h}_t\rVert_\infty$ and $\rho \lVert \mathbf{h}_t\rVert_2$ depending on the absolute bound $\overline{m}$ and the standard deviation $\overline{\sigma} < \overline{m}$, respectively. When there are lots of similar entries in the vector $\mathbf{h}_t(x)$ (e.g., due to dense sampling of the input domain $\mathcal{X}$), the $\infty$-norm of $\mathbf{h}_t(x)$ is substantially smaller than its $2$-norm. 
This leads to $\rho \lVert \mathbf{h}_t\rVert_\infty \ll \rho \lVert \mathbf{h}_t\rVert_2$ if $t$ is large, see Figure~\ref{fig:variance_comparison}, resulting in more emphasis on the standard deviation $\overline{\sigma}$ instead of the noise bound $\overline{m}$ in the Bernstein-type bound \eqref{eq:errorbound_bnd_bernstein_uniform} that thus outperforms the Hoeffding-type bound~\eqref{eq:errorbound_bnd_hoeffding_uniform} for $ \overline{\sigma} \ll \overline{m}$. Notably, for increasing input dimension $n_x$, the considered maximum number of $1\ 000$ samples is likely not sufficient for dense sampling of the input domain; thus, this improvement of the Bernstein-type bound~\eqref{eq:errorbound_bnd_bernstein_uniform} over the Hoeffding-type bound~\eqref{eq:errorbound_bnd_hoeffding_uniform} is not visible in Figure~\ref{fig:size_errbnd} for $n_x \in \{2,\,3\}$. A similar observation is made when varying the size $r$ of the input domain $\mathcal{X} = [0,\,r]$, as depicted in Figure~\ref{fig:size_errbnd_size}.
 
\begin{figure*}[!t]
    \includegraphics[page=6, clip, trim=2cm 5.8cm 2cm 17cm, width=\textwidth]{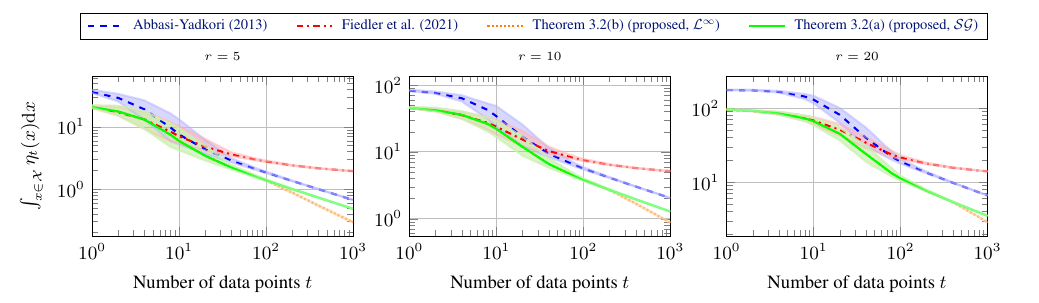}
    \caption{Comparison of the size of the uncertainty region via the integral of the probabilistic uniform error bound ${\eta}_t$ \eqref{eq:errorbound_uniform} over the input domain $\mathcal{X} = [0,\,r]$ for the kernel $k(x,x') = \exp{-(x-x')^2}$ and increasing number of data points $t$, with varying size $r$ of the input domain. The shaded areas show the $5\%$- to $95\%$ percentile range over $100$ Monte Carlo data collection runs.}
    \label{fig:size_errbnd_size}
\end{figure*}

Since the results presented thus far have considered only the squared exponential kernel, we now investigate the comparative behavior across different kernel types. In particular, we consider the linear kernel $k_{\mathrm{lin}}(\cdot,\cdot)$ and the Matérn kernels $k_\nu(\cdot,\cdot)$ with $\nu\in\{1/2,~3/2\}$ with parameter $l_\nu>0$, defined as
\begin{subequations}\label{eq:kernels}
    \begin{align}
        k_{\mathrm{lin}}(x,x') &= x^{\top} x',\\
        k_{1/2}(x,x') &= \exp{-\frac{\lVert x-x'\rVert_1}{l_{\nu}}},\\
        k_{3/2}(x,x') &= \left(1+\sqrt{3}\frac{\lVert x-x'\rVert_1}{l_{\nu}}\right)\exp{-\sqrt{3}\frac{\lVert x-x'\rVert_1}{l_{\nu}}},
    \end{align}
\end{subequations}
Results are shown in Figure~\ref{fig:size_errbnd_ker} for $l_\nu = 1$ and $n_x = 1$, yielding similar behaviour as for the previously considered scenarios. In particular, for the linear kernel, the proposed bound from Theorem~\ref{thm:errorbounds_uniform}(a) performs comparably to Abbasi-Yadkori's bound. However, as before, the proposed bound from Theorem~\ref{thm:errorbounds_uniform}(b) outperforms all other bounds in the considered scenario when a large amount of data is available.

\begin{figure*}[!t]
    \includegraphics[page=6, clip, trim=2cm 17.3cm 2cm 5.6cm, width=\textwidth]{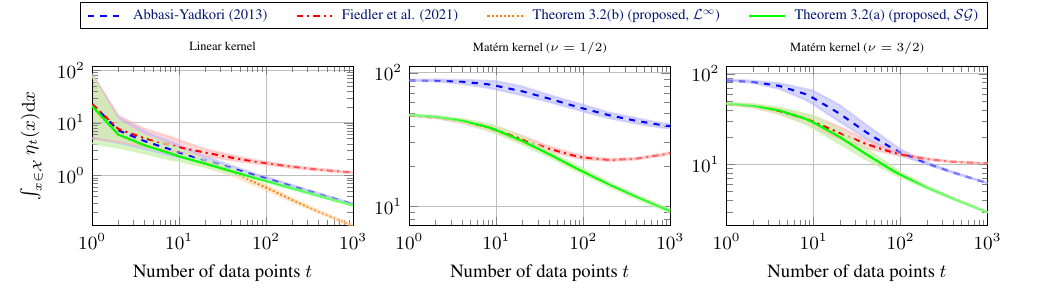}
    \caption{Comparison of the size of the uncertainty region via the integral of the probabilistic uniform error bound ${\eta}_t$ \eqref{eq:errorbound_uniform} over the input domain $\mathcal{X} [0,\,r]$ and increasing number of data points $t$, with varying kernel $k(x,x')$ from \eqref{eq:kernels} with $l_\nu = 1$ and $n_x = 1$. The shaded areas show the $5\%$- to $95\%$ percentile range over $100$ Monte Carlo data collection runs.}
    \label{fig:size_errbnd_ker}
\end{figure*}

Finally, we comment on the choice of the grid constant $\zeta_t$, which affects the scaling factors and discretization terms of the proposed bounds; see Table~\ref{tab:bounds}. Thus far, we have chosen $\zeta_t$ such that $\Delta_t^{\mathcal{SG}} = \Delta_{1,t}^{\mathrm{bnd}} = \Delta_{2,t}^{\mathrm{bnd}} = 0.001$ for all times $t \in \mathbb{N}_{1\,000}$. In contrast, the GPR literature suggests that the discretization terms decay over time, specifically when regret bounds are of interest \cite{Srinivas2009}.  Figure~\ref{fig:discerr} shows the mean size of the uncertainty region over the learning steps resulting from the proposed sub-Gaussian bound from Theorem~\ref{thm:errorbounds_uniform}(a) for different choices of the discretization terms $\Delta_t^{\mathcal{SG}}$. 
The corresponding grid constants $\zeta_t$ are computed using MATLAB's \texttt{fzero} solver. 
As can be seen, for the considered scenario, the choice of $\Delta^{\mathcal{SG}}_t$ has no strong effect on the mean uncertainty size, as long as $\Delta^{\mathcal{SG}}_t$ is chosen sufficiently small. However, we remark that choosing $\Delta^{\mathcal{SG}}_t \approx 0$ (and thus $\zeta_t \approx 0$) right from the start is not feasible since this would lead to a blow up of the scaling factor $\beta_t^{\mathcal{SG}} \to \infty$ due to the dependence on the covering number.

\begin{figure}[!t]
\centering
    \includegraphics[page=7, clip, trim=6.6cm 11.5cm 6.9cm 11.3cm, width=0.5\textwidth]{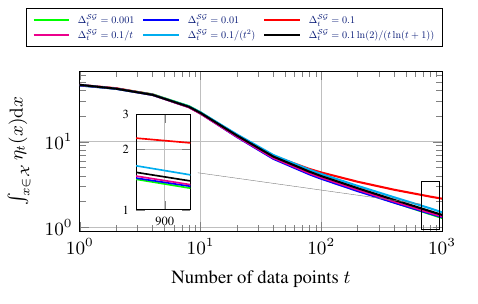}
    \caption{Comparison of the mean size of the uncertainty region via the integral of the probabilistic uniform error bound ${\eta}_t = \eta_t^{\mathcal{SG}}$ (see Theorem~\ref{thm:errorbounds_uniform}(a)) over the input domain $\mathcal{X} =[0,\,10]$ for the kernel $k(x,x') = \exp{-(x-x')^2}$ and increasing number of data points $t$ with varying discretization terms $\Delta_t^{\mathcal{SG}}$ (see Table~\ref{tab:bounds}).}
    \label{fig:discerr}
\end{figure}

\paragraph{Sub-exponential Noise}
In Section~\ref{sec:eval_subE}, we consider sub-exponential noise. Thus, we compare our proposed error bounds from Theorem~\ref{thm:errorbounds_uniform}(c)--(d) and Proposition~\ref{prop:errorbound_HT} with the bound from \citep[Lemma 8]{Chowdhury2019bayesian}. In particular, if $k(x,x)\le 1$ for all $x\in\mathcal{X}$, then \citep[Lemma 8]{Chowdhury2019bayesian} yields
\begin{equation} \label{eq:chowdhury}
    \prob{\forall x \in \mathcal{X},~ t \in \mathbb{N}:~\lvert f(x) - \hat{\mu}_t(x)\rvert \le B\sigma_t(x) + \frac{3}{\rho}\,b_t\,\sqrt{2\ln\left(\det\left(\mathbf{I}_t + \rho^{-2}\mathbf{K}_t\right)\right) + 2\ln(1/\delta)}~\sigma_t(x) } \ge 1-\delta
\end{equation}
 for all $\delta\in(0,1)$, where $\hat{\mu}_t(x) = \hat{\mathbf{y}}_t^\top\left(\mathbf{K}_t+\rho^2\mathbf{I}_t\right)^{-1}\mathbf{k}_t(x)$ is the truncated mean predictor using the truncated outputs $\hat{\mathbf{y}}_t = [\hat{y}_1,\dots,\hat{y}_t]^{\top}$ with $\hat{y}_i \doteq y_i \mathbbm{1}_{\lvert y_i\rvert \le b_i}$ for $i\in\mathbb{N}_t$, the indicator function $\mathbbm{1}_{(\cdot)}$, and the truncation level $b_t = \overline{v}\, t^{\frac{1}{4}}$ with some $\overline{v}>0$ such that $\mean{\lvert Y_t\rvert^2} \le \overline{v}^2$ for all $t \in\mathbb{N}$, where $Y_t = f(x_t) + M_t$ (cf. Proposition~\ref{prop:errorbound_HT}). Notably, the error bound~\eqref{eq:chowdhury} differs from the proposed error bounds in Theorem~\ref{thm:errorbounds_uniform} since i) a truncation of the output data is performed for the predictor $\hat{\mu}_t(x)$, and ii) the bound~\eqref{eq:chowdhury} relies on the statistics of the output $Y_t$, which depends on the unknown function $f(x_t)$. Thus, in order to obtain the parameter $\overline{v}$ that bounds the second moment of the ouput $Y_t$, additional knowledge of the function $f(x_t)$ is required, cf.~\eqref{eq:momentbound}. Considering $M_i~\in \mathcal{SE}(4\sigma_M^2,\, 4\sigma_M^2)$ with $\var{M_i} = 2\sigma_M^4$ and $\sigma_M=0.1$, as well as the additional knowledge $\lvert f(x_t)\rvert <2$, we obtain $\ol{v}=2.01$.

Although the truncation approach from Chowdhury \& Gopalan~\yrcite{Chowdhury2019bayesian} allows to employ an error bound similar to \eqref{eq:abbasi-yadkori} with logarithmic dependence on the confidence $\delta$, the bound is additionally scaled by the factor $b_t = \overline{v} t^{\frac{1}{4}}$, which has an unfavorable dependence on the number of data points, similar to Fiedler et al.'s bound~\eqref{eq:fiedler} and Proposition~\ref{prop:errorbound_HT}.

Regarding the proposed Chebyshev-based bound from Theorem~\ref{thm:errorbounds_uniform}(d), we remark that the discretization term $\Delta_t^{\mathcal{L}2}$ from Table~\ref{tab:bounds} cannot be chosen arbitrarily small via the grid constant $\zeta_t$ since the discretization term depends on the square root of the covering number. This results in a blow-up of $\Delta_t^{\mathcal{L}2}$ for an increasing number of data points, yielding the rapid increase in uncertainty size as depicted in Figure~\ref{fig:size_errbnd_subE}. For the simulations in Section~\ref{sec:eval_subE}, we choose the grid constant $\zeta_t$ for the Chebyshev-based bound from Theorem~\ref{thm:errorbounds_uniform}(d) by minimizing the weighted sum $\beta_t^{\mathcal{L}2}$ + 100 $\Delta_t^{\mathcal{L}2}$ (see Table~\ref{tab:bounds}) via MATLAB's \textit{fmincon} solver.

\end{document}